\documentclass{article}
\usepackage{natbib}
\setcitestyle{numbers,square}
\usepackage{iclr2024_conference,times}

\usepackage{hyperref}
\usepackage{url}

\usepackage{amsmath,amsfonts,bm}

\def\eqref#1{equation~\ref{#1}}

\def\floor#1{\lfloor #1 \rfloor}
\def\1{\bm{1}}

\DeclareMathAlphabet{\mathsfit}{\encodingdefault}{\sfdefault}{m}{sl}
\SetMathAlphabet{\mathsfit}{bold}{\encodingdefault}{\sfdefault}{bx}{n}

\newcommand{\Var}{\mathrm{Var}}

\usepackage[utf8]{inputenc} 
\usepackage[T1]{fontenc}    
\usepackage{hyperref}       
\usepackage{url}            
\usepackage{booktabs}       
\usepackage{amsfonts}       
\usepackage{nicefrac}       
\usepackage{microtype}      
\usepackage{xcolor}         

\usepackage{pifont}
\usepackage[flushleft]{threeparttable}
\usepackage{enumerate}
\usepackage{booktabs,makecell, multirow, tabularx}
\usepackage{graphicx}
\usepackage{colortbl}
\usepackage{bm}
\usepackage{algorithm}
\usepackage{algorithmicx}
\usepackage{algpseudocode}
\usepackage{mathrsfs}
\usepackage{mathtools}
\usepackage{amsthm,amssymb,enumitem}
\usepackage{float} 

\usepackage{subcaption}

\usepackage[framemethod=tikz]{mdframed}

\def\<{\left\langle} 
\def\>{\right\rangle}

\newtheorem{theorem}{Theorem}
\newtheorem{lemma}{Lemma}

\newtheorem{assumption}{Assumption}
\newtheorem{corollary}{Corollary}

\newtheorem{remark}{Remark}

\usepackage{graphicx} 

\title{New Insight of Variance reduce in Zero-Order Hard-Thresholding: Mitigating Gradient Error and Expansivity Contradictions}
\iclrfinalcopy
\author{Xinzhe Yuan$^1$, William de Vazelhes$^2$, Bin Gu$^{2,3*}$, Huan Xiong$^{1,2}$\thanks{Corresponding authors.}\\
{${^1}$ IASM, Harbin Institute of Technology} 
\\
{${^2}$Mohamed bin Zayed University of Artificial Intelligence} \\ 
{${^3}$School of Artificial Intelligence, Jilin University} \\ 
\texttt{22s012015@stu.hit.edu.cn}\\
\texttt{\{william.vazelhes, bin.gu\}@mbzuai.ac.ae}
\\
\texttt{huan.xiong.math@gmail.com}}
\date{September 2023}
\date{September 2023}

\begin{document}

\maketitle
\begin{abstract}
Hard-thresholding is an important type of algorithm in machine learning that is used to solve $\ell_0$ constrained optimization problems. However,  the true gradient of the objective function can be difficult to access in certain scenarios, which normally can be approximated by zeroth-order (ZO) methods. The SZOHT algorithm is the only algorithm tackling $\ell_0$ sparsity constraints with ZO gradients so far. Unfortunately,  SZOHT  has a notable limitation on the number of random directions 
due to the inherent conflict between the deviation of ZO gradients and the expansivity of the hard-thresholding operator. 
This paper approaches this problem by considering the role of variance and provides a new insight into variance reduction: mitigating the unique conflicts between ZO gradients and hard-thresholding.  Under this perspective, we propose a generalized variance reduced ZO hard-thresholding algorithm as well as the generalized convergence analysis under standard assumptions. The theoretical results demonstrate the new algorithm eliminates the restrictions on the number of random directions, leading to improved convergence rates and broader applicability compared with SZOHT.  Finally, we illustrate the utility of our method on a ridge regression problem as well as black-box adversarial attacks.

\end{abstract}
\section{Introduction}
$\ell_0$ constrained optimization is a fundamental method in large-scale machine learning, particularly in high-dimensional problems. This approach is widely favored for achieving sparse learning. It offers numerous advantages, notably enhancing efficiency by reducing memory usage, computational demands, and environmental impact. Additionally, this constraint plays a crucial role in combatting overfitting and facilitating precise statistical estimation \citep{negahban2012unified,raskutti2011minimax,buhlmann2011statistics,yuan2021stability}. In this study, we focus on the following problem:
\begin{equation}
\min_{\theta\in\mathbb{R}^d}\mathcal{F}(\theta)=\frac{1}{n}\sum^n_{i=1} f_i(\theta),\quad s.t.\  \mathcal{k}\theta\mathcal{k}_0\le k,\label{equ:queation1}
\end{equation}
Here, $\mathcal{F}(\theta)$ is the  (regularized) empirical risk. $\left\|\theta\right\|_0$ represents the number of non-zero directions. $d$ is the dimension of $\theta$. Unfortunately, due to the $\ell_0$ constraint, (\ref{equ:queation1}) becomes an NP-hard problem, rendering traditional methods unsuitable for its analysis.

Therefore, we consider using the hard-threshold iterative algorithm \citep{raskutti2011minimax,jain2014iterative,nguyen2017linear,yuan2017gradient}, which is a widely used technique for obtaining approximate solutions to NP-hard's $\ell_0$ constrained optimization problems. Specifically, this technique alternates between the gradient step and the application of the hard threshold operator $\mathcal{H}_k(\theta)$. Operator $\mathcal{H}_k(\theta)$ retains the top $k$ elements of $\theta$ while setting all other directions to zero. The advantage of hard-thresholding over its convex relaxations is that it can achieve similar precision without the need for computationally intensive adjustments, such as tuning $\ell_1$ penalties or constraints. Hard-thresholding was first used for its full gradient form\citep{jain2014iterative}. Nguyen \citep{nguyen2017linear} developed a stochastic gradient descent SGD version of hard thresholding known as StoIHT. Nevertheless, StoIHT's convergence condition is overly stringent for practical applications\citep{li2016nonconvex}.  To address this issue, \citep{zhou2018efficient}, \citep{shen2017tight}  and \citep{li2016nonconvex} implemented variance reduction techniques to improve the performance of StoIHT in real-world problem-solving.

However, this type of StoIHT is still not suitable for many problems.  For example, in certain graphical modeling tasks \citep{blumensath2009iterative},
obtaining the gradient is computationally hard. Even worse, in some settings, the gradient is inaccessible by nature, for instance in bandit problems \citep{shamir2017optimal}, black-box adversarial attacks\citep{tu2019autozoom,chen2017zoo,chen2019zo}, or reinforcement learning \citep{salimans2017evolution,mania2018simple,choromanski2020provably}.
To address these challenges, zeroth-order (ZO) optimization methods have been developed\citep{nesterov2017random}. These methods commonly replace the inaccessible gradient with its finite difference approximation which can be calculated by simply using the function evaluations. Subsequently, ZO methods have been adapted to handle convex constraint sets, rendering them suitable for solving the $\ell_1$ convex relaxation of the problem (\ref{equ:queation1})\citep{liu2018zeroth,balasubramanian2018zeroth}. However, it's essential to highlight that in the context of sparse optimization, $\ell_1$ regularization or constraints can introduce substantial estimation bias and result in inferior statistical properties when compared to $\ell_0$ regularization and constraints\citep{fan2001variable,zhang2010nearly}. 

To tackle this issue, a recent development introduced the stochastic zeroth-order hard-thresholding algorithm (SZOHT)\citep{de2022zeroth}, specifically designed for $\ell_0$ sparsity constraints and gradient-free optimization. Unfortunately, as the only available algorithm in zeroth-order hard-thresholding so far, SZOHT  has notable limitations due to the inherent conflict between the deviation of ZO estimators and the expansivity of the hard-thresholding.  This limitation makes the algorithm difficult to use in practice, and a natural question is proposed: Could we have a simple ZO hard-thresholding algorithm whose convergence does not rely on the number of $q$ (the number of random directions used to estimate the gradient, further defined in Section \ref{section2})?

In this paper, we provide a positive response to this question. Our approach centers on the role of variance in addressing this problem. We firmly believe that variance reduction can offer a dual benefit. It not only holds the potential to accelerate convergence speed but, more importantly, it can effectively mitigate the unique conflicts associated with zero-order hard-thresholding.
From this perspective, SZOHT is characterized by its limitation in restricting the sampling of zero-order gradients, essentially representing an incomplete approach to variance reduction. This incompleteness leads to strict conditions for SZOHT. In contrast, we have developed better algorithms by using historical gradients to reduce variance thoroughly.
We then provide the convergence and complexity analysis for the generalized variance reduce algorithm under the standard assumptions of sparse learning, which are restricted strong smoothness (RSS), and restricted strong convexity (RSC) \citep{nguyen2017linear,shen2017tight} to retain generality. These algorithms eliminate the restrictions on zero-order gradient steps, leading to improved convergence rates and broader applicability.  Crucial to our analysis is to provide how variance reduction mitigates contradictions on the parameters $q$ and $k$. Finally, we demonstrate the effectiveness of our method by applying it to both ridge regression problems and black-box adversarial attacks. Our results highlight that our method can achieve competitive performance when compared to state-of-the-art methods for zeroth-order algorithms designed to enforce sparsity.

     \begin{figure*}[htbp]
  \centering
  \includegraphics[scale=0.40]{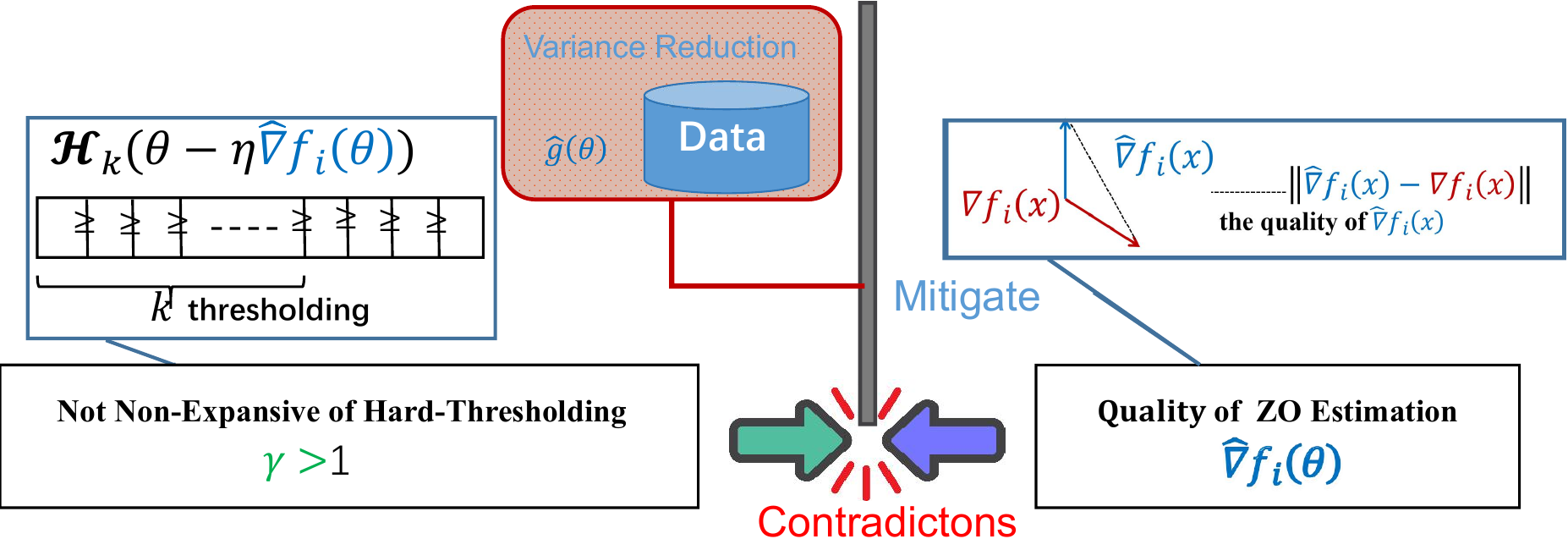}
   \vspace*{-4pt}
   \caption{Motivation of our algorithm.}
   \label{fig:intuitive}
   \vspace*{-12pt}
 \end{figure*}

The majority of our work can be summarized in three parts:
\begin{enumerate}
\item  New Perspective on Resolving Conflicts Between Zeroth-Order Methods and Hard-Thresholding. Our paper acknowledges the necessity of mitigating this contradiction, emphasizing the demand for a more flexible and resilient approach. By employing the perspective of variance to analyze this issue, our paper presents a more practical and effective solution.

\item Variance Reduction: Another key innovation presented in the paper is the introduction of variance reduction. This concept provides a unique solution to $\ell_0$-constrained zeroth-order optimization. By employing data-driven techniques to reduce variance, the paper not only enhances the algorithm's convergence but also expands its utility across a wider range of scenarios.

\item General Analysis: The introduction of a general analysis framework is another contribution to the paper. This framework systematically evaluates the performance and behavior of varying variance reduced algorithms under $\ell_0$-constraint  and ZO gradient.
\end{enumerate}
\section{Preliminaries}
\label{section2}

Throughout this paper, we use $\mathcal{k}\theta\mathcal{k}$ to denote the Euclidean norm for a vector, $\mathcal{k}\theta\mathcal{k}_\infty$ to denote the maximum absolute component of that vector, and $\mathcal{k}\theta\mathcal{k}_0$ to denote the $\ell_0$ norm (which is not a proper norm). The following two assumptions are widely adopted \citep{li2016nonconvex,nguyen2017linear} and are needed in this paper. 
\begin{assumption}[Restricted strong convexity (RSC) \citep{li2016nonconvex,nguyen2017linear}] \label{assumption1} A differentiable function $\mathcal F$ is restricted $\rho^-_s$-strongly convex at sparsity $s$ if there exists a generic constant $\rho^-_s>0$ such that for any $\theta$, $\theta'\in\mathbb R^d$ with $\mathcal{k}\theta-\theta'\mathcal{k}_0\le s$, we have:
\begin{equation}\label{d1}
\mathcal{F}(\theta)-\mathcal{F}(\theta')-\left<\nabla\mathcal{F}(\theta'),\theta-\theta'\right>\geq\frac{\rho^-_s}{2}\mathcal{k}\theta-\theta'\mathcal{k}^2_2.
\end{equation}
\end{assumption}
\begin{assumption}[Restricted strong smoothness (RSS) \citep{li2016nonconvex,nguyen2017linear}] For any $i\in [n]$, a differentiable function $f_i$ is restricted $\rho^+_s$-strongly smooth at sparsity level $s$ if there exists a generic constant $\rho^+_s>0$ such
that for any $\theta$, $\theta'\in\mathbb R^d$ with $\mathcal{k}\theta-\theta'\mathcal{k}_0\le s$, we have
\[\mathcal{k}\nabla f_i(\theta)-\nabla f_i(\theta')\mathcal{k}\le\rho^+_s\mathcal{k}\theta-\theta'\mathcal{k}.\]
\end{assumption}

 We assume that the objective function $\mathcal F (\theta)$ satisfies the RSC condition and that each component function ${f_i(\theta)}^n_{i=1}$ satisfies the RSS condition. We also define the restricted condition number as $\kappa_s=\rho^+_s/\rho^-_s$. This assumption ensures that the objective function behaves like a strongly convex and smooth function over a sparse domain, even when it is non-convex.
\subsection{ZO Estimate}\label{sec:2.2}
Then, we give our zeroth-order gradient estimator below adopted by \citep{de2022zeroth}: 
\begin{equation}\label{equation:ZO}
\hat\nabla f(\theta)= \frac{d}{q\mu}\sum^q_{i=1}(f(\theta+\mu\bm{u}_i)-f(\theta))\bm u_i,
\end{equation}
where each random direction $\bm u_i$ is a unit vector sampled uniformly from the set $\{\bm u\in \mathbb{R}^d : \mathcal{k}\bm u\mathcal{k}_0 \le s_2, \mathcal{k}\bm u\mathcal{k} = 1\}$, $q$ is the number of random unit vectors, and $\mu >0$ is a constant called the \textit{smoothing radius} (typically taken as small as possible, but no too small to avoid numerical errors). To obtain these vectors, we can first sample a random set of coordinates $S$ of size $s_2$ from $[d]$. Following, we sample a random vector $\bm u$ supported on $S$, in other words, uniformly sampled from the set $\{u\in \mathbb{R}^d: \bm u_{[d]-S}=\bm 0,\mathcal{k}\bm u\mathcal{k} = 1\}$. Especially, if $s_2 = d$, the general estimator is the usual vani·lla estimator with uniform smoothing on the sphere \citep{2018On}. 
Additionally, for convenience, 
we define $\mathcal{I}^*=supp(\theta^*)$ as the support of $\theta^*$. Let~$\theta^{(r)}$ be a sparse vector with $\mathcal{k}\theta^{(r)}\mathcal{k}_0 \leq k$ and support $\mathcal{I}^{(r)}=supp(\theta^{(r)})$.  Define, {with $\mathcal{H}_{2k}(\cdot)$ the hard-thresholding operator of sparsity $2k$:
$$ \widetilde{\mathcal{I}}= supp(\mathcal{H}_{2k}(\hat\nabla\mathcal{F}(\theta^*))\cup supp(\theta^*).$$
and let $\mathcal{I}=\mathcal{I}^{(r)}+\mathcal{I}^{(r+1)}+\widetilde{\mathcal{I}}$.
$\varepsilon_\mu={\rho^+_{s}}^2sd$, $\varepsilon_{\mathcal{I}}=\frac{2d}{q(s_2+2)}\left(\frac{(s-1)(s_2-1)}{d-1}+3\right)+2$, $\varepsilon_{\mathcal{I}^c}=\frac{2d}{q(s_2+2)}\left(\frac{s(s_2-1)}{-1}\right)$, $\varepsilon_{abs}=\frac{2d{\rho^+_s}^2ss_2}{q}\left(\frac{(s-1)(s_2-1)}{d-1}+1\right)+{\rho^+_s}^2sd$.
\subsection{Revisit of SZOHT}
Based on this assumption and ZO estimation, de Vazelhes proposed the SZOHT algorithm \citep{de2022zeroth}. The iteration relationship of this algorithm is: 
$$\theta^{(r+1)}=\mathcal{H}_k(\theta^{(r)}-\eta \hat\nabla\mathcal{F}(\theta^{(r)}))$$
where $\mathcal{H}_k(\cdot)$ is the hard-thresholding operator and $\hat\nabla\mathcal{F}(\theta^{(r)})$ is ZO gradient estimate defined by (\ref{equation:ZO}), 
$\eta$ represents learning rate. SZOHT can address some ZO $\ell_0$-constrained problems under specific conditions. However, it's important to note that the hard-thresholding operator, unlike the projection onto the $\ell_1$ ball, lacks non-expansiveness. Consequently, it has the potential to divert the algorithm's iteration away from the desired solution. To deal with this challenge, SZOHT imposes stringent limitations on both $k$ (hard-thresholding coefficients) and $q$. That is,
    $$\left(1-\frac{{\rho^-_s}^2}{(4\varepsilon_\mathcal{I}+1){\rho^+_s}^2}\right)\frac{k^*(4\varepsilon_\mathcal{I}+1)^2{\rho^+_s}^4}{{\rho^-_s}^4}\le k \le \frac{d-k^*}{2}$$
    and
    \begin{itemize}
        \item \textbf{if} $s_2>1$: $q\ge \frac{16d(s_2-1)k^*\kappa^2}{(s_2+2)(d-1)}\left[18\kappa^2-1+2\sqrt{9\kappa^2(9\kappa^2-1)+\frac{1}{2}-\frac{1}{2k^*}+\frac{3}{2}\frac{d-1}{k^*(s_2-1)}} \right]$\\
        \item \textbf{if} $s_2=1$: $q\ge\frac{8\kappa^2d}{\sqrt{\frac{d}{k^*}+1}}$
    \end{itemize}

Evidently, these conditions are exceedingly stringent and may not be suitable for numerous real-world problems. Therefore, we urgently need an algorithm with fewer constraints.

\section{General Analysis with variance}

In this section, we will analyze the random ZO hard-thresholding algorithm from the perspective of variance and provide a positive response to the above questions. These algorithms can be described using the following general iterative expression:
\begin{equation}\label{equ:itration}
    \theta^{(r+1)}=\mathcal{H}_k(\theta^{(r)}-\eta \hat g^{(r)}(\theta^{(r)})),
\end{equation}

where $\hat g^{(r)}(\theta^{(r)})$ is the generalized gradient estimate (applicable to all ZO hard-thresholding algorithms). Let $\alpha=1+\frac{2\sqrt{k^*}}{\sqrt{k-k^*}}$. Then, we have:
\begin{theorem}
Assume that  each $f_i$ is $(\rho_{s'}^+,s')-$RSS and that $\mathcal{F}$ is $(\rho^-_s, s)-$RSC. For any stochastic ZO hard-thresholding algorithm capable of expressing its iterative relationships as described in (\ref{equ:itration}), we can establish the following:
    \begin{equation}\label{equ:nht}
    \begin{split}
\mathbb{E}||\theta^{(r+1)}-\theta^*||^2_2&\le (1+\eta^2{\rho^-_s}^2)\alpha\mathbb{E}||\theta^{(r)}-\theta^*||^2_2+\eta^2\alpha\mathbb{E}||\hat{g}^{(r)}_{\mathcal{I}}(\theta^{(r)})||^2_2-2\eta\alpha\left[\mathcal{F}(\theta^{(r)})-\mathcal{F}(\theta^*)\right]\\
&+\alpha\frac{n^2\varepsilon_\mu \mu^2}{{\rho^-_s}^2}
\end{split}
\end{equation}
\end{theorem}
\begin{remark}
    Differing from the approach in \citep{yuan2017gradient,nguyen2017linear,de2022zeroth}, where the convergence inequality is segregated into linear convergence terms (represented as $(1+\eta^2{\rho^-s}^2)\alpha\mathbb{E}||\theta^{(r)}-\theta^*||^2_2$ in (\ref{equ:nht}) and error terms (represented as $\alpha\frac{n^2\varepsilon\mu \mu^2}{{\rho^-_s}^2}-2\eta\alpha\left[\mathcal{F}(\theta^{(r)})-\mathcal{F}(\theta^*)\right]$ in (\ref{equ:nht}), we have introduced the gradient squared term $\eta^2\alpha\mathbb{E}||\hat{g}^{(r)}(\theta^{(r)})||^2_2$ to elucidate the role of variance better. We can transform (\ref{equ:nht}) into the form of \citep{yuan2017gradient,nguyen2017linear,de2022zeroth} by establishing an upper bound for the gradient squared term, which is often feasible for specific algorithms.

\end{remark}

    \textbf{Conflict analysis through variance.} It is worth noting that among these three components, only the gradient squared term $\eta^2\alpha\mathbb{E}\|\hat{g}^{(r)}(\theta^{(r)})\|^2_2$ encompasses both the hard-thresholding parameter (included by $\alpha$) and the ZO gradient parameter (included by $\|\hat{g}^{(r)}(\theta^{(r)})\|^2_2$). In essence, this means that the conflict between expansivity and zeroth-order error can be fully encapsulated through the gradient squared term. More importantly, when our attention is directed towards the gradient squared term, we discover that in cases where the gradient estimation is unbiased, we obtain $\mathbb{E}\| \hat{g}^{(r)}(\theta^{(r)})\|^2_2=\Var\|\hat{g}^{(r)}(\theta^{(r)})\|_2+\left \|\nabla \mathcal{F}(\theta^{(r)})\right \|^2$, which means that $\mathbb{E} \| \hat{g}^{(r)}(\theta^{(r)}) \|^2_2$ only related to the variance of gradient estimation. \textit{This indicates that the conflict between the expansionary of hard-thresholding and ZO error is actually between hard-thresholding and the variance of gradient estimation.} In SZOHT, we have $\hat{g}^{(r)}(\theta^{(r)})=\hat\nabla \mathcal{F}(\theta^{(r)})$. Then, the gradient squared term becomes $\eta^2\alpha\mathbb{E}||\hat\nabla \mathcal{F}(\theta^{(r)})||^2_2$. In this scenario, to guarantee algorithm convergence, it becomes essential to ensure that the gradient squared term remains within a reasonable upper bound. Due to the fact that $\alpha$ is already required to satisfy certain conditions (which are generated by linear convergence terms and error terms), therefore, the sampling method for ZO gradients must be restricted, which leads to a reduction in the variance. However, due to the technique of sampling used to reduce the variance, the limitation on the number $q$ of random directions is introduced into SZOHT.

     \textbf{Improvement plan.} A natural idea is to use a more comprehensive variance reduction approach instead of only using sampling technique to reduce $\mathbb{E}||\hat{g}^{(r)}(\theta^{(r)})||^2_2$, which could  effectively alleviate the conflict between ZO estimation and hard-thresholding, ultimately enabling the design of algorithms with fewer constraints, broader applicability, and enhanced convergence speed. Based on this perspective, we have developed a generalized variance reduction ZO hard-thresholding algorithm that leverages historical gradients. We will provide a detailed explanation of this algorithm in the next section.
    \section{\textit{p}M-SZHT Algorithm Framework}
This section mainly presents the \textit{p}M-SZHT algorithm framework along with its convergence analysis. This framework encompasses the majority of unbiased stochastic variance-reduction ZO hard-thresholding methods, providing a generalized result. Subsequently, we introduce the VR-SZHT algorithm,  a special case under this framework. Additionally, we extend our discussion by introducing SARAH-ZHT (please note that the gradient estimate in this algorithm is biased) and providing its convergence analysis in the appendix.
    \subsection{\textit{p}M-SZHT} \label{sec:pmzht}
    We now present our generalized algorithm to solve the target problem (\ref{equ:queation1}), which we name $p$M-SZHT ($p$ Memorization Stochastic Zeroth-Order Hard-Thresholding). Each iteration of our algorithm is composed of two steps: (i) the gradient estimation step, and (ii) the hard thresholding step, where the gradient estimation step includes the variance reduce estimation and zeroth-order estimation. We give the full formal description of our algorithm in Algorithm (\ref{alg:qM}).

    In the gradient estimation step, we are utilizing the $p$-Memorization framework, which was originally proposed by Hofmann \citep{2015Variance} to analyze the sequential stochastic gradient algorithm for convex and smooth optimization problems. It's worth noting that our gradient estimation can be seen as its zeroth-order variant (the zeroth-order estimation is shown in Section 2.2). Here, we select in each iteration a random index set $J \subseteq [n]$ of memory locations to update according to:
\[ \forall j \in [n]: ~
\hat a^+_j:=\left\{\begin{aligned}
\hat\nabla f_j(\theta), \ \  \textrm{if} \ j\in J \\
\hat a_j,\ \ \ \ \ \  \textrm{otherwise}
\end{aligned}\right.
    \]
such that any $j$ has the same probability of $p/n$ being updated\footnote{Originally, $p$-Memorization is called $q$-Memorization. We change it to $p$ to avoid conflicting with random directions in zeroth order}, where $p$ is the number of directions updated each time (see \citep{2015Variance}). The value of $p$ set $J$, $\forall j,\sum_{J\ni j} \bm{P}\{J\}-\frac{p}{n}$. Its probability is determined by some specific algorithm. For example, if $\bm{P}\{J\}=1/\tbinom{n}{p}$ if $|J|=p$, and $\bm{P}\{J\}=0$ otherwise, we obtain the $p$-SAGA-ZHT algorithm. If $\bm{P}\{\emptyset\}=1-\frac{p}{n}$ and $\bm{P}\{[1:n]\}=\frac{p}{n}$, we obtain a variant of the VR-SZHT algorithm from Section~\ref{sec:svrg}. Those are the ZO hard-thresholding versions of the algorithms mentioned in \cite{2015Variance,gu2020unified}.

    In the hard thresholding step, we only keep the $k$ largest (in magnitude) components of the current iterate $\theta^{(r)}$. This ensures that all our iterates (including the last one) are $k$-sparse. This hard-thresholding operator has been studied for instance in \citep{shen2017tight}, and possesses several interesting properties. Firstly, it can be seen as a projection on the $\ell_0$ ball. Second, importantly, it is not non-expansive, contrary to other operators like the soft-thresholding operator \citep{shen2017tight}.

    \begin{algorithm}
    \caption{Stochastic variance reduced zeroth-order Hard-Thresholding with $p$-Memorization ($p$M-SZHT)}
    \label{alg:qM}
    \renewcommand{\algorithmicrequire}{\textbf{Input:}}   
    \renewcommand{\algorithmicensure}{\textbf{Output:}}    
    \begin{algorithmic}[1]
        \Require Learning rate $\eta$, maximum number of iterations $T$, initial point $\theta^{(0)}$, number of random directions $q$, and number of coordinates to keep at each iteration $k$.
        
        \Ensure $\theta^{(r)}$.
        
        \renewcommand{\algorithmicrequire}{\textbf{Parameters:}}
        
        \For{$r=1,\ldots,T$}
            \State Update $\hat a^{(r-1)}$
            \State Randomly sample $i_r\in \{1,2,\ldots,n\}$
            \State $\hat g^{(r-1)}(\theta^{(r-1)})= \hat\nabla f_{i_r}(\theta^{(r-1)})-\hat a^{(r-1)}_{i_r}+\frac{1}{n}\sum^n_{j=1}\hat a^{(r-1)}_j$

             \State $\theta^{(r)}=\mathcal{H}_k({\theta}^{(r-1)}-\eta g^{(r-1)}(\theta^{(r-1)}))$
        \EndFor
    \end{algorithmic} 
\end{algorithm}

\textbf{Convergence Analysis}: \
We provide the convergence analysis of \(p\)M-SZHT, using the assumptions from Section \ref{section2}, and demonstrate the correctness of the conclusions made in Section 3 by assessing whether the algorithm converges independently of \(q\).
\begin{theorem} \label{thm:1}
Suppose $\mathcal{F}(\theta)$ satisfies the RSC condition and that the functions $\{f_i(\theta)\}^n_{i=1}$satisfy the RSS condition with $s=2k+k^*$. 
For Algorithm \ref{alg:qM}, suppose that we run SZOHT with random supports of size $s_2$, $q$ random directions,
a learning rate of $\eta$, and $k$ coordinates kept at each iteration. We have:

\begin{equation}\label{equ:qM}
[\mathbb{E}\mathcal{F}(\theta^{(r+1)})-\mathcal{F}(\theta^*)] \le\gamma[\mathbb{E}\mathcal{F}(\theta^{(r)})-\mathcal{F}(\theta^*)]+2L_\mu+L_r
\end{equation}

here $L_\mu=\alpha\frac{n^2\varepsilon_\mu \mu^2}{{\rho^-_s}^2}+6\alpha\varepsilon_{abs}\mu^2+6\eta^2\alpha A_r$,  $L_r=\sqrt{s}||\nabla\mathcal{F}(\theta^*)||_{\infty}\mathbb{E}||{\theta}^{(r)}-\theta^*||_2+\eta^2(3\alpha((4\varepsilon_{\mathcal{I}}s+2)+\varepsilon_{\mathcal{I}^c}(d-k))\mathbb{E}\mathcal{k}\nabla f_{i_t}(\theta^*)\mathcal{k}^2_\infty)$, $\gamma=\left ( \frac{2\beta}{\rho^-_s}+48\eta^2\alpha \rho_s^+\varepsilon_{\mathcal{I}}-2\eta\alpha+1-\frac{p}{n} \right )$.
\end{theorem}
\begin{remark}(System error). This format of result is similar to the ones in  \citep{yuan2017gradient,nguyen2017linear,de2022zeroth}, the right of (\ref{equ:qM}) contains a linear convergence term $\gamma[\mathbb{E}\mathcal{F}(\theta^{(r)})-\mathcal{F}(\theta^*)]$, and system error $2L_\mu+L_r$. We note that if $\mathcal{F}$ has a $k^*$
-sparse unconstrained minimizer, which could happen in sparse reconstruction, or with overparameterized deep networks, then we would have $\mathcal{k}\nabla\mathcal{F}(\theta^*)\mathcal{k}_{\infty}= 0$ and $||\nabla f_{i_r}(\theta^*)||^2_\infty=0$, and hence that part of the system error $L_r$would vanish. In addition, we also have another system error $L_r$, which depends on the smoothing radius $\mu$, due to the error from the ZO estimate and the iterative method of $\hat a$.

\end{remark}
From this theorem, we know that if the algorithm converges, $\eta$ needs to lie in some specific interval.
\begin{corollary}
If 
\begin{equation}\label{equ:qMeta}
   \eta'-\frac{\sqrt{\Delta}}{2(48\varepsilon_\mathcal{I}\alpha\rho^+_s+\rho^-_s)}\le\eta\le    \max\left\{\eta'+\frac{\sqrt{\Delta}}{2(48\varepsilon_\mathcal{I}\alpha\rho^+_s+\rho^-_s)},\frac{1}{48\varepsilon_{\mathcal{I}}\rho^+_s}\right\} 
\end{equation}
algorithm \ref{alg:qM} converges. Here $\eta'=\frac{\alpha}{48\varepsilon_\mathcal{I}\alpha\rho^+_s+\rho^-_s}$, $\Delta=4\alpha^2-4(48\varepsilon_{\mathcal{I}}\alpha\rho^+_s+\rho^-_s)(1-\frac{p}{n}+\frac{2}{\rho^-_s})$.
\end{corollary}

\begin{remark}(Independence of $q$) 

When $k>k^*$, for any $q>0$ the necessary condition $\Delta>0$ for (\ref{equ:qMeta}) holds. We emphasize here that variance reduction can only make $q$ unable to determine whether to converge, but $q$ can still affect the convergence speed. In other words, variance reduction can mitigate the conflict, but cannot resolve it.

\end{remark}

    \subsection{VR-SZHT}\label{sec:svrg}
To offer a specific analysis, we introduce the VR-SZHT algorithm, which is the adaptation of the original SVRG method \cite{johnson2013accelerating} to our ZO hard-thresholding setting. In addition to the previously mentioned convergence analysis, we will also provide a complexity analysis to demonstrate the advantages of this algorithm, which extend beyond existing algorithm.
\begin{algorithm}[hb]
    \caption{Stochastic variance reduced zeroth-order Hard-Thresholding (VR-SZHT)}
    \label{alg:algorithm}
    \renewcommand{\algorithmicrequire}{\textbf{Input:}}   
    \renewcommand{\algorithmicensure}{\textbf{Output:}}    
    \begin{algorithmic}[1]
        \Require Learning rate $\eta$, maximum number of iterations $T$, initial point $\theta^{0}$, SVRG update frequency $m$, number of random directions $q$, and number of coordinates to keep at each iteration $k$.
        
        \Ensure $\theta^T$.
        
        \renewcommand{\algorithmicrequire}{\textbf{Parameters:}}
        
        \For{$r=1,\ldots,T$}
            \State $\theta^{(0)}  = \theta^{r-1}$;
            \State $\hat{\mu}=\frac{1}{n}\sum^n_{i=1}\hat\nabla f_{i}(\theta^{(0)})$;
            \For {$t=0,1,\ldots,m-1$}
                \State Randomly sample $i_t\in \{1,2,\ldots,n\}$;
                \State Compute ZO estimate $\hat\nabla f_{i_t}(\theta^{(r)})$, $\hat\nabla f_{i_t}(\theta^{(0)})$;
              \State  $\bm \bar{\theta}^{(r+1)}=\theta^{(r)}-\eta  ( \hat\nabla f_{i_t}(\theta^{(r)})-\hat\nabla f_{i_t}(\theta^{(0)})+  \hat{\mu} )  )$;
                \State $\theta^{(r+1)}=\mathcal{H}_k(\bm \bar{\theta}^{(r+1)})$;

            \EndFor
         \State   $\theta^{r}=\theta^{(r+1)}$, random $t'\in[m-1]$;
        \EndFor
    \end{algorithmic}
\end{algorithm}

\begin{theorem}\label{thm1}
Suppose $\mathcal{F}(\theta)$ satisfies the RSC condition and that the functions $\{f_i(\theta)\}^n_{i=1}$ satisfy the RSS condition with $s=2k+k^*$.When $\eta=\frac{\alpha \rho^-_S}{2(48\varepsilon_{\mathcal{I}}\alpha\rho^-_s\rho^+_s+{\rho^-_s}^2)}$, we have:
\begin{equation}
\delta\left[\mathcal{F}(\widetilde{\theta}^{(r)})-\mathcal{F}(\theta^*)\right] \le \gamma'\mathbb{E}[\mathcal{F}(\widetilde{\theta}^{(r-1)})-\mathcal{F}(\theta^*)]+L'_\mu+L.
\end{equation}
Here $\beta= (1+\eta^2{\rho_s^-}^2)\alpha$, $\delta=\frac{\beta^m-1}{\beta-1}(2\eta-48\varepsilon_{\mathcal{I}}\eta^2\rho^+_s)\alpha$, $\gamma'=(\frac{2\beta^m}{\rho^-_s}+\frac{48\eta^2\rho^+_s\varepsilon_{\mathcal{I}}\alpha(\beta^m-1)}{\beta-1})$, \\$L'_\mu=\frac{2\beta^m}{\rho^-_s}\sqrt{s}\mathcal{k}\nabla\mathcal{F}(\theta^*)\mathcal{k}_{\infty}\mathbb{E}\mathcal{k}\widetilde{\theta}^{(r-1)}-\theta^*\mathcal{k}_2+6\eta^2\frac{\beta^m-1}{\beta-1}\alpha ((4\varepsilon_{\mathcal{I}}s+2)+\varepsilon_{\mathcal{I}^c}(d-k))\mathbb{E}||\nabla f_{i_t}(\theta^*)||^2_\infty+3||\nabla_{\mathcal{I}}\mathcal{F}(\theta^*)||^2_2)$, and $L'=\frac{\beta^m-1}{\beta-1}\alpha (72\eta^2\varepsilon_{abs}\mu^2+\frac{n^2\varepsilon_\mu \mu^2}{{\rho^-_s}^2})$.
\end{theorem}
This theorem is similar to Theorem 2. And it is worth noting that $q$ is also independent in VR-SZHT, and can be found in the appendix due to space limitations.
\begin{corollary}
     The ZO query complexity of the algorithm is $\mathcal{O}\left([n+\frac{\kappa^3}{\kappa^2+1}]\log{(\frac{1}{\varepsilon})}\right)$. And the hard-thresholding query complexity is $\mathcal{O}\left(\log(\frac{1}{\varepsilon})\right)$.
 \end{corollary}
 When comparing VR-SZHT with SZOHT, where the ZO query complexity of SZOHT is \(\mathcal{O}\left((k+\frac{d}{s_2})\kappa^2\log{(\frac{1}{\varepsilon})}\right)\) and the hard-thresholding query complexity is \(\mathcal{O}\left(\kappa^2\log(\frac{1}{\varepsilon})\right)\), it becomes evident that the hard-thresholding query complexity of VR-SZHT is significantly reduced. Furthermore, as $k$ becomes large, the ZO complexity is also reduced.

\section{Experiments}\label{sec:syntreal}

We now compare the performance of VR-SZHT, SAGA-SZHT, and SARAH-SZHT (an adaptation of the SARAH variance reduction method \citep{nguyen2017sarah} to our ZO hard-thresholding setting, for which we provide the convergence analysis in Appendix~\ref{sec:sarah}) with that of the following algorithms, in terms of IZO (iterative zeroth-order oracle, i.e. number of calls to $f_i$) and NHT (number of hard-thresholding operations):
\begin{itemize}
\item SZOHT \citep{de2022zeroth}: a vanilla stochastic ZO hard-thresholding algorithm.
\item FGZOHT: the full gradient version of SZOHT.
\end{itemize}

\paragraph{Ridge Regression}

We first consider the following ridge regression problem, where malfunctions $f_i$ are defined as follows: 
$ f_i(\theta) = (x_i^{\top}\theta - y_i)^2 + \frac{\lambda}{2} \| \theta\|_2^2, $
where $\lambda$ is some regularization parameter. We generate each $x_i$ randomly from a unit norm ball in $\mathbb{R}^d$, and a true random model $\theta^*$ from a normal distribution $\mathcal{N}(0, I_{d \times d})$.  Each $y_i$ is defined as $y_i = x_i^{\top} \theta^*$. We set the constants of the problem as such: $n=10, d=5, \lambda = 0.5$. Before training, we preprocess each column by subtracting its mean and dividing it by its empirical standard deviation. We run each algorithm with $k=3,q=200, \mu=10^{-4}, s_2=d=5$, and for the variance reduced algorithms, we choose $m=10$. For all algorithms, the learning rate $\eta$ is found through grid-search in $\{0.005, 0.01, 0.05, 0.1, 0.5\}$: we choose the learning rates giving the lowest function value (averaged over several runs) at the end of training. We stop each algorithm once its number of IZO reaches 80,000. All curves are averaged over 3 runs, and we plot their mean and standard deviation in Figure \ref{fig:linreg}. As we can observe, SZOHT converges to higher function values than other algorithms: this illustrates the advantage of the variance reduction techniques, which can allow to attain smaller function values than plain SZOHT, but at a cheaper cost than FGZOHT.
     \begin{figure}[!bht]
  \centering
  \includegraphics[scale=0.25]{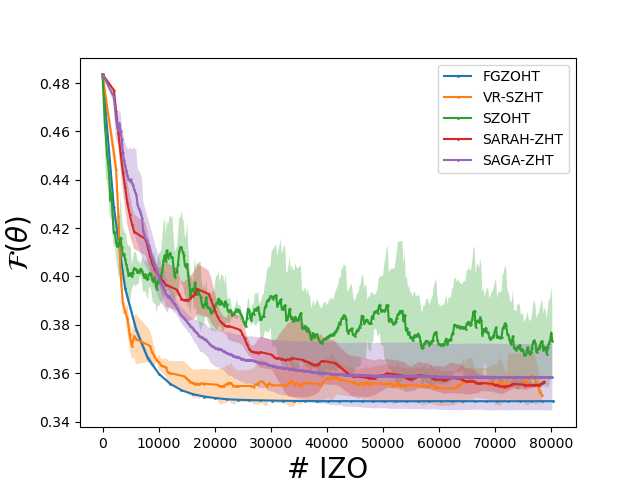}
    \includegraphics[scale=0.25]{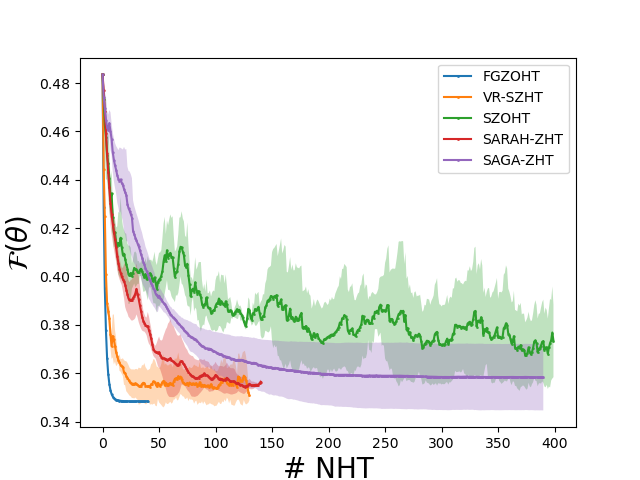}
   \caption{\#IZO and \#NHT on the ridge regression task.}
   \label{fig:linreg}
 \end{figure}

\begin{table*}[t!]
\caption{Comparison of universal adversarial attacks on $n=10$ images from the CIFAR-10 test-set, from the 'airplane' class. For each algorithm, the leftmost image is the sparse adversarial perturbation applied to each image in the row. ('auto' stands for 'automobile', and 'plane' for 'airplane') \vspace{0.3cm}} \label{table:CIFAR_airplane}
  \centering
  \begin{tabular}
      {ccccccccccc}
      \hline
      	Image ID & 3 & 27 & 44 & 90 & 97 & 98 & 111 & 116 & 125 & 153  \\
      \hline &&&&&&&&&& \vspace{-0.3cm} \\
      	Original &
        \parbox[c]{1.5em}{\includegraphics[width=0.30in]{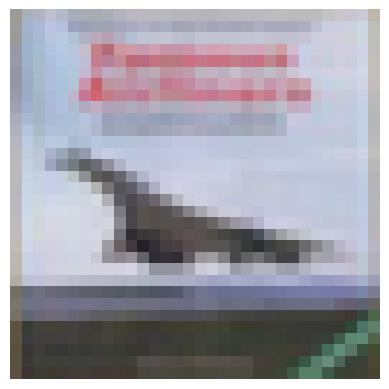}} &
        \parbox[c]{1.5em}{\includegraphics[width=0.30in]{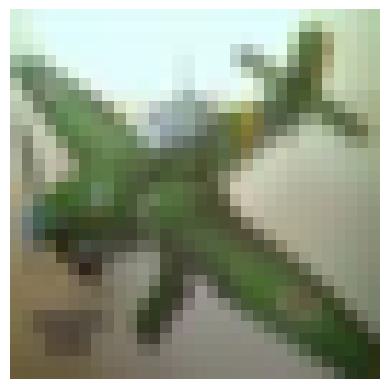}} &
        \parbox[c]{1.5em}{\includegraphics[width=0.30in]{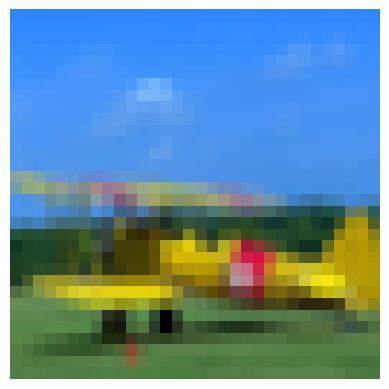}} &
        \parbox[c]{1.5em}{\includegraphics[width=0.30in]{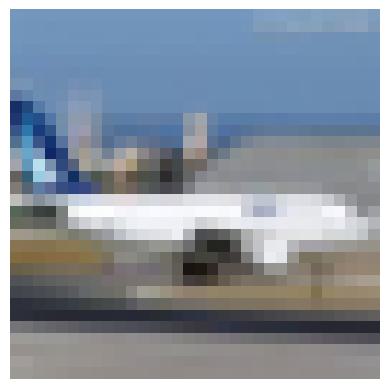}} &
        \parbox[c]{1.5em}{\includegraphics[width=0.30in]{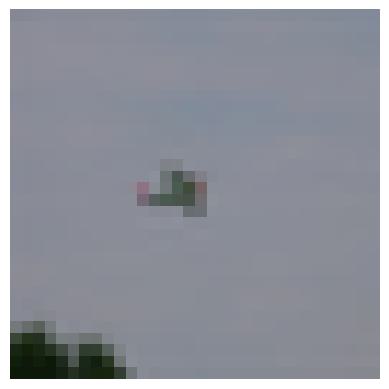}} &
        \parbox[c]{1.5em}{\includegraphics[width=0.30in]{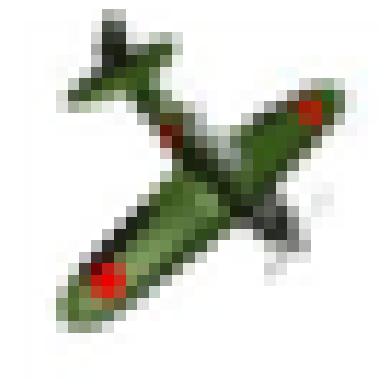}} &
        \parbox[c]{1.5em}{\includegraphics[width=0.30in]{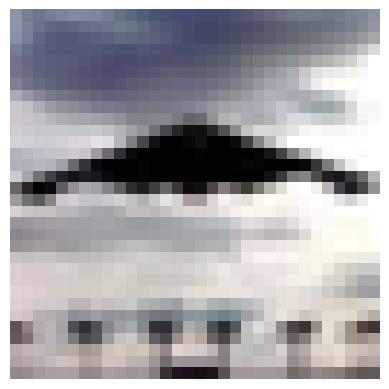}} &
        \parbox[c]{1.5em}{\includegraphics[width=0.30in]{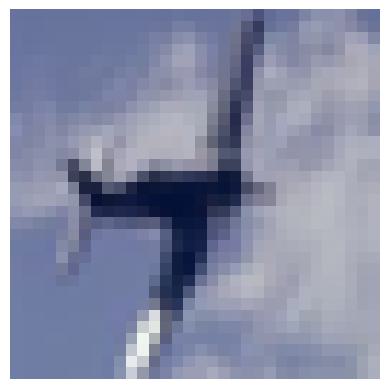}} &
        \parbox[c]{1.5em}{\includegraphics[width=0.30in]{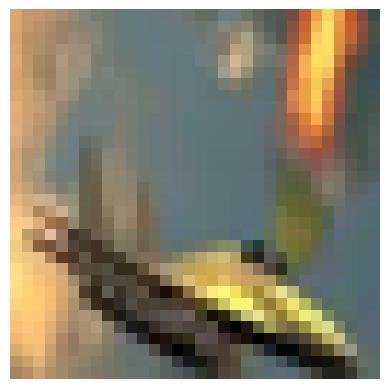}} &
        \parbox[c]{1.5em}{\includegraphics[width=0.30in]{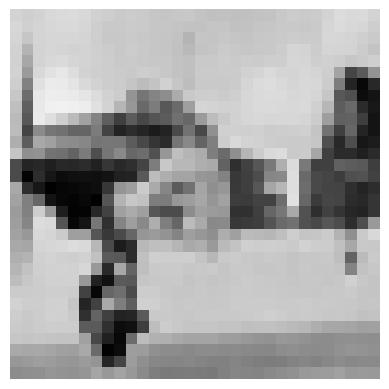}} \\
      \vspace{-0.2cm} \\
    \hline &&&&&&&&&& \vspace{-0.2cm} \\
      	FGZOHT &&&&&&&&&& \vspace{-0.2cm} \\
      &&&&&&&&&&\vspace{-0.2cm} \\
\parbox[c]{1.5em}{\includegraphics[width=0.30in]{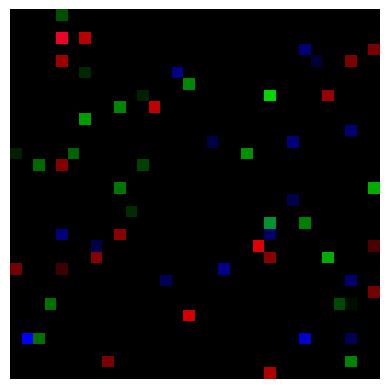}}  &
\parbox[c]{1.5em}{\includegraphics[width=0.30in]{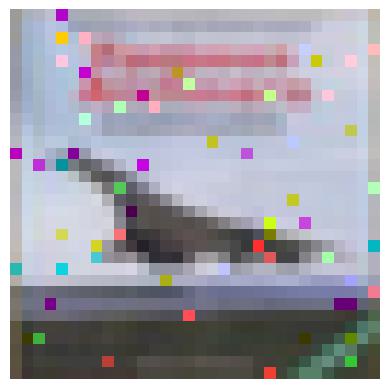}} &
\parbox[c]{1.5em}{\includegraphics[width=0.30in]{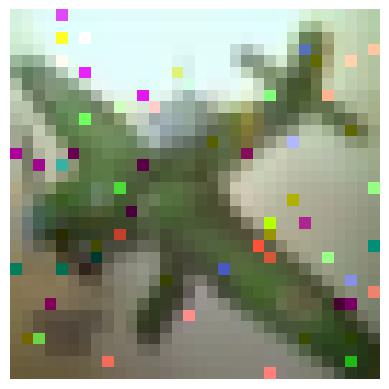}} &
\parbox[c]{1.5em}{\includegraphics[width=0.30in]{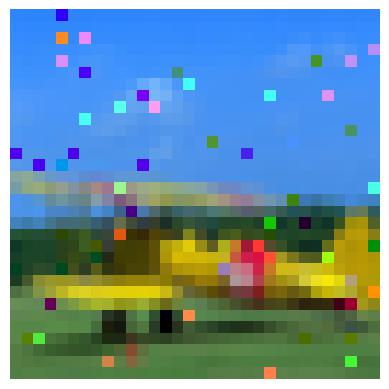}} &
\parbox[c]{1.5em}{\includegraphics[width=0.30in]{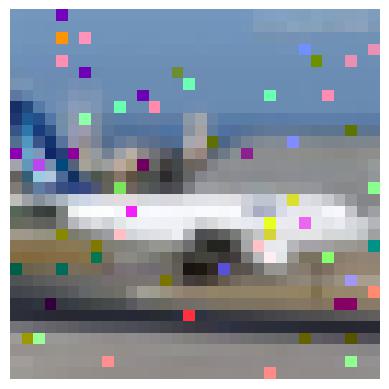}} &
\parbox[c]{1.5em}{\includegraphics[width=0.30in]{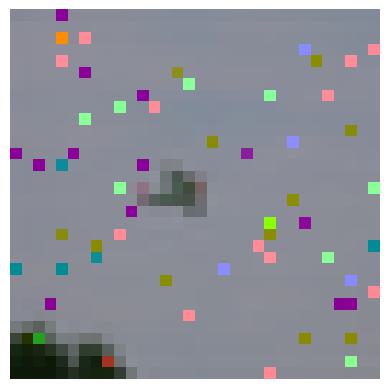}} &
\parbox[c]{1.5em}{\includegraphics[width=0.30in]{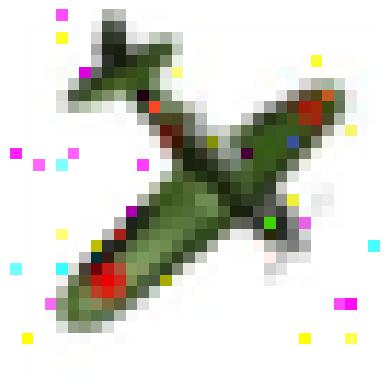}} &
\parbox[c]{1.5em}{\includegraphics[width=0.30in]{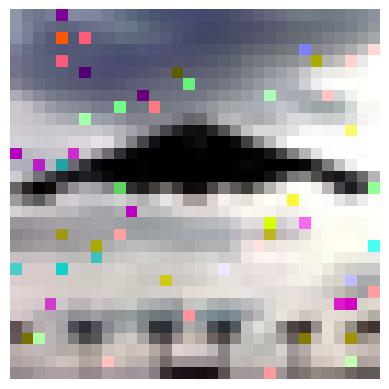}} &
\parbox[c]{1.5em}{\includegraphics[width=0.30in]{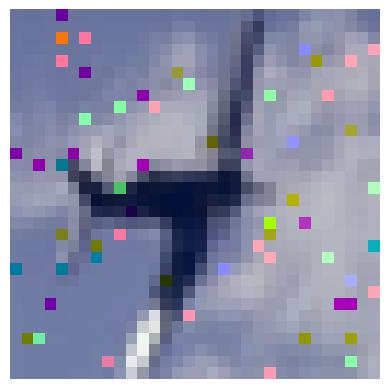}} &
\parbox[c]{1.5em}{\includegraphics[width=0.30in]{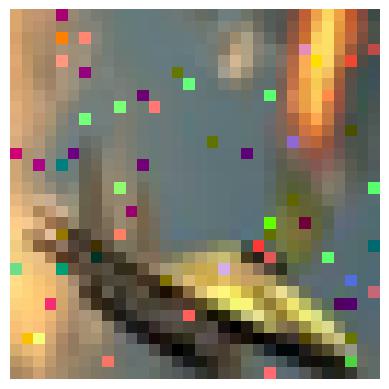}} &
\parbox[c]{1.5em}{\includegraphics[width=0.30in]{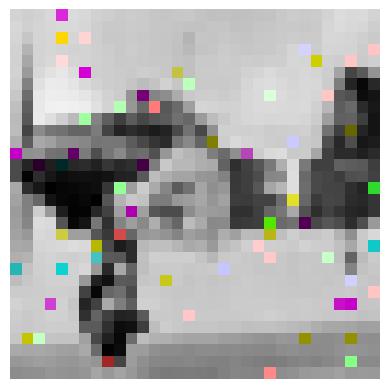}} \\
& plane & plane & plane & \textbf{ship} & \textbf{deer} & plane & plane & plane & \textbf{ship} & \textbf{truck} \\
      \hline
      
    \hline &&&&&&&&&& \vspace{-0.2cm} \\
      	SZOHT &&&&&&&&&& \vspace{-0.2cm} \\
      &&&&&&&&&&\vspace{-0.2cm} \\
\parbox[c]{1.5em}{\includegraphics[width=0.30in]{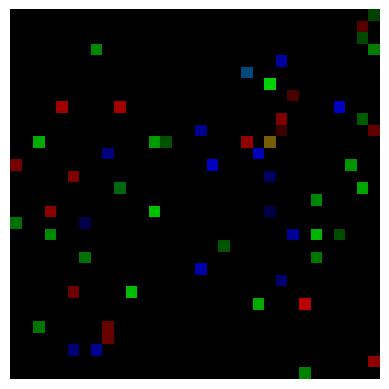}}  &
\parbox[c]{1.5em}{\includegraphics[width=0.30in]{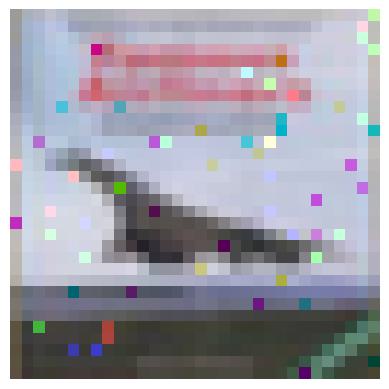}} &
\parbox[c]{1.5em}{\includegraphics[width=0.30in]{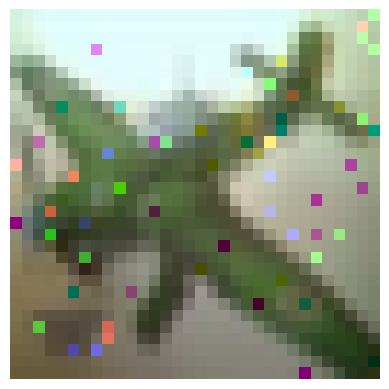}} &
\parbox[c]{1.5em}{\includegraphics[width=0.30in]{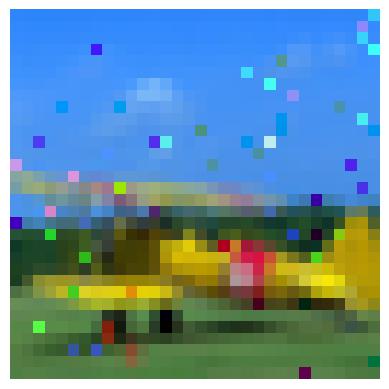}} &
\parbox[c]{1.5em}{\includegraphics[width=0.30in]{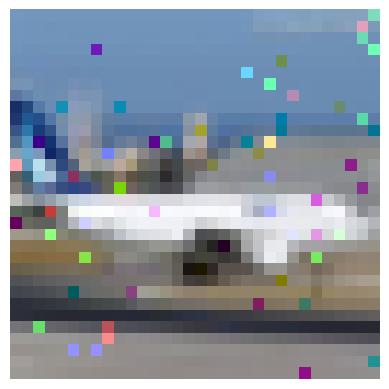}} &
\parbox[c]{1.5em}{\includegraphics[width=0.30in]{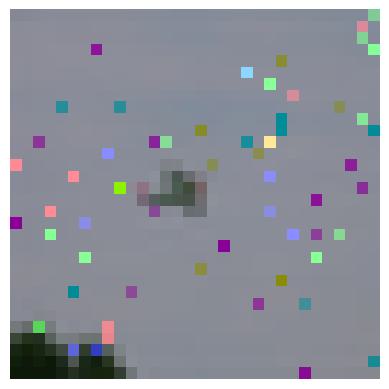}} &
\parbox[c]{1.5em}{\includegraphics[width=0.30in]{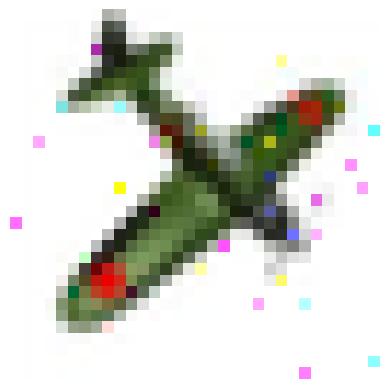}} &
\parbox[c]{1.5em}{\includegraphics[width=0.30in]{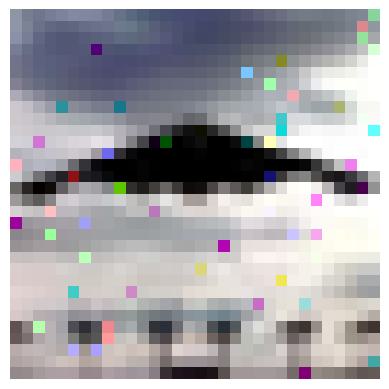}} &
\parbox[c]{1.5em}{\includegraphics[width=0.30in]{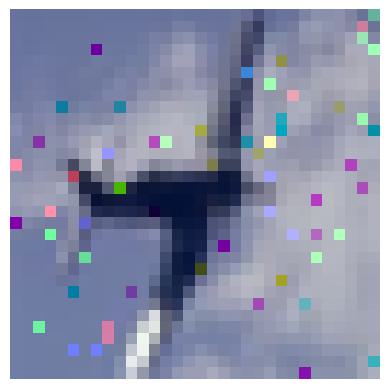}} &
\parbox[c]{1.5em}{\includegraphics[width=0.30in]{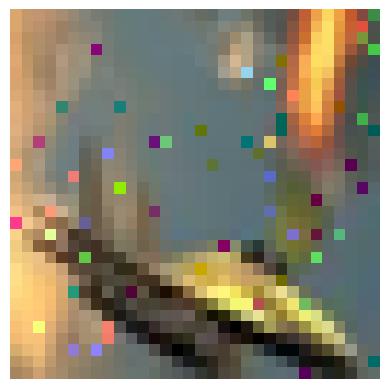}} &
\parbox[c]{1.5em}{\includegraphics[width=0.30in]{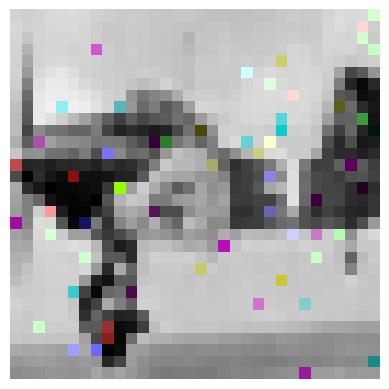}} \\
& plane & plane & plane & plane & \textbf{deer} & plane & \textbf{bird} & \textbf{bird} & \textbf{ship} & \textbf{truck} \\
         \hline

    \hline &&&&&&&&&& \vspace{-0.2cm} \\
      	VR-SZHT &&&&&&&&&& \vspace{-0.2cm} \\
      &&&&&&&&&&\vspace{-0.2cm} \\
\parbox[c]{1.5em}{\includegraphics[width=0.30in]{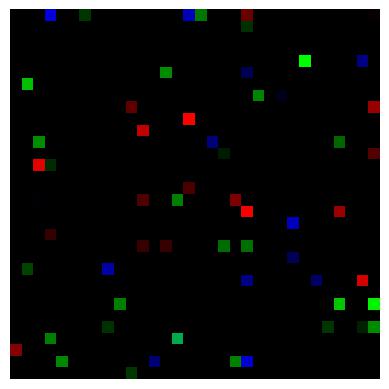}}  &
\parbox[c]{1.5em}{\includegraphics[width=0.30in]{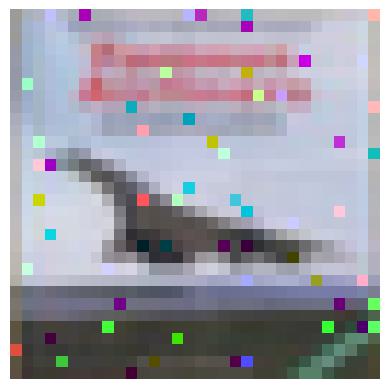}} &
\parbox[c]{1.5em}{\includegraphics[width=0.30in]{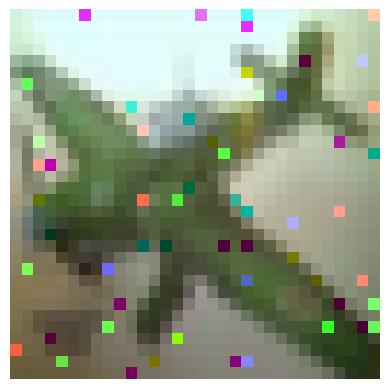}} &
\parbox[c]{1.5em}{\includegraphics[width=0.30in]{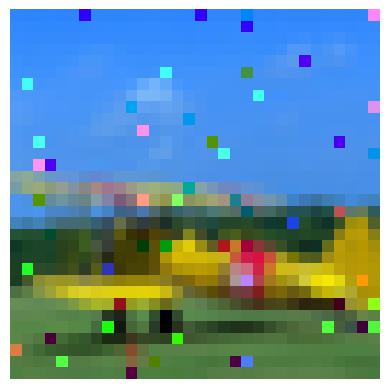}} &
\parbox[c]{1.5em}{\includegraphics[width=0.30in]{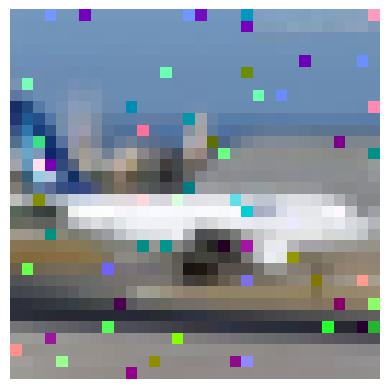}} &
\parbox[c]{1.5em}{\includegraphics[width=0.30in]{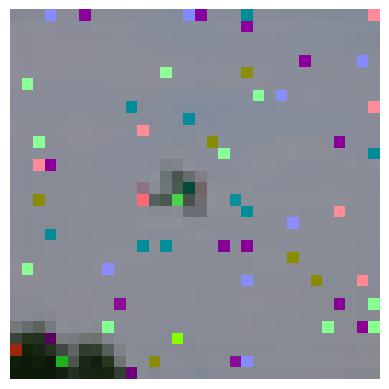}} &
\parbox[c]{1.5em}{\includegraphics[width=0.30in]{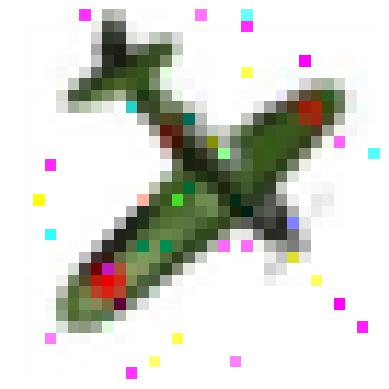}} &
\parbox[c]{1.5em}{\includegraphics[width=0.30in]{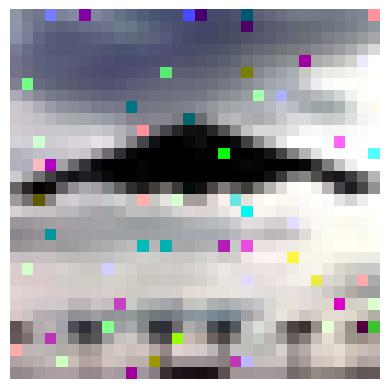}} &
\parbox[c]{1.5em}{\includegraphics[width=0.30in]{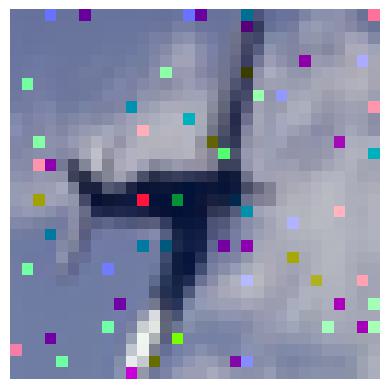}} &
\parbox[c]{1.5em}{\includegraphics[width=0.30in]{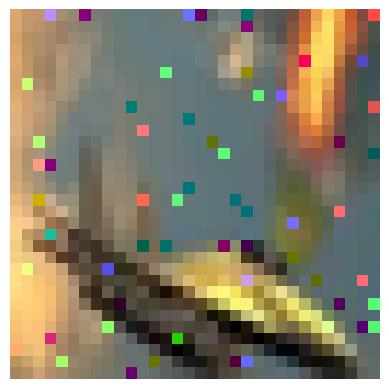}} &
\parbox[c]{1.5em}{\includegraphics[width=0.30in]{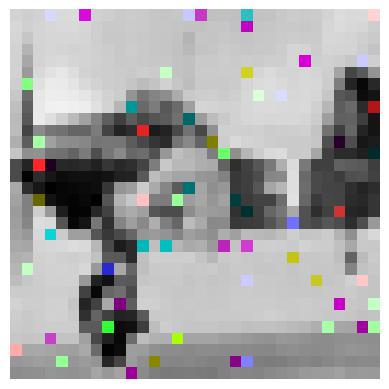}} \\
& plane & plane & \textbf{auto} & plane & \textbf{ship} & plane & plane & plane & \textbf{ship} & \textbf{truck} \\
         \hline

             \hline &&&&&&&&&& \vspace{-0.2cm} \\
      	SAGA-SZHT &&&&&&&&&& \vspace{-0.2cm} \\
      &&&&&&&&&&\vspace{-0.2cm} \\
\parbox[c]{1.5em}{\includegraphics[width=0.30in]{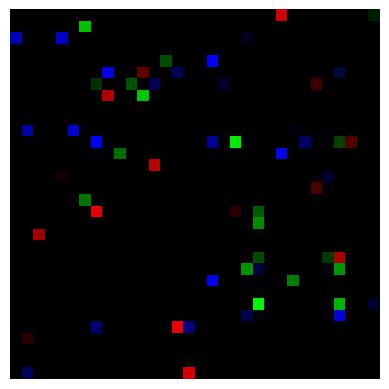}}  &
\parbox[c]{1.5em}{\includegraphics[width=0.30in]{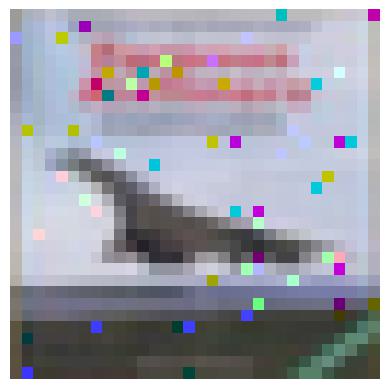}} &
\parbox[c]{1.5em}{\includegraphics[width=0.30in]{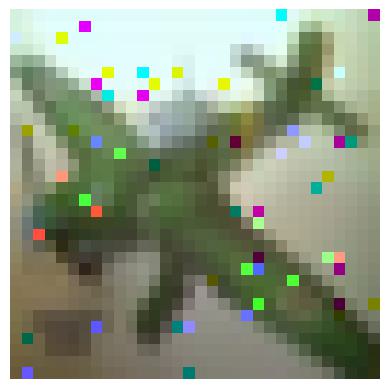}} &
\parbox[c]{1.5em}{\includegraphics[width=0.30in]{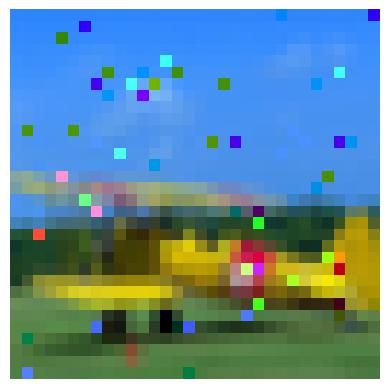}} &
\parbox[c]{1.5em}{\includegraphics[width=0.30in]{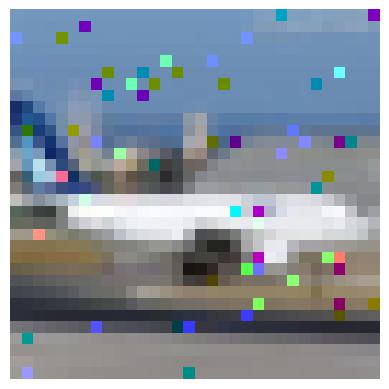}} &
\parbox[c]{1.5em}{\includegraphics[width=0.30in]{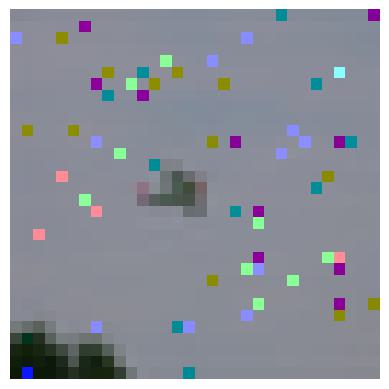}} &
\parbox[c]{1.5em}{\includegraphics[width=0.30in]{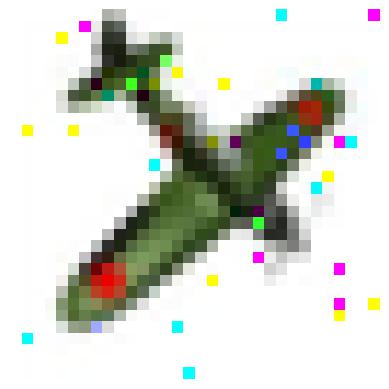}} &
\parbox[c]{1.5em}{\includegraphics[width=0.30in]{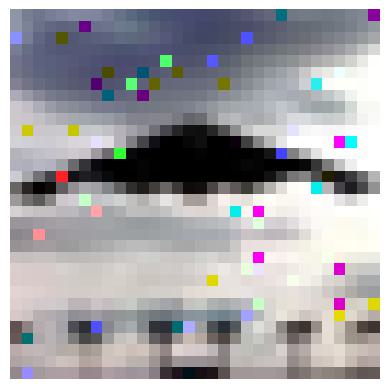}} &
\parbox[c]{1.5em}{\includegraphics[width=0.30in]{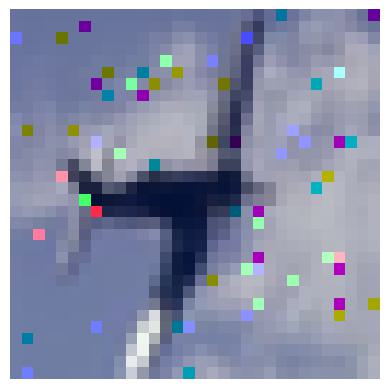}} &
\parbox[c]{1.5em}{\includegraphics[width=0.30in]{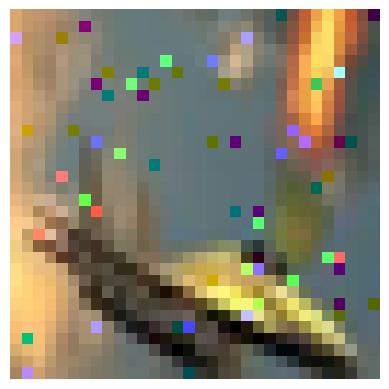}} &
\parbox[c]{1.5em}{\includegraphics[width=0.30in]{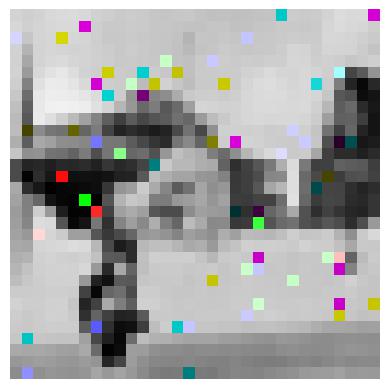}} \\
& plane & \textbf{frog} & plane & plane & \textbf{deer} & plane & plane & plane & \textbf{ship} & \textbf{ship} \\
         \hline

             \hline &&&&&&&&&& \vspace{-0.2cm} \\
      	SARAH-SZHT &&&&&&&&&& \vspace{-0.2cm} \\
      &&&&&&&&&&\vspace{-0.2cm} \\
\parbox[c]{1.5em}{\includegraphics[width=0.30in]{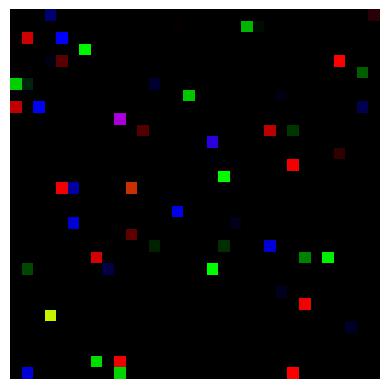}}  &
\parbox[c]{1.5em}{\includegraphics[width=0.30in]{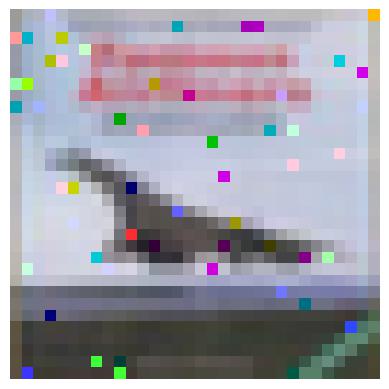}} &
\parbox[c]{1.5em}{\includegraphics[width=0.30in]{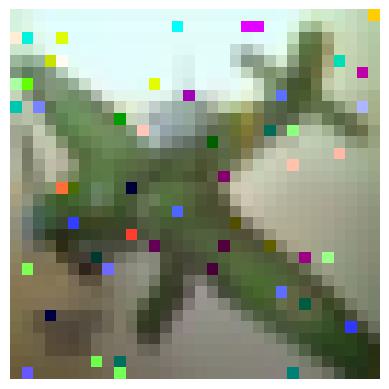}} &
\parbox[c]{1.5em}{\includegraphics[width=0.30in]{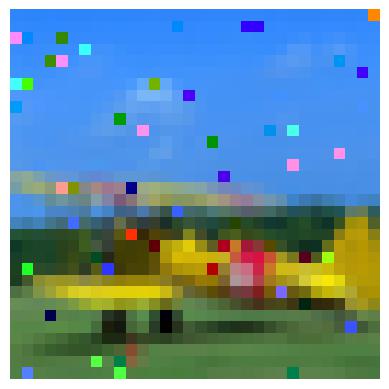}} &
\parbox[c]{1.5em}{\includegraphics[width=0.30in]{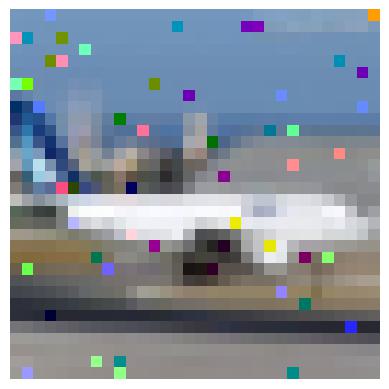}} &
\parbox[c]{1.5em}{\includegraphics[width=0.30in]{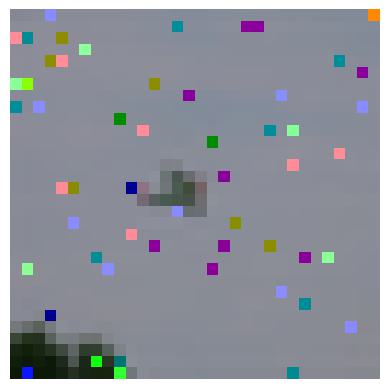}} &
\parbox[c]{1.5em}{\includegraphics[width=0.30in]{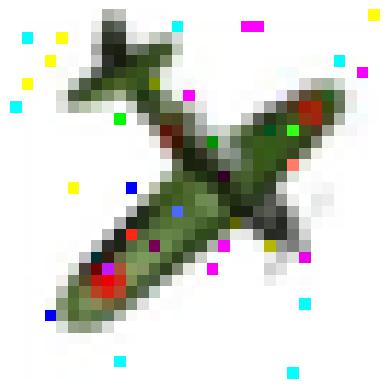}} &
\parbox[c]{1.5em}{\includegraphics[width=0.30in]{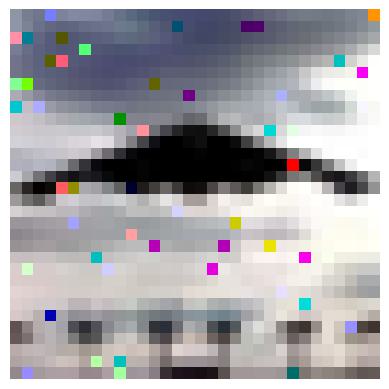}} &
\parbox[c]{1.5em}{\includegraphics[width=0.30in]{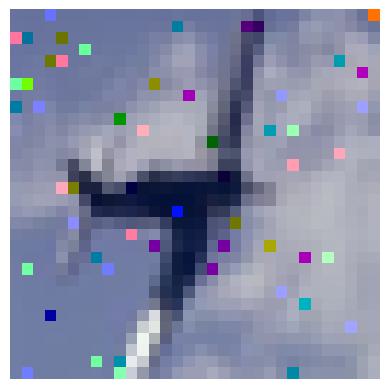}} &
\parbox[c]{1.5em}{\includegraphics[width=0.30in]{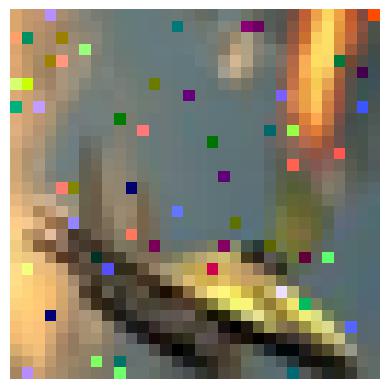}} &
\parbox[c]{1.5em}{\includegraphics[width=0.30in]{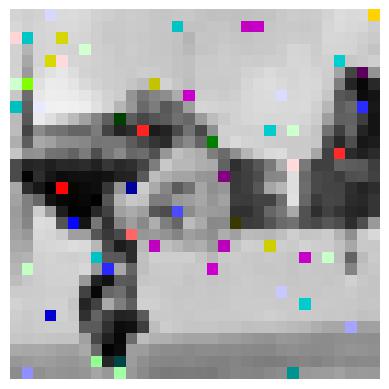}} \\
& plane & plane & \textbf{auto} & \textbf{ship} & \textbf{ship} & plane & \textbf{bird} & \textbf{plane} & \textbf{frog} & \textbf{truck} \\
         \hline
      
  \end{tabular}
\end{table*}

\begin{figure}[htpb]
  \centering
  \includegraphics[scale=0.25]{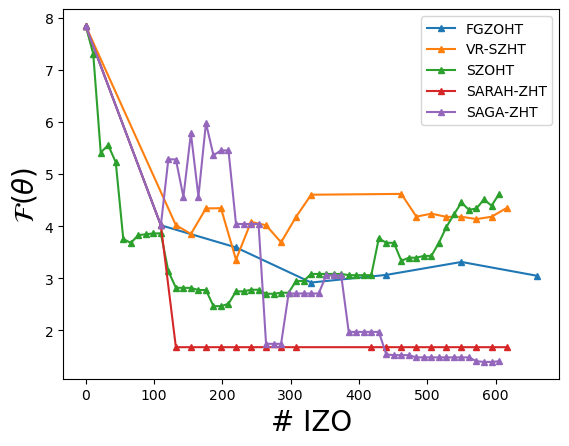}
   \includegraphics[scale=0.25]{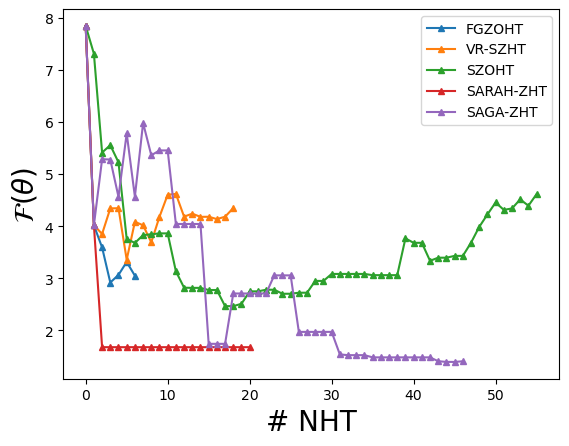}
     \caption{\#IZO and \#NHT on the few pixels adv. attacks (CIFAR-10), for the original class 'airplane'.}\label{fig:advat_cifar_airplane}
     \end{figure}

 \paragraph{Few Pixels Universal Adversarial Attacks}
Finally, we consider a few-pixel universal adversarial attacks problem. Let some classifier be trained on a dataset of images. We assume that it can only be accessed as a black box, i.e. it only returns the log probabilities of each estimated class, given an input image. This is a typical real-life scenario, where for instance the model can only be accessed through a provider's API. We seek to find a single perturbation $\theta \in \mathbb{R}^d$, to apply to several images at once,
(we denote those images by $x_i$, $i = \{1, \ldots,n \}$, and their true label as $y_i$)
to make the predicted class for those images different than their true class. Further discussion on universal perturbations can be found in \citep{dezfooli17}. In addition, we seek an adversarial perturbation that is sparse, to preserve as much as possible the original image. As is usual in black-box adversarial attacks, we maximize the following Carlini-Wagner loss \citep{carlini17,chen2017zoo}, which encourages the prediction from the model to be different from the true class:
\begin{align*}
f_i(\theta) =&  \max \{ F_{y_i}(\text{clip}(x_i + \theta)) - \max_{j \neq y_i}F_j(\text{clip}(x_i + \theta)), 0\},
\end{align*}
where $x_i$ is the original $i$-th image (rescaled to have values in $[-0.5, 0.5]$), of true class $y_i$, $\text{clip}$ denotes the clipping operation into $[-0.5, 0.5]$, $\theta$ is the universal perturbation that we seek to optimize, and each function $F_k$ outputs the log-probability of image $x_i$ being of class $k$ as predicted by the model, for $k \in \{1, .., K\}$, with $K$ the number of classes (similarly to \citep{chen2017zoo,liu2018zeroth,huang2019}). 
Similarly to \cite{liu2018zeroth} (Appendix A.11), we evaluate the algorithm on a dataset of $n=10$ images from the test-set of the CIFAR-10 dataset\citep{krizhevsky2009learning}, of dimensionality $32 \times 32 \times 3=3,072$, from the same class 'airplane', which we display in Table \ref{table:CIFAR_airplane}. We take as model $F$ a fixed neural network, already trained on the train-set of CIFAR-10, obtained from the supplementary material of \citep{de2022zeroth}.  We set $k=60$, $\mu=0.001$, $q=10$, $s_2=d=3,072$,  and the number of inner iterations of the variance reduced algorithms to $m=10$. We check at each iteration the number of IZO, and we stop training if it exceeds 600.  Finally, for each algorithm, we grid-search the learning rate $\eta$ in $\{0.001, 0.005, 0.01, 0.05\}$. The best learning rates (giving the curve which obtained the smallest minimum function value), are respectively: FGZOHT: 0.05, SZOHT: 0.005, VR-SZHT: 0.01, SAGA-SZHT: 0.05, SARAH-SZHT: 0.05. Our experiments are conducted on a workstation of 128 CPU cores.  The training curves are presented in Figure \ref{fig:advat_cifar_airplane}: SAGA-SZHT obtains the lowest function value at the end of the training, followed by SARAH-SZHT. In terms of attack success rate, SARAH-SZHT presents the highest success rate, as it has successfully attacked 7/10 images. We provide further results, on 3 more classes ('ship', 'bird', and 'dog') in the appendix, which demonstrate even further the advantage of variance reduction methods in our setting.

\section{Conclusion}
In this paper, we introduce a novel approach to address sparse zero-order optimization problems and leverage it to enhance existing algorithms. We perform a comprehensive convergence analysis of the generalized variance reduction algorithm, showcasing how variance reduction can effectively mitigate the limitations inherent in existing algorithms. To substantiate our claims, we validate our algorithm through experiments involving ridge regression and adversarial attacks.

{
\bibliographystyle{plainnat}
\bibliography{egbib}
}

\newpage
\section*{Contents of Appendices}
\setcounter{section}{0}

\section{Recall on various variance reductions algorithms}

In this section, we provide a short overview of our introduced variance-reduced algorithms, as well as the corresponding references in the first-order setting, with the corresponding space cost for an epoch and time cost for an epoch, similar to the Table 2 in \cite{gu2020unified} for improved clarity and to make our paper self-contained. The computational complexity of each algorithm is indeed similar to its first-order counterpart, when treating the sampling of the random directions at each iteration as a fixed cost.

\begin{table}[ht]
  \caption{
  {\small{Exposition of the variance reduction algorithms in our paper. $^1$: $p$-SAGA-SZHT, and SAGA-SZHT (i.e. $p$-SAGA-SZHT, with $p=1$) are both instantiations of pM-SZHT, as we mention in Section~\ref{sec:pmzht}.  $^2$: The pM-SZHT framework in Section~\ref{sec:pmzht} can also be specialized into a variant of SVRG, as described in Section~\ref{sec:pmzht}. However, since the original version of SVRG \cite{johnson2013accelerating} can often perform better than such variant in practice, we still have provided an independent analysis for the original SVRG  algorithm \cite{johnson2013accelerating} (adapted to our ZO hard-thresholding setting), in Section~\ref{sec:svrg}. $^3$:$p$-SAGA is originally called $q$-SAGA in \cite{2015Variance}, but we use $p$ instead to avoid confusion with the number of random directions $q$. $^4$: by \textit{Original algorithm}, we refer to the corresponding algorithm which was first introduced in the first-order setting. }}
  }
\centering
    \begin{tabular}{|c|c|c|c|c|}
      \hline
      \textbf{{Algorithm}} & \textbf{{Sec.}}  & \textbf{{Original algorithm $^4$}}  & \textbf{{Space cost for an epoch}} & \textbf{{Time cost for an epoch}} \\
      \hline
      {SAGA-SZHT$^1$} & {\ref{sec:pmzht}} & {SAGA} & {$O(dn)$} & {$O(d)$}\\
      {$p$-SAGA-SZHT$^1$} & {\ref{sec:pmzht}} & {$p$-SAGA$^3$} & {$O(dn)$} & {$O(pd)$}\\
      {VR-SZHT$^2$} & {\ref{sec:svrg}} & {SVRG} & {$O(d)$} & {$O(dn)$}\\
      {SARAH-SZHT} & {\ref{sec:sarah}} & {SARAH} & {$O(d)$} & {$O(dn)$}\\
      \hline
    \end{tabular}
  \label{tab:compalgos}
\end{table}

{Below we also sum-up the corresponding references of the original (i.e. first-order) versions of each of the variance reduction algorithms above: }

\begin{table}[ht]
  \caption{
  {\small{{Corresponding references for the first-order original versions of the variance reduction algorithms. $^1$ HT: hard-thresholding version of SVRG.}}}
  }
\centering
    \begin{tabular}{|c|c|}
      \hline
      \textbf{{Algorithm}} & \textbf{{References}}\\
      \hline
      {SAGA} & {\cite{defazio2014saga,2015Variance,gu2020unified} }\\
      {$p$-SAGA} & {\cite{2015Variance,gu2020unified}}\\
      {SVRG} & {\cite{johnson2013accelerating,2015Variance,gu2020unified}, HT$^1$: \cite{li2016nonconvex}} \\
      {SARAH} & {\cite{nguyen2017sarah}} \\
      \hline
    \end{tabular}
  \label{tab:algosref}
\end{table}

\section{Addition for Section 2}
 \begin{lemma}(Proof in \cite[Lemma 3.3]{li2016nonconvex}, )\label{lemma:HT}
    For $k>k^*$ and $\theta\in \mathbb{R}^d$, we have:
\begin{equation}\label{lemma4}
\mathcal{k}\mathcal{H}_k(\theta)-\theta^*\mathcal{k}^2_2\le\left(1+\frac{2\sqrt{k^*}}{\sqrt{k-k^*}}\right)\mathcal{k}\theta-\theta^*\mathcal{k}^2_2.
\end{equation}
\end{lemma}
\begin{remark}
    In fact, through assumption 2, we can  deduce another version of RSS conditions (\cite[Lemma 1.2.3]{2003Introductory} ), that is:
\begin{equation}\label{d2}
f_i(\theta)-f_i(\theta')-\left<\nabla f_i(\theta'),\theta-\theta'\right>\leq\frac{\rho^+_s}{2}\mathcal{k}\theta-\theta'\mathcal{k}^2_2.
\end{equation}
\end{remark}
\begin{proof}
    \noindent For all $\theta$, $\theta'\in \mathbb{R}^n$ we have 
        \[
        \begin{split}
        f(\theta')&=f(\theta)+\int^1_0 \left<f'(\theta+\tau(\theta'-\theta)),\theta'-\theta\right>d\tau \\
                &=f(\theta)+\left<f'(\theta),\theta'-\theta\right>+\int^1_0 \left<f'(\theta+\tau(\theta'-\theta))-f'(\theta),\theta'-\theta\right>d\tau
        \end{split}
        \]
Therefore
\[
    \begin{split}
        |f(\theta')-f(\theta)-\left<f'(\theta),\theta'-\theta\right>|&= \left |\int^1_0\left<f'(\theta+\tau(\theta'-\theta))-f'(\theta),\theta'-\theta\right>d\tau \right |\\
        &\le\int^1_0|\left<f'(\theta+\tau(\theta'-\theta))-f'(\theta),\theta'-\theta\right>|d\tau\\
        &\le\int^1_0\mathcal{k}f'(\theta+\tau(\theta'-\theta))-f'(\theta)\mathcal{k}\cdot\mathcal{k}\theta'-\theta\mathcal{k}d\tau\\
        &\le\int^1_0\tau\rho^+_s\mathcal{k}\theta'-\theta\mathcal{k}^2d\tau=\frac{\rho^+_s}{2}\mathcal{k}\theta'-\theta\mathcal{k}^2.
    \end{split}
\]
\end{proof}
\begin{lemma} 
(Proof in \cite{de2022zeroth}) Let us consider any support $\mathcal{I}\subset [d]$ of size $s(|\mathcal{I}|=s)$. For the ZO gradient estimator in (\ref{equation:ZO}), with $q$ random directions, and random supports of size $s_2$, and assuming that each $f_i$ is $(L_{s_2},s_2)$-RSS,  with $\hat\nabla_\mathcal{I} f(x) :=\mathcal{H}_k(\nabla f(\theta))$ on $\mathcal{I}$, we have:
\begin{itemize}[leftmargin=0.2in]
\item $\mathcal{k}\mathbb{E}\hat\nabla_\mathcal{I} f_i(x)-\nabla_\mathcal{I} f_i(x)\mathcal{k}^2\le \varepsilon_\mu\mu^2$.
\item $\mathbb{E}\mathcal{k}\hat\nabla_\mathcal{I} f_i(\theta)\mathcal{k}^2\le \varepsilon_\mathcal{I}\mathcal{k}\nabla_\mathcal{I}f_i(\theta)\mathcal{k}^2+\varepsilon_{\mathcal{I}^c}\mathcal{k}\nabla_{\mathcal{I}^c}f(\theta)\mathcal{k}^2+\varepsilon_{abs}\mu^2$.
\item $\mathbb{E}\mathcal{k}\hat\nabla_\mathcal{I} f_i(\theta)-\nabla f_i(x)\mathcal{k}^2\le 2(\varepsilon_\mathcal{I}+1)\mathcal{k}\nabla_\mathcal{I}f_i(\theta)\mathcal{k}^2+2\varepsilon_{\mathcal{I}^c}\mathcal{k}\nabla_{\mathcal{I}^c}f_i(\theta)\mathcal{k}^2+2\varepsilon_{abs}\mu^2$,
where $\varepsilon_\mu={\rho^+_{s}}^2sd$, $\varepsilon_{\mathcal{I}}=\frac{2d}{q(s_2+2)}(\frac{(s-1)(s_2-1)}{d-1}+3)+2$, $\varepsilon_{\mathcal{I}^c}=\frac{2d}{q(s_2+2)}\left(\frac{s(s_2-1)}{-1}\right)$, $\varepsilon_{abs}=\frac{2d{\rho^+_s}^2ss_2}{q}\left(\frac{(s-1)(s_2-1)}{d-1}+1\right)+{\rho^+_s}^2sd$.
\end{itemize}
\end{lemma}
\section{Proof of Section 3}
Here we prove the following inequality in Section 3:
\[
\begin{split}
\mathbb{E}||\theta^{(r)}-\eta \hat{g}^{(r)}_{\mathcal{I}}(\theta^{(r)})-\theta^*||^2_2&\le (1+\eta^2{\rho^-_s}^2)\mathbb{E}||\theta^{(r)}-\theta^*||^2_2+\eta^2\mathbb{E}||\hat{g}^{(r)}_{\mathcal{I}}(\theta^{(r)})||^2_2-2\eta\left[\mathcal{F}(\theta^{(r)})-\mathcal{F}(\theta^*)\right]\\
&+\frac{n^2\varepsilon_\mu \mu^2}{{\rho^-_s}^2}.\\
 \end{split}
\]
\begin{proof}
We denote $v=\theta^{(r)}-\eta g_{\mathcal{I}}^{(r)}(\theta^{(r)})$ and $\mathcal{I}=\mathcal{I}^*\cup\mathcal{I}^{(r)}\cup\mathcal{I}^{(r+1)}$, where $\mathcal{I}^*=supp(\theta^{*})$, $\mathcal{I}^{(r)}=supp(\theta^{(r)})$ and $\mathcal{I}^{(r+1)}=supp(\theta^{(r+1)})$
\begin{equation}
\begin{split}\label{A9}
\mathbb{E}||v-\theta^*||^2_2 &=\mathbb{E}||\theta^{(r)}-\theta^*||^2+\eta^2\mathbb{E}||\hat g^{(r)}_{\mathcal{I}}(\theta^{(r)})||^2_2-2\eta\left<\theta^{(r)}-\theta^*,\mathbb{E}\hat{g}^{(r)}_{\mathcal{I}}(\theta^{(r)})\right>\\
&=\mathbb{E}||\theta^{(r)}-\theta^*||^2_2+\eta^2\mathbb{E}||\hat{g}^{(r)}_{\mathcal{I}}(\theta^{(r)})||^2_2-2\eta\left<\theta^{(r)}-\theta^*,\mathbb{E}\hat\nabla_{\mathcal{I}}\mathcal{F}(\theta^{(r)})\right>\\
&=\mathbb{E}||\theta^{(r)}-\theta^*||^2_2+\eta^2\mathbb{E}||\hat{g}^{(r)}_{\mathcal{I}}(\theta^{(r)})||^2_2\\
&-2\eta\mathbb{E}\left<\theta^{(r)}-\theta^*,\hat\nabla_{\mathcal{I}}\mathcal{F}(\theta^{(r)})-\nabla_{\mathcal{I}}\mathcal{F}(\theta^{(r)})\right>-2\eta\mathbb{E}\left<\theta^{(r)}-\theta^*,\nabla_{\mathcal{I}}\mathcal{F}(\theta^{(r)})\right>\\
&\le\mathbb{E}||\theta^{(r)}-\theta^*||^2_2+\eta^2\mathbb{E}||\hat{g}^{(r)}_{\mathcal{I}}(\theta^{(r)})||^2_2-2\eta\mathbb{E}\left<\theta^{(r)}-\theta^*,\hat\nabla_{\mathcal{I}}\mathcal{F}(\theta^{(r)})-\nabla_{\mathcal{I}}\mathcal{F}(\theta^{(r)})\right>\\
&-2\eta\left[\mathcal{F}(\theta^{(r)})-\mathcal{F}(\theta^*)\right]\\
&=\mathbb{E}||\theta^{(r)}-\theta^*||^2_2+\eta^2\mathbb{E}||\hat{g}^{(r)}_{\mathcal{I}}(\theta^{(r)})||^2_2\\
&-2\eta\mathbb{E}\left<\sqrt{\eta}\rho^-_s(\theta^{(r)}-\theta^*),\frac{1}{\sqrt{\eta}\rho^-_s}(\hat\nabla_{\mathcal{I}}\mathcal{F}(\theta^{(r)})-\nabla_{\mathcal{I}}\mathcal{F}(\theta^{(r)}))\right>-2\eta\left[\mathcal{F}(\theta^{(r)})-\mathcal{F}(\theta^*)\right]\\
&\le(1+\eta^2{\rho^-_s}^2)\mathbb{E}||\theta^{(r)}-\theta^*||^2_2+\eta^2\mathbb{E}||\hat{g}^{(r)}_{\mathcal{I}}(\theta^{(r)})||^2_2-2\eta\left[\mathcal{F}(\theta^{(r)})-\mathcal{F}(\theta^*)\right]\\
&+\frac{1}{\rho^-_s}^2\mathbb{E}||(\hat\nabla_{\mathcal{I}}\mathcal{F}(\theta^{(r)})-\nabla_{\mathcal{I}}\mathcal{F}(\theta^{(r)}))||^2_2.\\
\end{split}
\end{equation}
For $\mathbb{E}||(\hat\nabla_{\mathcal{I}}\mathcal{F}(\theta^{(r)})-\nabla_{\mathcal{I}}\mathcal{F}(\theta^{(r)}))||^2_2$, we have:
\begin{equation}
\begin{split}\label{10}
\mathbb{E}||(\hat\nabla_{\mathcal{I}}\mathcal{F}(\theta^{(r)})-\nabla_{\mathcal{I}}\mathcal{F}(\theta^{(r)}))||^2_2&=||\mathbb{E}_{\bm u}(\hat\nabla_{\mathcal{I}}\sum^n_{i=1}{f_i}(\theta^{(r)})-\nabla_{\mathcal{I}}\sum^n_{i=1}{f_i}(\theta^{(r)}))||^2_2\\
&\le n\sum^n_{i=1}||\mathbb{E}_{\bm u}(\hat\nabla_{\mathcal{I}}{f_i}(\theta^{(r)})-\nabla_{\mathcal{I}}{f_i}(\theta^{(r)}))||^2_2\\
&\le n^2\varepsilon_\mu \mu^2.\\
\end{split}
\end{equation}
By constraining $\mathbb{E}||(\hat\nabla_{\mathcal{I}}\mathcal{F}(\theta^{(r)})-\nabla_{\mathcal{I}}\mathcal{F}(\theta^{(r)}))||^2_2$, we can turn (\ref{A9}) into:
\begin{equation}\label{11}
\begin{split}
\mathbb{E}||\theta^{(r)}-\eta \hat{g}^{(r)}_{\mathcal{I}}(\theta^{(r)})-\theta^*||^2_2&\le (1+\eta^2{\rho^-_s}^2)\mathbb{E}||\theta^{(r)}-\theta^*||^2_2+\eta^2\mathbb{E}||\hat{g}^{(r)}_{\mathcal{I}}(\theta^{(r)})||^2_2-2\eta\left[\mathcal{F}(\theta^{(r)})-\mathcal{F}(\theta^*)\right]\\
&+\frac{n^2\varepsilon_\mu \mu^2}{{\rho^-_s}^2}.\\
\end{split}
\end{equation} 
\end{proof}
\section{Proof of $p$M-SZHT}
\subsection{Proof of Theorem $1$}
Before providing the proof of Theorem $1$, we need the following lemma:
\begin{lemma}
    Suppose that the functions $f_i(x)$ satisfies the RSS condition with $s= 2k+ k^*$, For $q$M-SZHT, we can get:
    \begin{equation}
\begin{split}
    \mathbb{E}||\hat g^{(r)}(\theta^{(r)})||^2&\le \frac{24\rho^+_s}{n}\sum^{r-1}_{u=1}(1-\frac{q}{n})^{r-u-1}[\mathcal{F}(\theta^{(u)})-\mathcal{F}(\theta^*)]+24\rho^+_s(1-\frac{q}{n})^{r}[\mathcal{F}(\theta^{(0)})-\mathcal{F}(\theta^*)]\\
    &+48k||F(\theta^r)-F(\theta^*)||^2+3((4\varepsilon_{\mathcal{I}}s+2)+\varepsilon_{\mathcal{I}^c}(d-k))\mathbb{E}\mathcal{k}\nabla f_{i_r}(\theta^*)\mathcal{k}^2_\infty+6\varepsilon_{abs}\mu^2+6A_r.
\end{split}
    \end{equation}
\end{lemma}
\begin{proof}
we first get the upper bound of $\mathbb{E}||\hat g^{(r)}(\theta^{(r)})||^2$
\begin{equation}
\begin{split}\label{equ:aM_g}
\mathbb{E}||\hat g^{(r)}(\theta^{(r)})||^2&=\mathbb{E}||\hat\nabla f^{(r)}_{i_r}(\theta^{(r)})-\hat a^{(r)}_{i_r}+\frac{1}{n}\sum^n_{j=1}\hat a^{(r)}_j||^2\\
&=\mathbb{E}||\hat\nabla f^{(r)}_{i_r}(\theta^{(r)})-\nabla f^{(r)}_{i_r}(\theta^*)-\hat a^{(r)}_{i_r}+\nabla f^{(r)}_{i_r}(\theta^*)+\frac{1}{n}\sum^n_{j=1}\hat a^{(r)}_j-\nabla F(\theta^*)+\nabla F(\theta^*)||^2\\
&\le 3\mathbb{E} ||\hat{a}^{(r)}_{i_r}-\nabla f^{(r)}_{i_r}(\theta^*)-\left(\frac{1}{n}\sum_{j=1}^n{\hat{a}_{j}^{(r)}}-\nabla F(\theta^*)\right)||^2+3\mathbb{E} ||\hat{\nabla}f^{(r)}_{i_r}(\theta ^{(r)})-\nabla f^{(r)}_{i_r}(\theta^*)||^2+3||\nabla F(\theta^*)||^2
\\
&\le 3\mathbb{E} ||\hat{a}_{i_r}^{(r)}-\nabla f^{(r)}_{i_r}(\theta^*)||^2+3\mathbb{E} ||\hat{\nabla}f^{(r)}_{i_r}(\theta ^{(r)})-\nabla f_{i_r}^{(r)}(\theta^*)||^2+3||\nabla F(\theta^*)||^2\\
&\le 6\mathbb{E} ||{a}_{i_r}^{(r)}-\nabla f^{(r)}_{i_r}(\theta^*)||^2+6\mathbb{E} ||\hat{a}_{i_r}^{(r)}-a_{i_r}^{(r)})||^2+3\mathbb{E} ||\hat{\nabla}f_{i_r}^{(r)}(\theta ^{(r)})-\nabla f_{i_r}^{(r)}(\theta^*)||^2+3||\nabla F(\theta^*)||^2.\\
\end{split}
\end{equation}

From Lemma 3 we get 
\[
    \mathbb{E}\mathcal{k}\hat\nabla_\mathcal{I} f_i(\theta)-\nabla f_i(x)\mathcal{k}^2\le 2(\varepsilon_\mathcal{I}+1)\mathcal{k}\nabla_\mathcal{I}f_i(\theta)\mathcal{k}^2+2\varepsilon_{\mathcal{I}^c}\mathcal{k}\nabla_{\mathcal{I}^c}f_i(\theta)\mathcal{k}^2+2\varepsilon_{abs}\mu^2.
\]

Then we get:
\begin{equation}\label{equ:qmalpha}
\begin{split}
\mathbb{E} ||\hat{a}_{i_r}^{(r)}-a_{i_r}^{(r)}||^2&\le \frac{1}{n}\sum_{u=1}^{r-1}{(}1-\frac{q}{n})^{r-u-1}2((\varepsilon _{\mathcal{I}}+1)||\nabla _{\mathcal{I}}f_i(\theta^u)||^2+\varepsilon _{\mathcal{I} ^c}||\nabla _{\mathcal{I} ^c}f_i(\theta^u)||^2\\&+\varepsilon _{abs}\mu ^2)+(1-\frac{q}{n})^{r-1}2((\varepsilon _{\mathcal{I}}+1)||\nabla _{\mathcal{I}}f_i(\theta^0)||^2+\varepsilon _{\mathcal{I}^c}||\nabla _{\mathcal{I} ^c}f_i(\theta^0)||^2+\varepsilon _{abs}\mu ^2)=A_{r'}.
\end{split}
\end{equation}
For $\mathbb{E}\mathcal{k}\hat\nabla_\mathcal{I}f_{i_r}(\theta)-\nabla_{\mathcal{I}}f_{i_r}(\theta^*)\mathcal{k}^2_2$,and \cite{gu2020unified} we have:
\begin{equation}\label{mylem0}
\begin{split}
\mathbb{E}\mathcal{k}\hat\nabla_\mathcal{I}f_{i_r}(\theta)-\nabla_{\mathcal{I}}f_{i_r}(\theta^*)\mathcal{k}^2_2&\leq 4\varepsilon_{\mathcal{I}}\mathbb{E}\mathcal{k}\nabla f_{i_r}(\theta^t)-\nabla f_{i_r}(\theta^*)\mathcal{k}^2\\
&+((4\varepsilon_{\mathcal{I}}s+2)+\varepsilon_{\mathcal{I}^c}(d-k))\mathbb{E}\mathcal{k}\nabla f_{i_r}(\theta^*)\mathcal{k}^2_\infty+2\varepsilon_{abs}\mu^2.
\end{split}
\end{equation} 
From \cite{gu2020unified}, we know:
\begin{equation}\label{A1}
\begin{split}
    6||a^{(r)}-\nabla f(\theta^*)||^2+12\varepsilon_{\mathcal{I}}\mathbb{E} ||{\nabla}f^{(r)}(\theta ^{(r)})-\nabla f^{(r)}(\theta^*)||^2&\le \frac{24\rho^+_s}{n}\sum^{t-1}_{u=1}(1-\frac{q}{n})^{t-u-1}[\mathcal{F}(\theta^{(u)})-\mathcal{F}(\theta^*)]\\
    &+24\rho^+_s(1-\frac{q}{n})^{t}[\mathcal{F}(\theta^{(0)})-\mathcal{F}(\theta^*)]+48\varepsilon_{\mathcal{I}}\mathbb{E}[\mathcal{F}(\theta^{(r)})-\mathcal{F}(\theta^*)].\\
\end{split}    
\end{equation}
Taking (\ref{equ:qmalpha}),(\ref{A1}) into (\ref{equ:aM_g}),

\[
\begin{split}
    \mathbb{E}||\hat g^{(r)}(\theta^{(r)})||^2&\le \frac{24\rho^+_s}{n}\sum^{t-1}_{u=1}(1-\frac{q}{n})^{t-u-1}[\mathcal{F}(\theta^{(u)})-\mathcal{F}(\theta^*)]+24\rho^+_s(1-\frac{q}{n})^{t}[\mathcal{F}(\theta^{(0)})-\mathcal{F}(\theta^*)]\\
    &+48k\varepsilon_{\mathcal{I}}\mathbb{E}[\mathcal{F}(\theta^{(r)})-\mathcal{F}(\theta^*)]+3((4\varepsilon_{\mathcal{I}}s+2)+\varepsilon_{\mathcal{I}^c}(d-k))\mathbb{E}\mathcal{k}\nabla f_{i_r}(\theta^*)\mathcal{k}^2_\infty+6\varepsilon_{abs}\mu^2+6A_t.
\end{split}
\]
Here we get the conclusion.
\end{proof}
Then we can prove \textbf{Theorem 1}:
\begin{proof}
We denote $v=\theta^{(r)}-\eta g_{\mathcal{I}}^{(r)}(\theta^{(r)})$ and $\mathcal{I}=\mathcal{I}^*\cup\mathcal{I}^{(r)}\cup\mathcal{I}^{(r+1)}$, where $\mathcal{I}^*=supp(\theta^{*})$, $\mathcal{I}^{(r)}=supp(\theta^{(r)})$ and $\mathcal{I}^{(r+1)}=supp(\theta^{(r+1)})$
\begin{equation}
\begin{split}\label{B9}
\mathbb{E}||v-\theta^*||^2_2 &=\mathbb{E}||\theta^{(r)}-\theta^*||^2+\eta^2\mathbb{E}||\hat g^{(r)}_{\mathcal{I}}(\theta^{(r)})||^2_2-2\eta\left<\theta^{(r)}-\theta^*,\mathbb{E}\hat{g}^{(r)}_{\mathcal{I}}(\theta^{(r)})\right>\\
&=\mathbb{E}||\theta^{(r)}-\theta^*||^2_2+\eta^2\mathbb{E}||\hat{g}^{(r)}_{\mathcal{I}}(\theta^{(r)})||^2_2-2\eta\left<\theta^{(r)}-\theta^*,\mathbb{E}\hat\nabla_{\mathcal{I}}\mathcal{F}(\theta^{(r)})\right>\\
&=\mathbb{E}||\theta^{(r)}-\theta^*||^2_2+\eta^2\mathbb{E}||\hat{g}^{(r)}_{\mathcal{I}}(\theta^{(r)})||^2_2\\
&-2\eta\mathbb{E}\left<\theta^{(r)}-\theta^*,\hat\nabla_{\mathcal{I}}\mathcal{F}(\theta^{(r)})-\nabla_{\mathcal{I}}\mathcal{F}(\theta^{(r)})\right>-2\eta\mathbb{E}\left<\theta^{(r)}-\theta^*,\nabla_{\mathcal{I}}\mathcal{F}(\theta^{(r)})\right>\\
&\le\mathbb{E}||\theta^{(r)}-\theta^*||^2_2+\eta^2\mathbb{E}||\hat{g}^{(r)}_{\mathcal{I}}(\theta^{(r)})||^2_2-2\eta\mathbb{E}\left<\theta^{(r)}-\theta^*,\hat\nabla_{\mathcal{I}}\mathcal{F}(\theta^{(r)})-\nabla_{\mathcal{I}}\mathcal{F}(\theta^{(r)})\right>\\
&-2\eta\left[\mathcal{F}(\theta^{(r)})-\mathcal{F}(\theta^*)\right]\\
&=\mathbb{E}||\theta^{(r)}-\theta^*||^2_2+\eta^2\mathbb{E}||\hat{g}^{(r)}_{\mathcal{I}}(\theta^{(r)})||^2_2\\
&-2\eta\mathbb{E}\left<\sqrt{\eta}\rho^-_s(\theta^{(r)}-\theta^*),\frac{1}{\sqrt{\eta}\rho^-_s}(\hat\nabla_{\mathcal{I}}\mathcal{F}(\theta^{(r)})-\nabla_{\mathcal{I}}\mathcal{F}(\theta^{(r)}))\right>-2\eta\left[\mathcal{F}(\theta^{(r)})-\mathcal{F}(\theta^*)\right]\\
&\le(1+\eta^2{\rho^-_s}^2)\mathbb{E}||\theta^{(r)}-\theta^*||^2_2+\eta^2\mathbb{E}||\hat{g}^{(r)}_{\mathcal{I}}(\theta^{(r)})||^2_2-2\eta\left[\mathcal{F}(\theta^{(r)})-\mathcal{F}(\theta^*)\right]\\
&+\frac{1}{\rho^-_s}^2\mathbb{E}||(\hat\nabla_{\mathcal{I}}\mathcal{F}(\theta^{(r)})-\nabla_{\mathcal{I}}\mathcal{F}(\theta^{(r)}))||^2_2.\\
\end{split}
\end{equation}
The first inequality follows from N-dimensional mean inequality and the second inequality follows from Assumption 2.
For $\mathbb{E}||(\hat\nabla_{\mathcal{I}}\mathcal{F}(\theta^{(r)})-\nabla_{\mathcal{I}}\mathcal{F}(\theta^{(r)}))||^2_2$, we have:
\begin{equation}
\begin{split}\label{10}
\mathbb{E}||(\hat\nabla_{\mathcal{I}}\mathcal{F}(\theta^{(r)})-\nabla_{\mathcal{I}}\mathcal{F}(\theta^{(r)}))||^2_2&=||\mathbb{E}_{\bm u}(\hat\nabla_{\mathcal{I}}\sum^n_{i=1}{f_i}(\theta^{(r)})-\nabla_{\mathcal{I}}\sum^n_{i=1}{f_i}(\theta^{(r)}))||^2_2\\
&\le n\sum^n_{i=1}||\mathbb{E}_{\bm u}(\hat\nabla_{\mathcal{I}}{f_i}(\theta^{(r)})-\nabla_{\mathcal{I}}{f_i}(\theta^{(r)}))||^2_2\\
&\le n^2\varepsilon_\mu \mu^2.\\
\end{split}
\end{equation}
 By constraining $\mathbb{E}||(\hat\nabla_{\mathcal{I}}\mathcal{F}(\theta^{(r)})-\nabla_{\mathcal{I}}\mathcal{F}(\theta^{(r)}))||^2_2$, we can turn (\ref{B9}) into:
\begin{equation}
\begin{split}\label{11}
\mathbb{E}||\theta^{(r)}-\eta \hat{g}^{(r)}_{\mathcal{I}}(\theta^{(r)})-\theta^*||^2_2&\le (1+\eta^2{\rho^-_s}^2)\mathbb{E}||\theta^{(r)}-\theta^*||^2_2+\eta^2\mathbb{E}||\hat{g}^{(r)}_{\mathcal{I}}(\theta^{(r)})||^2_2\\
&-2\eta\left[\mathcal{F}(\theta^{(r)})-\mathcal{F}(\theta^*)\right]+\frac{n^2\varepsilon_\mu \mu^2}{{\rho^-_s}^2}.\\
\end{split}
\end{equation} 
Let $\alpha=1+\frac{2\sqrt{k^*}}{\sqrt{k-k^*}}$. Using Lemma 1, we have:
\begin{equation}
\begin{split}\label{12}
\mathbb{E}||\theta^{(r+1)}-\theta^*||^2_2&\le (1+\eta^2{\rho^-_s}^2)\alpha\mathbb{E}||\theta^{(r)}-\theta^*||^2_2+\alpha\frac{n^2\varepsilon_\mu \mu^2}{{\rho^-_s}^2}-2\eta\alpha\left[\mathcal{F}(\theta^{(r)})-\mathcal{F}(\theta^*)\right]\\
&+\eta^2(\alpha\frac{24\rho^+_s}{n}\sum^{t-1}_{u=1}(1-\frac{q}{n})^{t-u-1}[\mathcal{F}(\theta^{(u)})-\mathcal{F}(\theta^*)]\\
&+24\alpha\rho^+_s(1-\frac{q}{n})^{t}[\mathcal{F}(\theta^{(0)})-\mathcal{F}(\theta^*)]+48\alpha k\varepsilon_{\mathcal{I}}\mathbb{E}[\mathcal{F}(\theta^{(r)})-\mathcal{F}(\theta^*)]\\
&+3\alpha((4\varepsilon_{\mathcal{I}}s+2)+\varepsilon_{\mathcal{I}^c}(d-k))\mathbb{E}\mathcal{k}\nabla f_{i_r}(\theta^*)\mathcal{k}^2_\infty+6\alpha\varepsilon_{abs}\mu^2+6\alpha A).
\end{split}
\end{equation} 
Let $L_t=\alpha\frac{n^2\varepsilon_\mu \mu^2}{{\rho^-_s}^2}+\eta^2(3\alpha((4\varepsilon_{\mathcal{I}}s+2)+\varepsilon_{\mathcal{I}^c}(d-k))\mathbb{E}\mathcal{k}\nabla f_{i_r}(\theta^*)\mathcal{k}^2_\infty+6\alpha\varepsilon_{abs}\mu^2+6\alpha A_t)$, $\beta= (1+\eta^2{\rho_s^-}^2)\alpha$, then:
\begin{equation}\label{21}
\begin{split}
\mathbb{E}||\theta^{(r+1)}-\theta^*||^2_2&\le (1+\eta^2{\rho^-_s}^2)\alpha\mathbb{E}||\theta^{(r)}-\theta^*||^2_2+L-2\eta\alpha\left[\mathcal{F}(\theta^{(r)})-\mathcal{F}(\theta^*)\right]\\
&+\eta^2(\alpha\frac{24\rho^+_s}{n}\sum^{t-1}_{u=1}(1-\frac{q}{n})^{t-u-1}[\mathcal{F}(\theta^{(u)})-\mathcal{F}(\theta^*)]+24\alpha\rho^+_s(1-\frac{q}{n})^{t}[\mathcal{F}(\theta^{(0)})-\mathcal{F}(\theta^*)]\\
&+48\alpha \rho_s^+\varepsilon_{\mathcal{I}}\mathbb{E}[\mathcal{F}(\theta^{(r)})-\mathcal{F}(\theta^*)].\\
\end{split}
\end{equation}
Here, we use RSC and RSS condition, we have:
\begin{equation}\label{equ:t0}
    \begin{split}
\mathcal{F}(\theta^{(r+1)})-\mathcal{F}(\theta^*)&\le (\frac{2\beta}{\rho^-_s}+48\eta^2\alpha k\varepsilon_{\mathcal{I}}-2\eta\alpha)\left[\mathcal{F}(\theta^{(r)})-\mathcal{F}(\theta^*)\right]+L\\
&+\eta^2(\alpha\frac{24\rho^+_s}{n}\sum^{r-1}_{u=1}(1-\frac{q}{n})^{r-u-1}[\mathcal{F}(\theta^{(u)})-\mathcal{F}(\theta^*)]+24\alpha\rho^+_s(1-\frac{q}{n})^{t}[\mathcal{F}(\theta^{(0)})-\mathcal{F}(\theta^*)]\\
&-\frac{2}{\rho^+_s}\mathbb{E}\left<\nabla\mathcal{F}(\theta^*),\widetilde{\theta}^*-\theta^{(r+1)}\right>+\frac{2\beta}{\rho^-_s}\mathbb{E}\left<\nabla\mathcal{F}(\theta^*),\widetilde{\theta}^*-\theta^{(r)}\right>\\
&\le (\frac{2\beta}{\rho^-_s}+48\eta^2\alpha k\varepsilon_{\mathcal{I}}-2\eta\alpha)\left[\mathcal{F}(\theta^{(r)})-\mathcal{F}(\theta^*)\right]+L_t+\sqrt{s}||\nabla\mathcal{F}(\theta^*)||_{\infty}\mathbb{E}||{\theta}^{(r)}-\theta^*||_2\\
&+\eta^2(\alpha\frac{24\rho^+_s}{n}\sum^{t-1}_{u=1}(1-\frac{q}{n})^{t-u-1}\left[\mathcal{F}(\theta^{(u)})-\mathcal{F}(\theta^*)\right]+24\alpha\rho^+_s(1-\frac{q}{n})^{t}\left[\mathcal{F}(\theta^{(0)})-\mathcal{F}(\theta^*)\right].\\
\end{split}
\end{equation}
From \cite{gu2020unified}[Lemma11], we have:
\begin{equation}\label{equ:qM}
\begin{split}
    [\mathbb{E}\mathcal{F}(\theta^{(r+1)})-\mathcal{F}(\theta^*)] &\le(\left ( \frac{2\beta}{\rho^-_s}+48\eta^2\alpha \rho_s^+\varepsilon_{\mathcal{I}}-2\eta\alpha+1-\frac{p}{n} \right ))[\mathbb{E}\mathcal{F}(\theta^{(r)})-\mathcal{F}(\theta^*)]\\
    &+2(\alpha\frac{n^2\varepsilon_\mu \mu^2}{{\rho^-_s}^2}+6\alpha\varepsilon_{abs}\mu^2+6\eta^2\alpha A_r)+(\sqrt{s}||\nabla\mathcal{F}(\theta^*)||_{\infty}\mathbb{E}||{\theta}^{(r)}-\theta^*||_2\\
    &+\eta^2(3\alpha((4\varepsilon_{\mathcal{I}}s+2)+\varepsilon_{\mathcal{I}^c}(d-k))\mathbb{E}\mathcal{k}\nabla f_{i_r}(\theta^*)\mathcal{k}^2_\infty))
\end{split}
\end{equation}

\end{proof}
\section{VR-SZHT}
\subsection{algorithm}
Since the $p$-M algorithm is difficult to provide specific parameter analysis, we present a special case of the $p$-M algorithm, named VR-SZHT, for analysis. 
\begin{algorithm}[htp]
    \caption{Stochastic variance reduced zeroth-order Hard-Thresholding (VR-SZHT)}
    \label{alg:algorithm}
    \renewcommand{\algorithmicrequire}{\textbf{Input:}}   
    \renewcommand{\algorithmicensure}{\textbf{Output:}}    
    \begin{algorithmic}
        \Require Learning rate $\eta$, maximum number of iterations $T$, initial point $\theta^{0}$, SVRG update frequency $m$, number of random directions $q$, and number of coordinates to keep at each iteration $k$.
        
        \Ensure $\theta^T$.
        
        \renewcommand{\algorithmicrequire}{\textbf{Parameters:}}
        
        \For{$r=1,\ldots,T$}
            \State $\theta^{(0)}  = \theta^{r-1}$;
            \State $\hat{\mu}=\frac{1}{n}\sum^n_{i=1}\hat\nabla f_{i}(\theta^{(0)})$;
            \For {$t=0,1,\ldots,m-1$}
                \State Randomly sample $i_r\in \{1,2,\ldots,n\}$;
                \State Compute ZO estimate $\hat\nabla f_{i_r}(\theta^{(r)})$, $\hat\nabla f_{i_r}(\theta^{(0)})$;
              \State  $\bm \bar{\theta}^{(r+1)}=\theta^{(r)}-\eta  ( \hat\nabla f_{i_r}(\theta^{(r)})-\hat\nabla f_{i_r}(\theta^{(0)})+  \hat{\mu} )  )$;
                \State $\theta^{(r+1)}=\mathcal{H}_k(\bm \bar{\theta}^{(r+1)})$;

            \EndFor
         \State   $\theta^{r}=\theta^{(t')}$, random $t'\in[m-1]$
        \EndFor
    \end{algorithmic}
\end{algorithm}
In the stochastic variance reduced gradient part, different from the stochastic gradient, we use  $\hat\nabla f_{i_r}(\theta^{(r)})-\hat\nabla f_{i_r}(\theta^{(0)})+  \hat{\mu} $ to replace $\hat\nabla f_i(\theta^{(r)})$. Through this iteration, we reduce the variance and stabilize the algorithm. 
\subsection{Convergence}
In this section, we will provide a convergence analysis for VR-SZHT. Recall that the unknown sparse vector of interest is denoted as $\theta^*\in \mathbb{R}^d$, where $\mathcal{k}\theta^*\mathcal{k}_0\geq k^*$, and the hard-thresholding operator $\mathcal{H}_k:\mathbb{R}^d\rightarrow\mathbb{R}^d$ keeps the largest $k$ entries (in magnitude) and  the rest is set to zero. For ease of notation, we use $\mathbb{E}(\cdot)=\mathbb{E}_{\bm u,i_r}(\cdot)$, and we denote the full gradient and the stochastic variance reduced gradient by:
\begin{equation}\label{k1}
\begin{split}
\mu(\widetilde\theta)&=\nabla\mathcal{F}(\widetilde{\theta})=\frac{1}{n}\sum^n_{i=1}\nabla f_i(\widetilde{\theta}),g^{(r)}(\theta^{(r)})\\&=\nabla f_{i_r}(\theta^t)-\nabla f_{i_r}(\theta^{(0)})+\frac{1}{n}\sum^n_{i=1}\nabla f_{i}(\theta^{(0)});\\
\hat\mu(\widetilde\theta)&=\hat\nabla\mathcal{F}(\widetilde{\theta})=\frac{1}{n}\sum^n_{i=1}\hat\nabla f_i(\widetilde{\theta}),\hat g^{(r)}(\theta^{(r)})\\&=\hat\nabla f_{i_r}(\theta^t)-\hat\nabla f_{i_r}(\theta^{(0)})+\frac{1}{n}\sum^n_{i=1}\hat\nabla f_{i}(\theta^{(0)});\\
\hat g^{(r)} (\theta^{(r)})&=\hat\nabla f_{i_r}(\theta^{(r)})-\mathbb{E}\hat \nabla f_{i_r}(\widetilde{\theta})+\hat\mu(\widetilde\theta).
\end{split}
\end{equation}
\begin{theorem}\label{thm1}
Suppose $\mathcal{F}(\theta)$ satisfies the RSC condition and that the functions $\{f_i(\theta)\}^n_{i=1}$satisfy the RSS condition with $s=2k+k^*$. Let $\mathcal{I}^*=supp(\theta^*)$ denote the support of $\theta^*$. Let $\theta^{(r)}$ be a sparse vector with $\mathcal{k}\theta^{(r)}\mathcal{k}_0 \leq k$ and support $\mathcal{I}^{(r)}=supp(\theta^{(r)})$.
$$ \widetilde{\mathcal{I}}= supp(\mathcal{H}_{2k}(\hat\nabla\mathcal{F}(\theta^*))\cup supp(\theta^*).$$
Then, we have:
\begin{equation}
\begin{split}
&\frac{\beta^m-1}{\beta-1}(2\eta-48\varepsilon_{\mathcal{I}}\eta^2\rho^+_s)\alpha\left[\mathcal{F}(\widetilde{\theta}^{(r)})-\mathcal{F}(\theta^*)\right] \le\\ &(\frac{2\beta^m}{\rho^-_s}+\frac{48\eta^2\rho^+_s\varepsilon_{\mathcal{I}}\alpha(\beta^m-1)}{\beta-1})\mathbb{E}[\mathcal{F}(\widetilde{\theta}^{(r-1)})-\mathcal{F}(\theta^*)]\\&
+\frac{2\beta^m}{\rho^-_s}\sqrt{s}\mathcal{k}\nabla\mathcal{F}(\theta^*)\mathcal{k}_{\infty}\mathbb{E}\mathcal{k}\widetilde{\theta}^{(r-1)}-\theta^*\mathcal{k}_2+\frac{\beta^m-1}{\beta-1}\alpha L.\\
\end{split}
\end{equation}
\end{theorem}
Before presenting the proof of convergence, it is necessary to examine the boundary of the (ZO) SVRG gradient $\hat g_\mathcal{I}^{(r)}(\theta^{(r)})$. For this purpose, we first derive the following three lemmas. 
\begin{lemma}\label{lemma:1}
Under the condition of Theorem \ref{thm1}, for any $\mathcal{I}\supseteq(\mathcal{I}^*\cup\mathcal{I}^{(r)})$, we have $\mathbb{E}[\hat g^{(r)}(\theta^{(r)})]=\mathbb{E}_{\bm u}\hat\nabla\mathcal{F}(\theta^{(r)})$ and
\begin{equation}\label{l2}
\begin{split}
\mathbb{E}\mathcal{k}\hat g_\mathcal{I}^{(r)}(\theta^{(r)})\mathcal{k}^2_2&\leq 12\varepsilon_{\mathcal{I}}\mathbb{E}(\mathcal{k}\nabla f_{i_r}(\theta^t)-\nabla f_{i_r}(\theta^*)\mathcal{k}^2\\&+\mathcal{k}\nabla f_{i_r}(\theta^0)-\nabla f_{i_r}(\theta^*)\mathcal{k}^2)\\&+6((4\varepsilon_{\mathcal{I}}s+2)+\varepsilon_{\mathcal{I}^c}(d-k))\mathbb{E}\mathcal{k}\nabla f_{i_r}(\theta^*)\mathcal{k}^2_\infty\\
&+12\varepsilon_{abs}\mu^2+3\mathcal{k}\nabla_{\mathcal{I}}\mathcal{F}(\theta^*)\mathcal{k}^2_2.
\end{split}
\end{equation}
\end{lemma}
\begin{proof}
It is straightforward that the stochastic variance reduced gradient satisfies:
\[\mathbb{E}\hat g^{(r)} (\theta^{(r)})=\mathbb{E}\hat\nabla f_{i_r}(\theta^{(r)})-\mathbb{E}\hat \nabla f_{i_r}(\widetilde{\theta})+\hat\mu(\widetilde\theta)=\mathbb{E}_{\bm u} \hat\nabla\mathcal{F}(\theta^{(r)})\]
Thus $\hat g^{(r)}(\theta^{(r)})$ is a unbiased estimator of $\hat\nabla\mathcal{F}(\theta^{(r)})$. The form of an inapplicable zero order gradient is the same. As a result, the first claim is verified.\par
For the second claim, we have:
\begin{equation}\label{mylem11}
\begin{split}
\mathbb{E}\mathcal{k}g_\mathcal{I}^{(r)}(\theta^{(r)})\mathcal{k}^2_2=&\mathbb{E}\hat\nabla\mathcal{F}(\theta^{(r)})\\
\textbf{}\overset{\textrm{\ding{172}}}{\le}& 3\mathbb{E}\mathcal{k}\hat\nabla_\mathcal{I}f_{i_r}(\theta^{(r)})-\nabla_{\mathcal{I}}f_{i_r}(\theta^*)\mathcal{k}^2_2+3\mathcal{k}\nabla_{\mathcal{I}}\mathcal{F}(\theta^*)\mathcal{k}^2_2\\
&+3\mathbb{E}\mathcal{k}[\hat\nabla_\mathcal{I}f_{i_r}(\theta^{(0)})-\nabla_{\mathcal{I}}f_{i_r}(\theta^*)]-\hat\nabla_{\mathcal{I}}\mathcal{F}(\theta^{(0)})+\nabla_{\mathcal{I}}\mathcal{F}(\theta^*)\mathcal{k}^2_2\\
=& 3\mathbb{E}\mathcal{k}\hat\nabla_\mathcal{I}f_{i_r}(\theta^{(r)})-\nabla_{\mathcal{I}}f_{i_r}(\theta^*)\mathcal{k}^2_2+3\mathcal{k}\nabla_{\mathcal{I}}\mathcal{F}(\theta^*)\mathcal{k}^2_2\\
&+3\mathbb{E}_{\bm{u}} [\mathbb{E}_{i_r}\mathcal{k}[\hat\nabla_\mathcal{I}f_{i_r}(\theta^{(0)})-\nabla_{\mathcal{I}}f_{i_r}(\theta^*)]-\hat\nabla_{\mathcal{I}}\mathcal{F}(\theta^{(0)})+\nabla_{\mathcal{I}}\mathcal{F}(\theta^*)\mathcal{k}^2_2]\\
\textbf{}\overset{\textrm{\ding{173}}}{\le}& 3\mathbb{E}\mathcal{k}\hat\nabla_\mathcal{I}f_{i_r}(\theta^{(r)})-\nabla_{\mathcal{I}}f_{i_r}(\theta^*)\mathcal{k}^2_2+3\mathcal{k}\nabla_{\mathcal{I}}\mathcal{F}(\theta^*)\mathcal{k}^2_2\\
&+3\mathbb{E}_{\bm{u}} [\mathbb{E}_{i_r}\mathcal{k}\hat\nabla_\mathcal{I}f_{i_r}(\theta^{(0)})-\nabla_{\mathcal{I}}f_{i_r}(\theta^*)\mathcal{k}^2_2]\\
=&3\mathbb{E}\mathcal{k}\hat\nabla_\mathcal{I}f_{i_r}(\theta^{(r)})-\nabla_{\mathcal{I}}f_{i_r}(\theta^*)\mathcal{k}^2_2+3\mathcal{k}\nabla_{\mathcal{I}}\mathcal{F}(\theta^*)\mathcal{k}^2_2\\
&+3\mathbb{E}\mathcal{k}\hat\nabla_\mathcal{I}f_{i_r}(\theta^{(0)})-\nabla_{\mathcal{I}}f_{i_r}(\theta^*)\mathcal{k}^2_2\\
\end{split}
\end{equation}
The  inequality $\textrm{\ding{172}}$ follows from the power mean inequality$\mathcal{k}a+b+c\mathcal{k}^2_2\le3\mathcal{k}a\mathcal{k}^2_2+3\mathcal{k}b\mathcal{k}^2_2+3\mathcal{k}c\mathcal{k}^2_2$, and $\textrm{\ding{173}}$ is follows from $\mathbb{E}\mathcal{k}x-\mathbb{E}x\mathcal{k}^2_2\le\mathbb{E}\mathcal{k}x\mathcal{k}^2_2$. Now we focus on $\mathbb{E}\mathcal{k}\hat\nabla_\mathcal{I}f_{i_r}(\theta)-\nabla_{\mathcal{I}}f_{i_r}(\theta^*)\mathcal{k}^2_2$. Actually, the boundary of $\mathbb{E}\mathcal{k}\hat\nabla_{\mathcal{I}}f_{i_r}(\theta)-\nabla_{\mathcal{I}}f_{i_r}(\theta^*)\mathcal{k}^2_2$ is available in Lemma 3, that is

\begin{equation}\label{mylem01}
\begin{split}
\mathbb{E}\mathcal{k}\hat\nabla_\mathcal{I}f_{i_r}(\theta)-\nabla_{\mathcal{I}}f_{i_r}(\theta^*)\mathcal{k}^2_2&\leq 4\varepsilon_{\mathcal{I}}\mathbb{E}\mathcal{k}\nabla f_{i_r}(\theta^t)-\nabla f_{i_r}(\theta^*)\mathcal{k}^2\\
&+((4\varepsilon_{\mathcal{I}}s+2)+\varepsilon_{\mathcal{I}^c}(d-k))\mathbb{E}\mathcal{k}\nabla f_{i_r}(\theta^*)\mathcal{k}^2_\infty+2\varepsilon_{abs}\mu^2
\end{split}
\end{equation}
Taking (\ref{mylem01}) into (\ref{mylem11}):
\[
\begin{split}
\mathbb{E}\mathcal{k}g_\mathcal{I}^{(r)}(\theta^{(r)})\mathcal{k}^2_2&\leq 12\varepsilon_{\mathcal{I}}\mathbb{E}(\mathcal{k}\nabla f_{i_r}(\theta^t)-\nabla f_{i_r}(\theta^*)\mathcal{k}^2+\mathcal{k}\nabla f_{i_r}(\theta^0)-\nabla f_{i_r}(\theta^*)\mathcal{k}^2)\\
&+6((4\varepsilon_{\mathcal{I}}s+2)+\varepsilon_{\mathcal{I}^c}(d-k))\mathbb{E}\mathcal{k}\nabla f_{i_r}(\theta^*)\mathcal{k}^2_\infty+12\varepsilon_{abs}\mu^2+3\mathcal{k}\nabla_{\mathcal{I}}\mathcal{F}(\theta^*)\mathcal{k}^2_2
\end{split}
\]
\end{proof}
The unbiasedness of SVRG implies that $\mathbb{E}[\hat g^{(r)}(\theta^{(r)})]=\mathbb{E}_{\bm u}\hat\nabla\mathcal{F}(\theta^{(r)})$ is obvious. As for inequality (\ref{12}), we first use the mean inequality to divide it into three parts, and then prove it using Lemma 1. 
Now we provide the proof for Theorem 3:
\begin{proof}
We denote $v=\theta^{(r)}-\eta g_{\mathcal{I}}^{(r)}(\theta^{(r)})$ and $\mathcal{I}=\mathcal{I}^*\cup\mathcal{I}^{(r)}\cup\mathcal{I}^{(r+1)}$, where $\mathcal{I}^*=supp(\theta^{*})$, $\mathcal{I}^{(r)}=supp(\theta^{(r)})$ and $\mathcal{I}^{(r+1)}=supp(\theta^{(r+1)})$
\begin{equation}
\begin{split}\label{9}
\mathbb{E}||v-\theta^*||^2_2 &=\mathbb{E}||\theta^{(r)}-\theta^*||^2+\eta^2\mathbb{E}||\hat g^{(r)}_{\mathcal{I}}(\theta^{(r)})||^2_2-2\eta\left<\theta^{(r)}-\theta^*,\mathbb{E}\hat{g}^{(r)}_{\mathcal{I}}(\theta^{(r)})\right>\\
&=\mathbb{E}||\theta^{(r)}-\theta^*||^2_2+\eta^2\mathbb{E}||\hat{g}^{(r)}_{\mathcal{I}}(\theta^{(r)})||^2_2-2\eta\left<\theta^{(r)}-\theta^*,\mathbb{E}\hat\nabla_{\mathcal{I}}\mathcal{F}(\theta^{(r)})\right>\\
&=\mathbb{E}||\theta^{(r)}-\theta^*||^2_2+\eta^2\mathbb{E}||\hat{g}^{(r)}_{\mathcal{I}}(\theta^{(r)})||^2_2\\
&-2\eta\mathbb{E}\left<\theta^{(r)}-\theta^*,\hat\nabla_{\mathcal{I}}\mathcal{F}(\theta^{(r)})-\nabla_{\mathcal{I}}\mathcal{F}(\theta^{(r)})\right>-2\eta\mathbb{E}\left<\theta^{(r)}-\theta^*,\nabla_{\mathcal{I}}\mathcal{F}(\theta^{(r)})\right>\\
&\le\mathbb{E}||\theta^{(r)}-\theta^*||^2_2+\eta^2\mathbb{E}||\hat{g}^{(r)}_{\mathcal{I}}(\theta^{(r)})||^2_2-2\eta\mathbb{E}\left<\theta^{(r)}-\theta^*,\hat\nabla_{\mathcal{I}}\mathcal{F}(\theta^{(r)})-\nabla_{\mathcal{I}}\mathcal{F}(\theta^{(r)})\right>\\
&-2\eta\left[\mathcal{F}(\theta^{(r)})-\mathcal{F}(\theta^*)\right]\\
&=\mathbb{E}||\theta^{(r)}-\theta^*||^2_2+\eta^2\mathbb{E}||\hat{g}^{(r)}_{\mathcal{I}}(\theta^{(r)})||^2_2-2\eta\mathbb{E}\left<\sqrt{\eta}\rho^-_s(\theta^{(r)}-\theta^*),\frac{1}{\sqrt{\eta}\rho^-_s}(\hat\nabla_{\mathcal{I}}\mathcal{F}(\theta^{(r)})-\nabla_{\mathcal{I}}\mathcal{F}(\theta^{(r)}))\right>\\
&-2\eta\left[\mathcal{F}(\theta^{(r)})-\mathcal{F}(\theta^*)\right]\\
&\le(1+\eta^2{\rho^-_s}^2)\mathbb{E}||\theta^{(r)}-\theta^*||^2_2+\eta^2\mathbb{E}||\hat{g}^{(r)}_{\mathcal{I}}(\theta^{(r)})||^2_2-2\eta\left[\mathcal{F}(\theta^{(r)})-\mathcal{F}(\theta^*)\right]\\
&+\frac{1}{\rho^-_s}^2\mathbb{E}||(\hat\nabla_{\mathcal{I}}\mathcal{F}(\theta^{(r)})-\nabla_{\mathcal{I}}\mathcal{F}(\theta^{(r)}))||^2_2\\
\end{split}
\end{equation}
For $\mathbb{E}||(\hat\nabla_{\mathcal{I}}\mathcal{F}(\theta^{(r)})-\nabla_{\mathcal{I}}\mathcal{F}(\theta^{(r)}))||^2_2$, we have:
\begin{equation}
\begin{split}\label{10}
\mathbb{E}||(\hat\nabla_{\mathcal{I}}\mathcal{F}(\theta^{(r)})-\nabla_{\mathcal{I}}\mathcal{F}(\theta^{(r)}))||^2_2&=||\mathbb{E}_{\bm u}(\hat\nabla_{\mathcal{I}}\sum^n_{i=1}{f_i}(\theta^{(r)})-\nabla_{\mathcal{I}}\sum^n_{i=1}{f_i}(\theta^{(r)}))||^2_2\\
&\le n\sum^n_{i=1}||\mathbb{E}_{\bm u}(\hat\nabla_{\mathcal{I}}{f_i}(\theta^{(r)})-\nabla_{\mathcal{I}}{f_i}(\theta^{(r)}))||^2_2\\
&\le n^2\varepsilon_\mu \mu^2\\
\end{split}
\end{equation}
The first inequality follows from N-dimensional mean inequality and the second inequality follows from Assumption 2. By constraining $\mathbb{E}||(\hat\nabla_{\mathcal{I}}\mathcal{F}(\theta^{(r)})-\nabla_{\mathcal{I}}\mathcal{F}(\theta^{(r)}))||^2_2$, we can turn (\ref{9}) into:
\begin{equation}
\begin{split}\label{11}
\mathbb{E}||\theta^{(r)}-\eta \hat{g}^{(r)}_{\mathcal{I}}(\theta^{(r)})-\theta^*||^2_2&\le (1+\eta^2{\rho^-_s}^2)\mathbb{E}||\theta^{(r)}-\theta^*||^2_2+\eta^2\mathbb{E}||\hat{g}^{(r)}_{\mathcal{I}}(\theta^{(r)})||^2_2-2\eta\left[\mathcal{F}(\theta^{(r)})-\mathcal{F}(\theta^*)\right]\\
&+\frac{n^2\varepsilon_\mu \mu^2}{{\rho^-_s}^2}\\
\end{split}
\end{equation} 
Let $\alpha=1+\frac{2\sqrt{k^*}}{\sqrt{k-k^*}}$. Using Lemma 2, we have:
\begin{equation}
\begin{split}\label{12}
    \mathbb{E}||\theta^{(r+1)}-\theta^*||^2_2&\le (1+\eta^2{\rho^-_s}^2)\alpha\mathbb{E}||\theta^{(r)}-\theta^*||^2_2+\eta^2\alpha\mathbb{E}||\hat{g}^{(r)}_{\mathcal{I}}(\theta^{(r)})||^2_2-2\eta\alpha\left[\mathcal{F}(\theta^{(r)})-\mathcal{F}(\theta^*)\right]\\
    &+\alpha\frac{n^2\varepsilon_\mu \mu^2}{{\rho^-_s}^2}\\
    &\le (1+\eta^2{\rho^-_s}^2)\alpha\mathbb{E}||\theta^{(r)}-\theta^*||^2_2+\alpha\frac{n^2\varepsilon_\mu \mu^2}{{\rho^-_s}^2}-2\eta\alpha\left[\mathcal{F}(\theta^{(r)})-\mathcal{F}(\theta^*)\right]\\
    &+\eta^2\alpha(12\varepsilon_{\mathcal{I}}\mathbb{E}(||\nabla f_{i_r}(\bm{x}^t)-\nabla f_{i_r}(\theta^*)||^2+||\nabla f_{i_r}(\theta^0)-\nabla f_{i_r}(\theta^*)||^2)\\
    &+6((4\varepsilon_{\mathcal{I}}s+2)+\varepsilon_{\mathcal{I}^c}(d-k))\mathbb{E}||\nabla f_{i_r}(\theta^*)||^2_\infty+12\varepsilon_{abs}\mu^2+3||\nabla_{\mathcal{I}}\mathcal{F}(\theta^*)||^2_2)\\
    \end{split}
    \end{equation} 
    For $||\nabla f_{i_r}(\bm{\theta}^t)-\nabla f_{i_r}(\theta^*)||^2$, we can easily get 
    $$\mathbb{E}||\nabla f_{i_r}(\bm{\theta}^t)-\nabla f_{i_r}(\theta^*)||^2=\frac{1}{n}||\nabla f_{i_r}(\bm{\theta}^t)-\nabla f_{i_r}(\theta^*)||^2\le 4\rho^+_s[\mathcal{F}(\theta)-\mathcal{F}(\theta^*)]$$ 
    by RSS condition. As a reasult, we have: 
    \begin{equation}
    \begin{split}
    \mathbb{E}||\theta^{(r+1)}-\theta^*||^2_2 &\le (1+\eta^2{\rho^-_s}^2)\alpha\mathbb{E}||\theta^{(r)}-\theta^*||^2_2+\alpha\frac{n^2\varepsilon_\mu \mu^2}{{\rho^-_s}^2}-\alpha(2\eta-48\varepsilon_{\mathcal{I}}\eta^2\rho^+_s)\left[\mathcal{F}(\theta^{(r)})-\mathcal{F}(\theta^*)\right]\\
    &+48\eta^2\rho^+_s\alpha\varepsilon_{\mathcal{I}}\mathbb{E}[\mathcal{F}(\theta^0)-\mathcal{F}(\theta^*)]\\
    &+6\eta^2\alpha((4\varepsilon_{\mathcal{I}}s+2)+\varepsilon_{\mathcal{I}^c}(d-k))\mathbb{E}||\nabla f_{i_r}(\theta^*)||^2_\infty+12\varepsilon_{abs}\mu^2+3||\nabla_{\mathcal{I}}\mathcal{F}(\theta^*)||^2_2)\\
    \end{split}
    \end{equation}
    Let $L=6\eta^2((4\varepsilon_{\mathcal{I}}s+2)+\varepsilon_{\mathcal{I}^c}(d-k))\mathbb{E}||\nabla f_{i_r}(\theta^*)||^2_\infty+12\varepsilon_{abs}\mu^2+3||\nabla_{\mathcal{I}}\mathcal{F}(\theta^*)||^2_2)+\frac{n^2\varepsilon_\mu \mu^2}{{\rho^-_s}^2}$, $\beta= (1+\eta^2{\rho_s^-}^2)\alpha$, then:
    \begin{equation}\label{21}
    \begin{split}
    \mathbb{E}||\theta^{(r+1)}-\theta^*||^2_2+\alpha(2\eta-48\varepsilon_{\mathcal{I}}\eta^2\rho^+_s)\left[\mathcal{F}(\theta^{(r)})-\mathcal{F}(\theta^*)\right] &\le \beta\mathbb{E}||\theta^{(r)}-\theta^*||^2_2\\
    &+48\eta^2\rho^+_s\varepsilon_{\mathcal{I}}\alpha\mathbb{E}[\mathcal{F}(\theta^0)-\mathcal{F}(\theta^*)]+\alpha L\\
    \end{split}
    \end{equation}
     By summing (\ref{21}) over $t=0,\ldots,m-1$, we have:
    \begin{equation}
\begin{split}\label{next}
\mathbb{E}||\theta^{(m)}-\theta^*||^2_2+\frac{\beta^m-1}{\beta-1}(2\eta-48\varepsilon_{\mathcal{I}}\eta^2\rho^+_s)\alpha\left[\mathcal{F}(\widetilde{\theta}^{(r)})-\mathcal{F}(\theta^*)\right] &\le \beta^m\mathbb{E}||\theta^{(r)}-\theta^*||^2_2\\
&+48\eta^2\rho^+_s\varepsilon_{\mathcal{I}}\alpha\frac{\beta^m-1}{\beta-1}\mathbb{E}[\mathcal{F}(\widetilde{\theta}^{(r-1)})-\mathcal{F}(\theta^*)]+\frac{\beta^m-1}{\beta-1}\alpha L\\
\end{split}
\end{equation}
Through RSC condition and the definition of $\widetilde{I}$, it further follows from (\ref{next}) that:
\begin{equation}
\begin{split}\label{18}
\frac{\beta^m-1}{\beta-1}(2\eta-48\varepsilon_{\mathcal{I}}\eta^2\rho^+_s)\alpha\left[\mathcal{F}(\widetilde{\theta}^{(r)})-\mathcal{F}(\theta^*)\right] &\le \left(\frac{2\beta^m}{\rho^-_s}+\frac{48\eta^2\rho^+_s\varepsilon_{\mathcal{I}}\alpha(\beta^m-1)}{\beta-1}\right)\mathbb{E}[\mathcal{F}(\widetilde{\theta}^{(r-1)})-\mathcal{F}(\theta^*)]\\&
+\frac{2\beta^m}{\rho^-_s}\mathbb{E}\left<\nabla\mathcal{F}(\theta^*),\widetilde{\theta}^{(r-1)}-\theta^*\right>+\frac{\beta^m-1}{\beta-1}\alpha L\\
&\le \left(\frac{2\beta^m}{\rho^-_s}+\frac{48\eta^2\rho^+_s\varepsilon_{\mathcal{I}}\alpha(\beta^m-1)}{\beta-1}\right)\mathbb{E}[\mathcal{F}(\widetilde{\theta}^{(r-1)})-\mathcal{F}(\theta^*)]\\&
+\frac{2\beta^m}{\rho^-_s}\sqrt{s}||\nabla\mathcal{F}(\theta^*)||_{\infty}\mathbb{E}||\widetilde{\theta}^{(r-1)}-\theta^*||_2+\frac{\beta^m-1}{\beta-1}\alpha L\\
\end{split}
\end{equation}
Here $\varepsilon_\mu={\rho^+_{s}}^2sd$, $\varepsilon_{\mathcal{I}}=\frac{2d}{q(s_2+2)}(\frac{(s-1)(s_2-1)}{d-1}+3)+2$, $\varepsilon_{\mathcal{I}^c}=\frac{2d}{q(s_2+2)}\left(\frac{s(s_2-1)}{-1}\right)$, $\varepsilon_{abs}=\frac{2d{\rho^+_s}^2ss_2}{q}\left(\frac{(s-1)(s_2-1)}{d-1}+1\right)+{\rho^+_s}^2sd$
\end{proof}

\subsection{Relationship between Parameters}
Upon completing the proof, we can conclude that the convergence of the algorithm is contingent on the coefficient of $\left[\mathcal{F}(\widetilde{\theta}^{(r)})-\mathcal{F}(\theta^*)\right]$, $[\mathcal{F}(\widetilde{\theta}^{(r-1)})-\mathcal{F}(\theta^*)]$, which is determined by the values of $m$, $\eta$, $\varepsilon_{\mathcal{I}}$ and $\alpha$. In the subsequent section, we will investigate the relationship among these parameters. Recall (\ref{18}), we have 
\[
\begin{split}
\frac{\beta^m-1}{\beta-1}(2\eta-48\varepsilon_{\mathcal{I}}\eta^2\rho^+_s)\alpha\left[\mathcal{F}(\widetilde{\theta}^{(r)})-\mathcal{F}(\theta^*)\right] &\le (\frac{2\beta^m}{\rho^-_s}+\frac{48\eta^2\rho^+_s\varepsilon_{\mathcal{I}}\alpha(\beta^m-1)}{\beta-1})\mathbb{E}[\mathcal{F}(\widetilde{\theta}^{(r-1)})-\mathcal{F}(\theta^*)]\\&
+\frac{2\beta^m}{\rho^-_s}\sqrt{s}\mathcal{k}\nabla\mathcal{F}(\theta^*)\mathcal{k}_{\infty}\mathbb{E}\mathcal{k}\widetilde{\theta}^{(r-1)}-\theta^*\mathcal{k}_2+\frac{\beta^m-1}{\beta-1}\alpha L\\
\end{split}
\]
Actually, if the algorithm converges, we have:
\[
    \frac{\beta-1}{\alpha(\beta^m-1)(2\eta-48\varepsilon_{\mathcal{I}}\eta^2\rho^+_s)}(\frac{2\beta^m}{\rho^-_s}+\frac{48\eta^2\rho^+_s\varepsilon_{\mathcal{I}}\alpha(\beta^m-1)}{\beta-1})\le 1,
\]
which is:
\begin{equation}\label{equ:conver}
        \frac{24\varepsilon_{\mathcal{I}}\eta\rho^+_s}{(1-24\varepsilon_{\mathcal{I}}\eta\rho^+_s)}+\frac{(\beta-1)\beta^m}{\eta\alpha\rho^-_s(\beta^m-1)(1-24\varepsilon_{\mathcal{I}}\eta\rho^+_s)}\le 1 .
\end{equation}
 (\ref{equ:conver}) provides us with a more detailed understanding of the relationship between the zeroth-order gradient estimator, the stochastic variance reduced gradient estimator, and the hard-thresholding operator when the algorithm converges. Based on this equation, we can derive two corollaries as follows:
\begin{corollary}
When the learning rate $\eta\in[\eta_{\min},\min\{\eta_{\max},\frac{1}{48\varepsilon_{\mathcal{I}}\rho^+_s}\}]$, there always exists an $m$ such that the algorithm converges.Here
\[\eta_{\max}=\frac{\alpha\rho^-_s+\sqrt{(\alpha{\rho^-_s})^2-4(48\varepsilon_{\mathcal{I}}\alpha\rho^-_s\rho^+_s+{\rho^-_s}^2)(\alpha-1)}}{2(48\varepsilon_{\mathcal{I}}\alpha\rho^-_s\rho^+_s+{\rho^-_s}^2)};\]
\[\eta_{\min}=\frac{\alpha\rho^-_s-\sqrt{(\alpha{\rho^-_s})^2-4(48\varepsilon_{\mathcal{I}}\alpha\rho^-_s\rho^+_s+{\rho^-_s}^2)(\alpha-1)}}{2(48\varepsilon_{\mathcal{I}}\alpha\rho^-_s\rho^+_s+{\rho^-_s}^2)}.\]
\end{corollary}

\begin{remark}
This corollary reveals that there is a range of values for $\eta$. When $\eta<\frac{1}{48\varepsilon_{\mathcal{I}}\rho^+_s}$, we have
    \begin{equation}\label{equ:beta}
           \beta^{-m}\le \frac{(1-48\varepsilon_{\mathcal{I}}\eta\rho^+_s)\eta\alpha\rho^-_s-(\beta-1)}{(1-48\varepsilon_{\mathcal{I}}\eta\rho^+_s)\eta\alpha\rho^-_s}.
    \end{equation}
    This means that if 
         \begin{equation}\label{equ:etam}
         (1-48\varepsilon_{\mathcal{I}}\eta\rho^+_s)\eta\alpha\rho^-_s-(\beta-1)>0,
     \end{equation}
    holds, there is always a suitable value of $m$ that makes the algorithm converge.  We recall $\beta=(1+\eta^2{\rho^-_s}^2)\alpha$ and (\ref{equ:beta}) becomes:

        \begin{equation}\label{equ:condition}
        (48\varepsilon_{\mathcal{I}}\alpha\rho^-_s\rho^+_s+{\rho^-_s}^2)\eta^2-\alpha\rho^-_s\eta+\alpha-1<0.
    \end{equation}
If we solve (\ref{equ:condition}) , we will obtain $\eta_{\min}<\eta<\eta_{\max}$. To achieve optimal convergence, we suggest setting $\eta=\frac{\alpha \rho^-_S}{2(48\varepsilon_{\mathcal{I}}\alpha\rho^-_s\rho^+_s+{\rho^-_s}^2)}$. We will revisit this recommendation later on in Remark 6.
        
\end{remark}
\begin{remark}
The value of k in equation (\ref{equ:condition}) is subject to a limitation. For the algorithm to converge, we require that $\alpha^2>4(48\varepsilon_{\mathcal{I}}\alpha\kappa_s+1)(\alpha-1)$, which is a polynomial of k. Therefore, we can obtain the boundary of $k$ by solving this inequality. Note that SARAH-SZHT can be proven in the same way as VR-SZHT, so it will not be repeated in this paper.
\end{remark}
\begin{proof}
    let $\eta\le \frac{C_3}{\varepsilon_{\mathcal{I}}\rho^+_s}<\frac{1}{72\varepsilon_{\mathcal{I}}\rho^+_s}<\frac{C_3}{144\rho^+_s}$, then
\[\frac{24\varepsilon_{\mathcal{I}}\eta\rho^+_s}{1-24\varepsilon_{\mathcal{I}}\eta\rho^+_s}<\frac{24C_3}{1-24C_3}<\frac{1}{2}\]
If $k\ge C_1\kappa^2_sk^*$, $\eta\ge \frac{C_2}{\rho^+_s}$ with $C_2<C_3<C_3\varepsilon_{\mathcal{I}}$, and $A = \frac{2}{\sqrt{C_1-1}}+\frac{C_3^2}{4\kappa_s}+\frac{C_3^2}{2\sqrt{C_1-1}\kappa^2_s}$. Then we have $\alpha\le 1+\frac{2}{\sqrt{C_1-1}\kappa_s}$, $\beta\le 1+\frac{A}{\kappa_s}$ and
\[
\begin{split}
    \frac{\beta^m(\beta-1)}{\eta\rho^-_s\alpha(1-24\varepsilon_{\mathcal{I}}\eta\rho^+_s)(\beta^m-1)}&\le\frac{\frac{A}{\kappa_s}}{\frac{2C_2}{3\kappa_s}(1-(1+\frac{A}{\kappa_s})^{-m})}\\&=\frac{3A}{2C_2(1-(1+\frac{A}{\kappa_s})^{-m})}
\end{split}\]
It is guaranteed $\frac{\beta^m(\beta-1)}{\eta\rho^-_s\alpha(1-24\varepsilon_{\mathcal{I}}\eta\rho^+_s)(\beta^m-1)}<\frac{1}{2}$ if we have
 \[m>\log_{1+\frac{A}{\kappa_s}}\frac{C_2}{C_2-3A}=\frac{\log\frac{C_2}{C_2-3A}}{\log(1+\frac{A}{\kappa_s})}\]
Using the fact that $ln(1+x)>\frac{x}{2}$
\[\begin{split}
    \frac{\log\frac{C_2}{C_2-3A}}{\log(1+\frac{A}{\kappa_s})}&\le \log\frac{C_2}{C_2-3A}\left(\frac{2\kappa_s}{A}\right)
\end{split}\]
Then $\frac{\beta^m(\beta-1)}{\eta\rho^-_s\alpha(1-24\varepsilon_{\mathcal{I}}\eta\rho^+_s)(\beta^m-1)}<\frac{1}{2}$ holds if m satisfies
\[
\begin{split}
    m&\ge \log \left (\frac{C_2}{C_2-3A} \right )\left(\frac{2\kappa_s}{A}\right)\\
    &=\log \left (\frac{C_2}{C_2-3A} \right )\left(\frac{2\kappa_s^3}{\frac{2}{\sqrt{C_1-1}}\kappa_s^2+\frac{C_3^2}{4}\kappa_s+\frac{C_3^2}{2\sqrt{C_1-1}}}\right)\\
    &=\mathcal{O}\left(\frac{\kappa^3}{\kappa^2+1}\right)\\
\end{split}
\]
 if we want to contorl the error below $\varepsilon$, we need $\mathcal{O}\left(\log(\frac{1}{\varepsilon})\right)$ outer iterations. 
 And for each outer iterations, we need to calculate a full (ZO) gradient and m stochastic variance reduced (ZO) gradients. So the query complexity of the algorithm is $[n+\frac{\kappa^3}{\kappa^2+1}]\log{(\frac{1}{\varepsilon})}$. When $\kappa_s$ is large, the complexity will tend to $[n+\kappa]\log{(\frac{1}{\varepsilon})}$. 
When $\kappa_s$ is small, the complexity will tend to $[n+\kappa^3]\log{(\frac{1}{\varepsilon})}$. Thus VR-SZHT yields a significant improvement over SZOHT.
\end{proof}

 \begin{corollary}
     The query complexity of the algorithm is $\mathcal{O}([n+\frac{\kappa^3}{\kappa^2+1}]\log{(\frac{1}{\varepsilon})})$.
 \end{corollary}
If we want the error to be below $\varepsilon$, we need $\mathcal{O}\left(\log(\frac{1}{\varepsilon})\right)$ outer iteration. For each outer iteration, we need to calculate a full (ZO) gradient and $m$ stochastic variance reduced (ZO) gradients. So the query complexity of the algorithm is $\mathcal{O}([n+\frac{\kappa^3}{\kappa^2+1}]\log{(\frac{1}{\varepsilon})})$. When $\kappa_s$ is large, the complexity will tend to $\mathcal{O}([n+\kappa]\log{(\frac{1}{\varepsilon})})$. When $\kappa_s$ is small, the complexity will tend to $\mathcal{O}([n+\kappa^3]\log{(\frac{1}{\varepsilon})})$. Thus, VR-SZHT yields a significant improvement over SZOHT. See the appendix for the specific proof.
\section{SARAH-SZHT}\label{sec:sarah}
\textbf{Algorithm}:  
With the development of variance reduction theory, an iterative form of gradient descent algorithm has emerged, such as SARAH~\cite{nguyen2017sarah}, etc. Such algorithms are similar to SVRG, in that they utilize both inner and outer loops. Unfortunately, the $p$-Memorization framework is not applicable to such algorithms.  Therefore, in this section, we will attempt to apply the iterative form of variance descent algorithm to gradient simulation. 
    \begin{algorithm}
    \caption{Stochastic variance reduced zeroth-order Hard-Thresholding with SARAH (SARAH-SZHT)}
    \label{alg:SPIDER}
    \renewcommand{\algorithmicrequire}{\textbf{Input:}}   
    \renewcommand{\algorithmicensure}{\textbf{Output:}}    
    \begin{algorithmic}[1]
        \Require Learning rate $\eta$, maximum number of iterations $T$, initial point $\theta^{0}$, number of random directions $q$, maximum error $\varepsilon$, and number of coordinates to keep at each iteration $k$.
        
        \Ensure $\theta^T$.
        
        \renewcommand{\algorithmicrequire}{\textbf{Parameters:}}
        
        \For{$r=1,\ldots,T$}
            \State $\theta^0= \widetilde{\theta}^{(r-1)}$
            \State $g^{(0)} = \frac{1}{n}\sum^n_{i=1} \hat\nabla f_i(\theta_0)$
            \State $\theta^{(1)} = \theta^{(0)}-\eta g^{(0)}$
            \For{$t = 1,...,m-1$}
            \State Sample $i_t$ uniformly at random from $[n]$
            \State $\hat g^{(t)} = \hat\nabla f_{i_t}(\theta^{(t)}) -\hat\nabla f_{i_t}(\theta^{(t-1)}) +\hat g^{t-1}$
            \State $\theta^{(t+1)}=\theta^{(t)}-\eta g^t$
            \EndFor
            \State Set $\widetilde{\theta}^{(r)}=\theta^{(d)}$ with $d$ chosen uniformly at random from $\{0,1,\ldots,m\}$.

        \EndFor
    \end{algorithmic}  
\end{algorithm}

\textbf{Convergence Analysis}: \ 
 In this section, we will provide a convergence analysis for SARAH-SZHT. using the assumptions from section 2, and discussing the relationship between parameters, providing a positive answer to the question from the section. Following the analysis in section 3, we first obtain the boundary of $\mathbb{E}\mathcal{k}g_\mathcal{I}^{(t)}(\theta^{(t)})\mathcal{k}^2_2$.
\begin{lemma}
     Suppose that the Assumption 1 and 3, for $\eta<2/\rho^+_s$ we can get: 
    \begin{equation}
\begin{split}
\mathbb{E}\mathcal{k}g_\mathcal{I}^{(t)}(\theta^{(t)})\mathcal{k}^2_2&\leq 12\varepsilon_{\mathcal{I}}\mathbb{E}(\mathcal{k}\nabla f_{i_t}(\theta^{(t)})-\nabla f_{i_t}(\theta^*)\mathcal{k}^2+\mathcal{k}\nabla f_{i_t}(\theta^{(t-1)})-\nabla f_{i_t}(\theta^*)\mathcal{k}^2)\\
&+6((4\varepsilon_{\mathcal{I}}s+2)+\varepsilon_{\mathcal{I}^c}(d-k))\mathbb{E}\mathcal{k}\nabla f_{i_t}(\theta^*)\mathcal{k}^2_\infty+12\varepsilon_{abs}\mu^2+3\mathcal{k}\nabla_{\mathcal{I}}\mathcal{F}(\theta^*)\mathcal{k}^2_2.
\end{split}
    \end{equation}
\end{lemma}
Therefore, we derive the following conclusion.
\begin{theorem}
Suppose the same assumption as Theorem 1. For Algorithm \ref{alg:SPIDER}, we have:
\[
\begin{split}
    \gamma'\mathbb{E}\left[\mathcal{F}(\widetilde{\theta}^{(r)})-\mathcal{F}(\theta^*)\right]
&\le \left(\frac{2\beta^{m-1}}{\rho^-_s}+48\eta^2\rho^+_s\varepsilon_{\mathcal{I}}\alpha\beta^{m-2}\right)\mathbb{E}[\mathcal{F}(\widetilde{\theta}^{(r-1)})-\mathcal{F}(\theta^*)]\\
&+\frac{2\beta^{m-1}}{\rho^-_s}\sqrt{s}\mathcal{k}\nabla\mathcal{F}(\theta^*)\mathcal{k}_{\infty}\mathbb{E}\mathcal{k}\widetilde{\theta}^{(r-1)}-\theta^*\mathcal{k}_2+\frac{\beta^{m-1}-1}{\beta-1}\alpha L,\\
\end{split}
\]
here $\gamma'=\left(\frac{\beta^{m-1}-1}{\beta-1}2\eta-\left(\frac{\beta^{m-1}-1}{\beta-1}+\frac{\beta^{m-2}-1}{\beta-1}\right)48\varepsilon_{\mathcal{I}}\eta^2\rho^+_s)\alpha\right)$.
\end{theorem}
Based on the coefficient of $\mathbb{E}\left[\mathcal{F}(\widetilde{\theta}^{(r)})-\mathcal{F}(\theta^*)\right]$ and $\mathbb{E}\left[\mathcal{F}(\widetilde{\theta}^{(r-1)})-\mathcal{F}(\theta^*)\right]$, we derive the boundary of $\eta$ when the algorithm converges.
\begin{corollary}\label{col:3}
If
\[\frac{\alpha-\sqrt{\alpha^2-4(48\varepsilon_{\mathcal{I}}\alpha\rho^+_s+\alpha\rho^-_s)(\alpha-1)}}{2(48\varepsilon_{\mathcal{I}}\alpha\rho^+_s+\alpha\rho^-_s)}\le\eta\le\frac{\alpha+\sqrt{\alpha^2-4(48\varepsilon_{\mathcal{I}}\alpha\rho^+_s+\alpha\rho^-_s)(\alpha-1)}}{2(48\varepsilon_{\mathcal{I}}\alpha\rho^+_s+\alpha\rho^-_s)},\]
there always exists an $m$ such that the algorithm converges. 
\end{corollary}
\begin{remark}
    Like $p$M-SZHT, from corollary \ref{col:3}, we know that q does not play a decisive role in convergence. This indicates that it is feasible to use variance reduction algorithms instead of restricting $q$ to make the algorithm converge.
\end{remark}
\begin{remark}
    Unfortunately, in the SARAH-SZHT algorithm, the gradient estimation $\hat g$ will accumulate errors during the inner loop process. This will cause the algorithm to not converge when $m$ is large.
\end{remark}
 \begin{corollary}
     The query complexity of the algorithm is $\mathcal{O}([n+\frac{\kappa^3}{\kappa^2+1}]\log{(\frac{1}{\varepsilon})})$.
 \end{corollary} 
\subsection{Proof of Lemma 2}

\begin{proof}
It is straightforward that the stochastic variance reduced gradient satisfies:
\[\mathbb{E}\hat g^{(r)} (\theta^{(r)})=\mathbb{E}\hat\nabla f_{i_r}(\theta^{(r)})-\mathbb{E}\hat \nabla f_{i_r}({\theta}^{(r-1)})+\hat g^{(r-1)}({\theta}^{(r-1)})=\mathbb{E} \hat\nabla f_{i_r}(\theta^{(r)}).\]
Thus $\hat g^{(r)}(\theta^{(r)})$ is a unbiased estimator of $\hat\nabla\mathcal{F}(\theta^{(r)})$. The form of inapplicable zero order gradient is the same. As a result, the first claim is verified.\par
For the second claim, we have:
\begin{equation}\label{mylem1}
\begin{split}
\mathbb{E}\mathcal{k}g_\mathcal{I}^{(r)}(\theta^{(r)})\mathcal{k}^2_2\textbf{}\overset{\textrm{\ding{172}}}{\le}& 3\mathbb{E}\mathcal{k}\hat\nabla_\mathcal{I}f_{i_r}(\theta^{(r)})-\nabla_{\mathcal{I}}f_{i_r}(\theta^*)\mathcal{k}^2_2+3\mathcal{k}\nabla_{\mathcal{I}}\mathcal{F}(\theta^*)\mathcal{k}^2_2\\
&+3\mathbb{E}\mathcal{k}[\hat\nabla_\mathcal{I}f_{i_r}(\theta^{(r-1)})-\nabla_{\mathcal{I}}f_{i_r}(\theta^*)]-\hat g_{\mathcal{I}}^{(r-1)}(\theta^{(r-1)})+\nabla_{\mathcal{I}}\mathcal{F}(\theta^*)\mathcal{k}^2_2\\
=& 3\mathbb{E}\mathcal{k}\hat\nabla_\mathcal{I}f_{i_r}(\theta^{(r)})-\nabla_{\mathcal{I}}f_{i_r}(\theta^*)\mathcal{k}^2_2+3\mathcal{k}\nabla_{\mathcal{I}}\mathcal{F}(\theta^*)\mathcal{k}^2_2\\
&+3\mathbb{E}_{\bm{u}} [\mathbb{E}_{i_r}\mathcal{k}[\hat\nabla_\mathcal{I}f_{i_r}(\theta^{(r-1)})-\nabla_{\mathcal{I}}f_{i_r}(\theta^*)]-\hat\nabla_{\mathcal{I}}\mathcal{F}(\theta^{(r-1)})+\nabla_{\mathcal{I}}\mathcal{F}(\theta^*)\mathcal{k}^2_2]\\
\textbf{}\overset{\textrm{\ding{173}}}{\le}& 3\mathbb{E}\mathcal{k}\hat\nabla_\mathcal{I}f_{i_r}(\theta^{(r)})-\nabla_{\mathcal{I}}f_{i_r}(\theta^*)\mathcal{k}^2_2+3\mathcal{k}\nabla_{\mathcal{I}}\mathcal{F}(\theta^*)\mathcal{k}^2_2\\
&+3\mathbb{E}_{\bm{u}} [\mathbb{E}_{i_r}\mathcal{k}\hat\nabla_\mathcal{I}f_{i_r}(\theta^{(r-1)})-\nabla_{\mathcal{I}}f_{i_r}(\theta^*)\mathcal{k}^2_2]\\
=&3\mathbb{E}\mathcal{k}\hat\nabla_\mathcal{I}f_{i_r}(\theta^{(r)})-\nabla_{\mathcal{I}}f_{i_r}(\theta^*)\mathcal{k}^2_2+3\mathcal{k}\nabla_{\mathcal{I}}\mathcal{F}(\theta^*)\mathcal{k}^2_2\\
&+3\mathbb{E}\mathcal{k}\hat\nabla_\mathcal{I}f_{i_r}(\theta^{(r-1)})-\nabla_{\mathcal{I}}f_{i_r}(\theta^*)\mathcal{k}^2_2.\\
\end{split}
\end{equation}
The  inequality $\textrm{\ding{172}}$ follows from the power mean inequality$\mathcal{k}a+b+c\mathcal{k}^2_2\le3\mathcal{k}a\mathcal{k}^2_2+3\mathcal{k}b\mathcal{k}^2_2+3\mathcal{k}c\mathcal{k}^2_2$, and $\textrm{\ding{173}}$ is follows from $\mathbb{E}\mathcal{k}x-\mathbb{E}x\mathcal{k}^2_2\le\mathbb{E}\mathcal{k}x\mathcal{k}^2_2$. Now we focus on $\mathbb{E}\mathcal{k}\hat\nabla_\mathcal{I}f_{i_r}(\theta)-\nabla_{\mathcal{I}}f_{i_r}(\theta^*)\mathcal{k}^2_2$. Actually, the boundary of $\mathbb{E}\mathcal{k}\hat\nabla_{\mathcal{I}}f_{i_r}(\theta)-\nabla_{\mathcal{I}}f_{i_r}(\theta^*)\mathcal{k}^2_2$ is available in Lemma 2, that is

\begin{equation}\label{mylem0}
\begin{split}
\mathbb{E}\mathcal{k}\hat\nabla_\mathcal{I}f_{i_r}(\theta)-\nabla_{\mathcal{I}}f_{i_r}(\theta^*)\mathcal{k}^2_2&\leq 4\varepsilon_{\mathcal{I}}\mathbb{E}\mathcal{k}\nabla f_{i_r}(\theta^t)-\nabla f_{i_r}(\theta^*)\mathcal{k}^2\\
&+((4\varepsilon_{\mathcal{I}}s+2)+\varepsilon_{\mathcal{I}^c}(d-k))\mathbb{E}\mathcal{k}\nabla f_{i_r}(\theta^*)\mathcal{k}^2_\infty+2\varepsilon_{abs}\mu^2.
\end{split}
\end{equation}
Taking (\ref{mylem0}) into (\ref{mylem1}):
\[
\begin{split}
\mathbb{E}\mathcal{k}g_\mathcal{I}^{(r)}(\theta^{(r)})\mathcal{k}^2_2&\leq 12\varepsilon_{\mathcal{I}}\mathbb{E}(\mathcal{k}\nabla f_{i_r}(\theta^{(r)})-\nabla f_{i_r}(\theta^*)\mathcal{k}^2+\mathcal{k}\nabla f_{i_r}(\theta^{(r-1)})-\nabla f_{i_r}(\theta^*)\mathcal{k}^2)\\
&+6((4\varepsilon_{\mathcal{I}}s+2)+\varepsilon_{\mathcal{I}^c}(d-k))\mathbb{E}\mathcal{k}\nabla f_{i_r}(\theta^*)\mathcal{k}^2_\infty+12\varepsilon_{abs}\mu^2+3\mathcal{k}\nabla_{\mathcal{I}}\mathcal{F}(\theta^*)\mathcal{k}^2_2.
\end{split}
\]
\end{proof}
\subsection{Proof of Theorem 2}
\begin{proof}
We denote $v=\theta^{(r)}-\eta g_{\mathcal{I}}^{(r)}(\theta^{(r)})$ and $\mathcal{I}=\mathcal{I}^*\cup\mathcal{I}^{(r)}\cup\mathcal{I}^{(r+1)}$, where $\mathcal{I}^*=supp(\theta^{*})$, $\mathcal{I}^{(r)}=supp(\theta^{(r)})$ and $\mathcal{I}^{(r+1)}=supp(\theta^{(r+1)})$
\begin{equation}
\begin{split}\label{9}
\mathbb{E}||v-\theta^*||^2_2 &=\mathbb{E}||\theta^{(r)}-\theta^*||^2+\eta^2\mathbb{E}||\hat g^{(r)}_{\mathcal{I}}(\theta^{(r)})||^2_2-2\eta\left<\theta^{(r)}-\theta^*,\mathbb{E}\hat{g}^{(r)}_{\mathcal{I}}(\theta^{(r)})\right>\\
&=\mathbb{E}||\theta^{(r)}-\theta^*||^2_2+\eta^2\mathbb{E}||\hat{g}^{(r)}_{\mathcal{I}}(\theta^{(r)})||^2_2-2\eta\left<\theta^{(r)}-\theta^*,\mathbb{E}\hat\nabla_{\mathcal{I}}\mathcal{F}(\theta^{(r)})\right>\\
&=\mathbb{E}||\theta^{(r)}-\theta^*||^2_2+\eta^2\mathbb{E}||\hat{g}^{(r)}_{\mathcal{I}}(\theta^{(r)})||^2_2\\
&-2\eta\mathbb{E}\left<\theta^{(r)}-\theta^*,\hat\nabla_{\mathcal{I}}\mathcal{F}(\theta^{(r)})-\nabla_{\mathcal{I}}\mathcal{F}(\theta^{(r)})\right>-2\eta\mathbb{E}\left<\theta^{(r)}-\theta^*,\nabla_{\mathcal{I}}\mathcal{F}(\theta^{(r)})\right>\\
&\le\mathbb{E}||\theta^{(r)}-\theta^*||^2_2+\eta^2\mathbb{E}||\hat{g}^{(r)}_{\mathcal{I}}(\theta^{(r)})||^2_2-2\eta\mathbb{E}\left<\theta^{(r)}-\theta^*,\hat\nabla_{\mathcal{I}}\mathcal{F}(\theta^{(r)})-\nabla_{\mathcal{I}}\mathcal{F}(\theta^{(r)})\right>\\
&-2\eta\left[\mathcal{F}(\theta^{(r)})-\mathcal{F}(\theta^*)\right]\\
&=\mathbb{E}||\theta^{(r)}-\theta^*||^2_2+\eta^2\mathbb{E}||\hat{g}^{(r)}_{\mathcal{I}}(\theta^{(r)})||^2_2-2\eta\mathbb{E}\left<\sqrt{\eta}\rho^-_s(\theta^{(r)}-\theta^*),\frac{1}{\sqrt{\eta}\rho^-_s}(\hat\nabla_{\mathcal{I}}\mathcal{F}(\theta^{(r)})-\nabla_{\mathcal{I}}\mathcal{F}(\theta^{(r)}))\right>\\
&-2\eta\left[\mathcal{F}(\theta^{(r)})-\mathcal{F}(\theta^*)\right]\\
&\le(1+\eta^2{\rho^-_s}^2)\mathbb{E}||\theta^{(r)}-\theta^*||^2_2+\eta^2\mathbb{E}||\hat{g}^{(r)}_{\mathcal{I}}(\theta^{(r)})||^2_2-2\eta\left[\mathcal{F}(\theta^{(r)})-\mathcal{F}(\theta^*)\right]\\
&+\frac{1}{\rho^-_s}^2\mathbb{E}||(\hat\nabla_{\mathcal{I}}\mathcal{F}(\theta^{(r)})-\nabla_{\mathcal{I}}\mathcal{F}(\theta^{(r)}))||^2_2.\\
\end{split}
\end{equation}
For $\mathbb{E}||(\hat\nabla_{\mathcal{I}}\mathcal{F}(\theta^{(r)})-\nabla_{\mathcal{I}}\mathcal{F}(\theta^{(r)}))||^2_2$, we have:
\begin{equation}
\begin{split}\label{10}
\mathbb{E}||(\hat\nabla_{\mathcal{I}}\mathcal{F}(\theta^{(r)})-\nabla_{\mathcal{I}}\mathcal{F}(\theta^{(r)}))||^2_2&=||\mathbb{E}_{\bm u}(\hat\nabla_{\mathcal{I}}\sum^n_{i=1}{f_i}(\theta^{(r)})-\nabla_{\mathcal{I}}\sum^n_{i=1}{f_i}(\theta^{(r)}))||^2_2\\
&\le n\sum^n_{i=1}||\mathbb{E}_{\bm u}(\hat\nabla_{\mathcal{I}}{f_i}(\theta^{(r)})-\nabla_{\mathcal{I}}{f_i}(\theta^{(r)}))||^2_2\\
&\le n^2\varepsilon_\mu \mu^2.\\
\end{split}
\end{equation}
The first inequality follows from N-dimensional mean inequality and the second inequality follows from Assumption 2. By constraining $\mathbb{E}||(\hat\nabla_{\mathcal{I}}\mathcal{F}(\theta^{(r)})-\nabla_{\mathcal{I}}\mathcal{F}(\theta^{(r)}))||^2_2$, we can turn (\ref{9}) into:
\begin{equation}
\begin{split}\label{11}
\mathbb{E}||\theta^{(r)}-\eta \hat{g}^{(r)}_{\mathcal{I}}(\theta^{(r)})-\theta^*||^2_2&\le (1+\eta^2{\rho^-_s}^2)\mathbb{E}||\theta^{(r)}-\theta^*||^2_2+\eta^2\mathbb{E}||\hat{g}^{(r)}_{\mathcal{I}}(\theta^{(r)})||^2_2-2\eta\left[\mathcal{F}(\theta^{(r)})-\mathcal{F}(\theta^*)\right]\\
&+\frac{n^2\varepsilon_\mu \mu^2}{{\rho^-_s}^2}.\\
\end{split}
\end{equation} 
Let $\alpha=1+\frac{2\sqrt{k^*}}{\sqrt{k-k^*}}$. Using Lemma 2, we have:
\begin{equation}
\begin{split}\label{12}
\mathbb{E}||\theta^{(r+1)}-\theta^*||^2_2&\le (1+\eta^2{\rho^-_s}^2)\alpha\mathbb{E}||\theta^{(r)}-\theta^*||^2_2+\eta^2\alpha\mathbb{E}||\hat{g}^{(r)}_{\mathcal{I}}(\theta^{(r)})||^2_2-2\eta\alpha\left[\mathcal{F}(\theta^{(r)})-\mathcal{F}(\theta^*)\right]\\
&+\alpha\frac{n^2\varepsilon_\mu \mu^2}{{\rho^-_s}^2}\\
&\le (1+\eta^2{\rho^-_s}^2)\alpha\mathbb{E}||\theta^{(r)}-\theta^*||^2_2+\alpha\frac{n^2\varepsilon_\mu \mu^2}{{\rho^-_s}^2}-2\eta\alpha\left[\mathcal{F}(\theta^{(r)})-\mathcal{F}(\theta^*)\right]\\
&+\eta^2\alpha(12\varepsilon_{\mathcal{I}}\mathbb{E}(||\nabla f_{i_r}(\bm{x}^t)-\nabla f_{i_r}(\theta^*)||^2+||\nabla f_{i_r}(\theta^{(r-1)})-\nabla f_{i_r}(\theta^*)||^2)\\
&+6((4\varepsilon_{\mathcal{I}}s+2)+\varepsilon_{\mathcal{I}^c}(d-k))\mathbb{E}||\nabla f_{i_r}(\theta^*)||^2_\infty+12\varepsilon_{abs}\mu^2+3||\nabla_{\mathcal{I}}\mathcal{F}(\theta^*)||^2_2).\\
\end{split}
\end{equation} 
For $||\nabla f_{i_r}(\bm{\theta}^t)-\nabla f_{i_r}(\theta^*)||^2$, we can easily get 
$$\mathbb{E}||\nabla f_{i_r}(\bm{\theta}^t)-\nabla f_{i_r}(\theta^*)||^2=\frac{1}{n}||\nabla f_{i_r}(\bm{\theta}^t)-\nabla f_{i_r}(\theta^*)||^2\le 4\rho^+_s[\mathcal{F}(\theta)-\mathcal{F}(\theta^*)]$$ 
by RSS condition. As a reasult, we have: 
\begin{equation}
\begin{split}
\mathbb{E}||\theta^{(r+1)}-\theta^*||^2_2 &\le (1+\eta^2{\rho^-_s}^2)\alpha\mathbb{E}||\theta^{(r)}-\theta^*||^2_2+\alpha\frac{n^2\varepsilon_\mu \mu^2}{{\rho^-_s}^2}-\alpha(2\eta-48\varepsilon_{\mathcal{I}}\eta^2\rho^+_s)\left[\mathcal{F}(\theta^{(r)})-\mathcal{F}(\theta^*)\right]\\
&+48\eta^2\rho^+_s\alpha\varepsilon_{\mathcal{I}}\mathbb{E}[\mathcal{F}(\theta^{(r-1)})-\mathcal{F}(\theta^*)]\\
&+6\eta^2\alpha((4\varepsilon_{\mathcal{I}}s+2)+\varepsilon_{\mathcal{I}^c}(d-k))\mathbb{E}||\nabla f_{i_r}(\theta^*)||^2_\infty+12\varepsilon_{abs}\mu^2+3||\nabla_{\mathcal{I}}\mathcal{F}(\theta^*)||^2_2).\\
\end{split}
\end{equation}
Let $L=6\eta^2((4\varepsilon_{\mathcal{I}}s+2)+\varepsilon_{\mathcal{I}^c}(d-k))\mathbb{E}||\nabla f_{i_r}(\theta^*)||^2_\infty+12\varepsilon_{abs}\mu^2+3||\nabla_{\mathcal{I}}\mathcal{F}(\theta^*)||^2_2)+\frac{n^2\varepsilon_\mu \mu^2}{{\rho^-_s}^2}$, $\beta= (1+\eta^2{\rho_s^-}^2)\alpha$, then:
\begin{equation}\label{21}
\begin{split}
\mathbb{E}||\theta^{(r+1)}-\theta^*||^2_2+\alpha(2\eta-48\varepsilon_{\mathcal{I}}\eta^2\rho^+_s)\mathbb{E}\left[\mathcal{F}(\theta^{(r)})-\mathcal{F}(\theta^*)\right] &\le \beta\mathbb{E}||\theta^{(r)}-\theta^*||^2_2\\
&+48\eta^2\rho^+_s\varepsilon_{\mathcal{I}}\alpha\mathbb{E}[\mathcal{F}(\theta^{(r-1)})-\mathcal{F}(\theta^*)]+\alpha L.\\
\end{split}
\end{equation}
 By summing (\ref{21}) over $t=1,\ldots,m-1$, Let $G=\mathbb{E}||\theta^{(m)}-\theta^*||^2_2+\frac{\beta^{m-1}-1}{\beta-1}(2\eta-48\varepsilon_{\mathcal{I}}\eta^2\rho^+_s)\alpha\left[\mathcal{F}(\widetilde{\theta}^{(r)})-\mathcal{F}(\theta^*)\right]$we have:
\begin{equation}
\begin{split}\label{next}
 G &\le \beta^{m-1}\mathbb{E}||\theta^{(r)}-\theta^*||^2_2+48\eta^2\rho^+_s\varepsilon_{\mathcal{I}}\alpha\frac{\beta^{m-2}-1}{\beta-1}\mathbb{E}[\mathcal{F}(\widetilde{\theta}^{(r)})-\mathcal{F}(\theta^*)]\\
&+48\eta^2\rho^+_s\varepsilon_{\mathcal{I}}\alpha\beta^{m-2}\mathbb{E}[\mathcal{F}(\widetilde{\theta}^{(r-1)})-\mathcal{F}(\theta^*)]+\frac{\beta^{m-1}-1}{\beta-1}\alpha L.\\
\end{split}
\end{equation}
That is
\[
    \begin{split}
    &\mathbb{E}||\theta^{(m)}-\theta^*||^2_2+(\frac{\beta^{m-1}-1}{\beta-1}2\eta-(\frac{\beta^{m-1}-1}{\beta-1}+\frac{\beta^{m-2}-1}{\beta-1})48\varepsilon_{\mathcal{I}}\eta^2\rho^+_s)\alpha\left[\mathcal{F}(\widetilde{\theta}^{(r)})-\mathcal{F}(\theta^*)\right]\\ &\le \beta^{m-1}\mathbb{E}||\theta^{(r)}-\theta^*||^2_2+48\eta^2\rho^+_s\varepsilon_{\mathcal{I}}\alpha\beta^{m-2}\mathbb{E}[\mathcal{F}(\widetilde{\theta}^{(r-1)})-\mathcal{F}(\theta^*)]+\frac{\beta^{m-1}-1}{\beta-1}\alpha L.\\
    \end{split}
\]
Through RSC condition and the definition of $\widetilde{I}$, it further follows from (\ref{next}) that:
\begin{equation}
\begin{split}\label{18}
&\mathbb{E}||\theta^{(m)}-\theta^*||^2_2+(\frac{\beta^{m-1}-1}{\beta-1}2\eta-(\frac{\beta^{m-1}-1}{\beta-1}+\frac{\beta^{m-2}-1}{\beta-1})48\varepsilon_{\mathcal{I}}\eta^2\rho^+_s)\alpha\left[\mathcal{F}(\widetilde{\theta}^{(r)})-\mathcal{F}(\theta^*)\right]\\
 &\le \left(\frac{2\beta^{m-1}}{\rho^-_s}+48\eta^2\rho^+_s\varepsilon_{\mathcal{I}}\alpha\beta^{m-2}\right)\mathbb{E}[\mathcal{F}(\widetilde{\theta}^{(r-1)})-\mathcal{F}(\theta^*)]+\frac{2\beta^{m-1}}{\rho^-_s}\mathbb{E}\left<\nabla\mathcal{F}(\theta^*),\widetilde{\theta}^*-\theta^{(r-1)}\right>+\frac{\beta^{m-1}-1}{\beta-1}\alpha L\\
&\le \left(\frac{2\beta^{m-1}}{\rho^-_s}+48\eta^2\rho^+_s\varepsilon_{\mathcal{I}}\alpha\beta^{m-2}\right)\mathbb{E}[\mathcal{F}(\widetilde{\theta}^{(r-1)})-\mathcal{F}(\theta^*)]+\frac{2\beta^{m-1}}{\rho^-_s}\sqrt{s}||\nabla\mathcal{F}(\theta^*)||_{\infty}\mathbb{E}||\widetilde{\theta}^{(r-1)}-\theta^*||_2+\frac{\beta^{m-1}-1}{\beta-1}\alpha L.\\
\end{split}
\end{equation}
Here $\varepsilon_\mu={\rho^+_{s}}^2sd$, $\varepsilon_{\mathcal{I}}=\frac{2d}{q(s_2+2)}(\frac{(s-1)(s_2-1)}{d-1}+3)+2$, $\varepsilon_{\mathcal{I}^c}=\frac{2d}{q(s_2+2)}\left(\frac{s(s_2-1)}{-1}\right)$, $\varepsilon_{abs}=\frac{2d{\rho^+_s}^2ss_2}{q}\left(\frac{(s-1)(s_2-1)}{d-1}+1\right)+{\rho^+_s}^2sd$
\end{proof}
\subsection{Proof of Corollary 2}

If the algorithm converges, we have:
\[
\frac{2\beta^{m-1}-1}{\alpha\rho^-_s}+48\varepsilon_{\mathcal{I}}\eta^2\rho^+_s\beta^{m-2}\le \frac{\beta^{m-1}-1}{\beta-1}2\eta-\beta^{m-2}48\varepsilon_{\mathcal{I}\eta^2\rho^+_s}.
\]
which is:
\begin{equation}\label{equ:converb}
\frac{\beta-1}{\alpha\rho^-_s\eta}+\frac{48\varepsilon_{\mathcal{I}}\eta^2\rho^+_s(\beta-1)}{\beta}\le 1-\frac{1}{\beta^{m-1}}.
\end{equation}
If there is a $m\in\mathbb{N}^+$ that holds (\ref{equ:converb}), then there must be:
\[\frac{\beta-1}{\alpha\rho^-_s\eta}+\frac{48\varepsilon_{\mathcal{I}}\eta^2\rho^+_s(\beta-1)}{\beta}\le 1.\]
Recall $\beta=\alpha(1+\eta^2\rho^-_s)^2$, we have:
\[
    (48\varepsilon_{\mathcal{I}}\alpha\rho^+_s+\alpha\rho^-_s)\eta^2-\alpha\eta+(\alpha-1)\le 0.
\]
This is a quadratic inequality about $\eta$, and a conclusion can be obtained through discriminant.

\section{Extra experiments}

\subsection{Ridge Regression}\label{sec:ridgetoy}

In this section we provide additional curves for the first problem described in section \ref{sec:syntreal}, which we recall here for sake of completeness. We consider a ridge regression problem, where each function $f_i$ is defined as follows: 
$$ f_i(\theta) = (x_i^{\top}\theta - y_i)^2 + \frac{\lambda}{2} \| \theta\|_2^2,$$
where $\lambda$ is some regularization parameter.

\subsection{Synthetic Experiment}

\paragraph{Experimental Setting} First, as in the main paper, we consider a synthetic dataset: we generate each $x_i$ randomly from a unit norm ball in $\mathbb{R}^d$, and a true random model $\theta^*$ from a normal distribution $\mathcal{N}(0, I_{d \times d})$.  Each $y_i$ is defined as $y_i = x_i^{\top} \theta^*$. We set the constants of the problem as such: $n=10, d=5, \lambda = 0.5$. Before training, we pre-process each column by subtracting its mean and dividing it by its empirical standard deviation. We run each algorithm with $\mu=10^{-4}, s_2=d=5$, and for the variance reduced algorithms, we choose $m=10$. For all algorithms, the learning rate $\eta$ is found through grid-search in $\{0.005, 0.01, 0.05, 0.1, 0.5\}$, and we keep the $\eta$ giving the lowest function value (averaged over 3 runs) at the end of training. We stop each algorithm once its number of IZO reaches 80,000. We plot in Figures \ref{fig:k2_curves},  \ref{fig:k3_curves}, and  \ref{fig:k4_curves} the mean and standard deviation of the curves for a value of $k=2$, $3$, and $4$ respectively.

\begin{figure}[htbp]
  \centering
  \begin{subfigure}{0.3\textwidth}
    \centering
    \includegraphics[width=\linewidth]{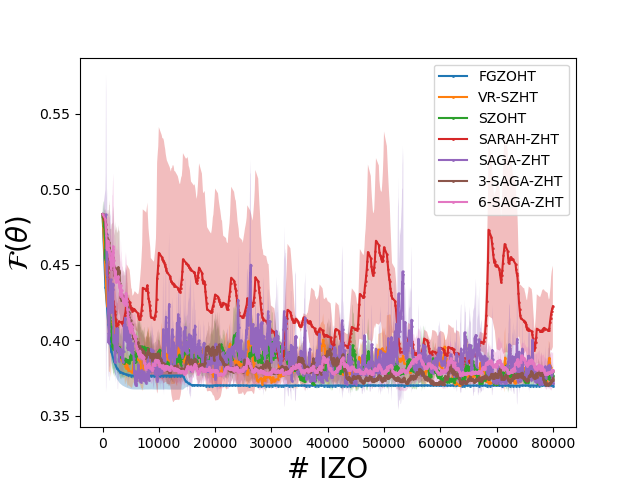}
    \caption{q=50}
    \label{fig:q50_k2_izo}
  \end{subfigure}
  \begin{subfigure}{0.3\textwidth}
    \centering
    \includegraphics[width=\linewidth]{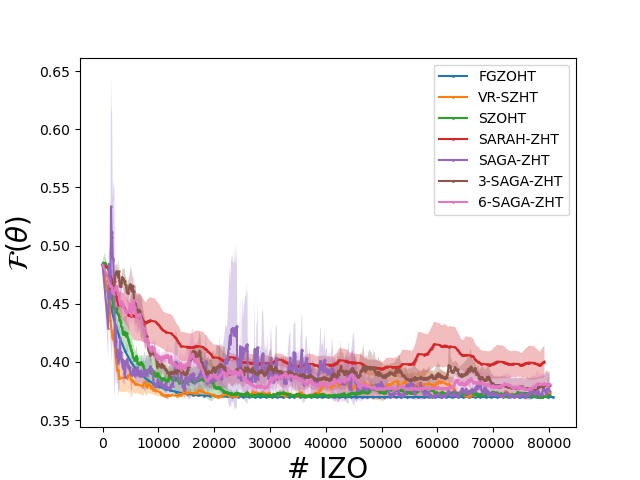}
    \caption{q=100}
    \label{fig:q100_k2_izo}
  \end{subfigure}
  \begin{subfigure}{0.3\textwidth}
    \centering
    \includegraphics[width=\linewidth]{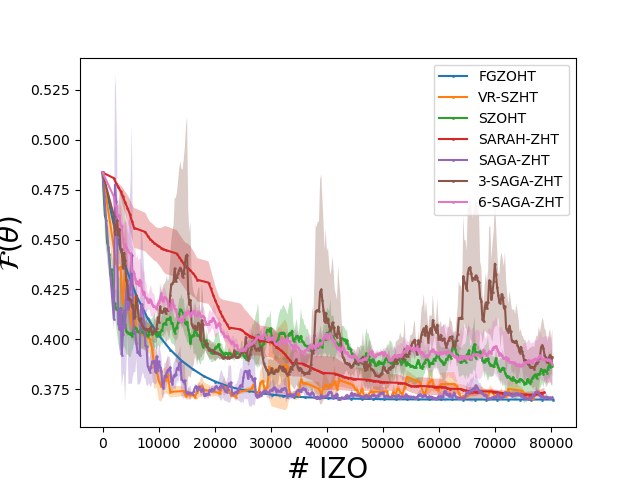}
    \caption{q=200}
    \label{fig:q200_k2_izo}
  \end{subfigure}
  
  \medskip
  
  \begin{subfigure}{0.3\textwidth}
    \centering
    \includegraphics[width=\linewidth]{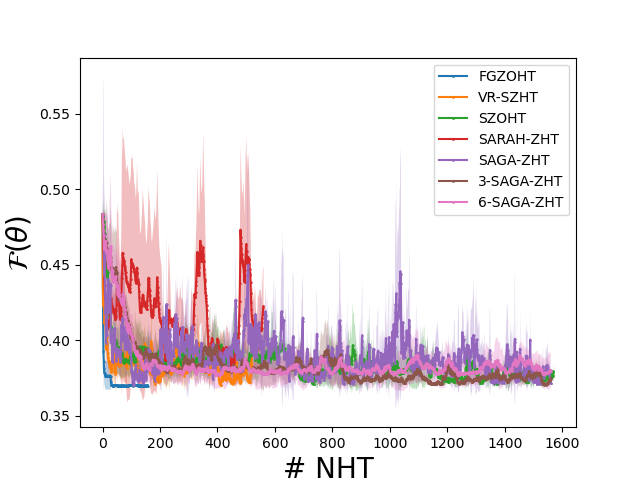}
    \caption{q=50}
    \label{fig:q50_k2_nht}
  \end{subfigure}
  \begin{subfigure}{0.3\textwidth}
    \centering
    \includegraphics[width=\linewidth]{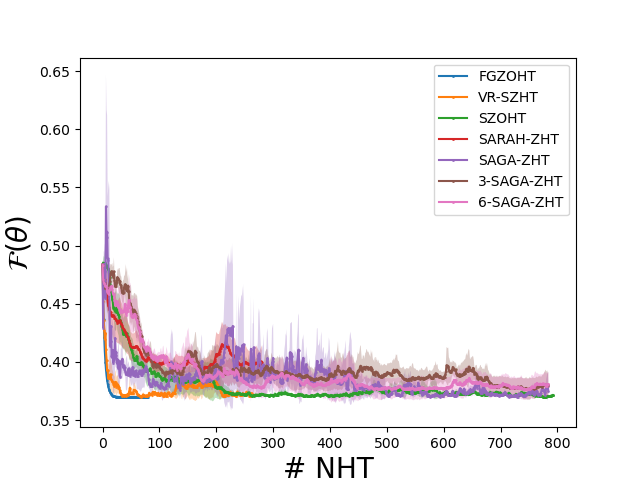}
    \caption{q=100}
    \label{fig:q100_k2_nht}
  \end{subfigure}
  \begin{subfigure}{0.3\textwidth}
    \centering
    \includegraphics[width=\linewidth]{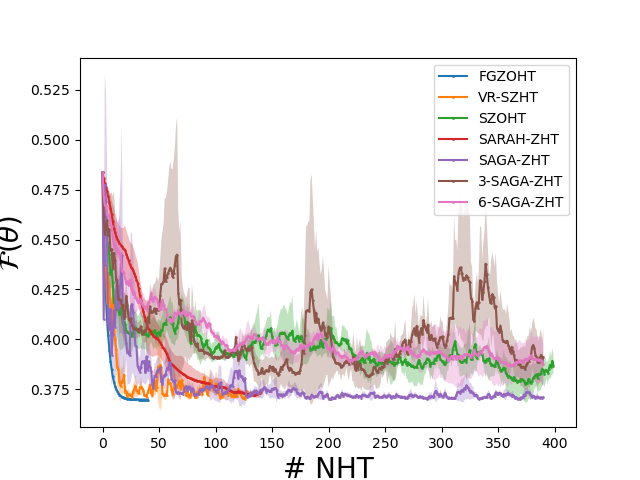}
    \caption{q=200}
    \label{fig:q200_k2_nht}
  \end{subfigure}

  \caption{\#IZO (up) and \#NHT (down) on the ridge regression task, synthetic example (k=2).}
  \label{fig:k2_curves}
\end{figure}

\begin{figure}[htbp]
  \centering
  \begin{subfigure}{0.3\textwidth}
    \centering
    \includegraphics[width=\linewidth]{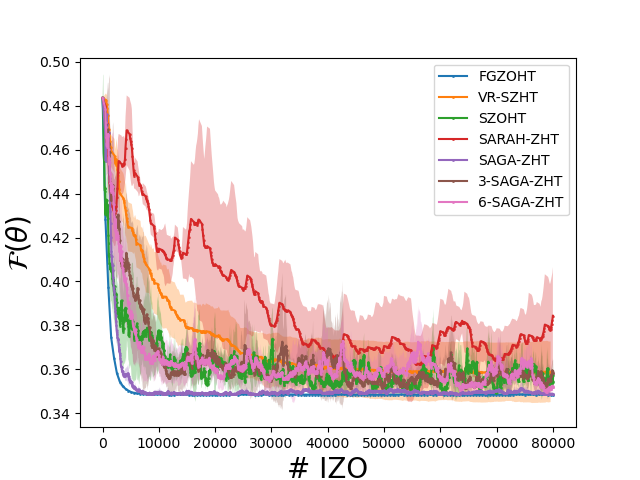}
    \caption{q=50}
    \label{fig:q50_k3_izo}
  \end{subfigure}
  \begin{subfigure}{0.3\textwidth}
    \centering
    \includegraphics[width=\linewidth]{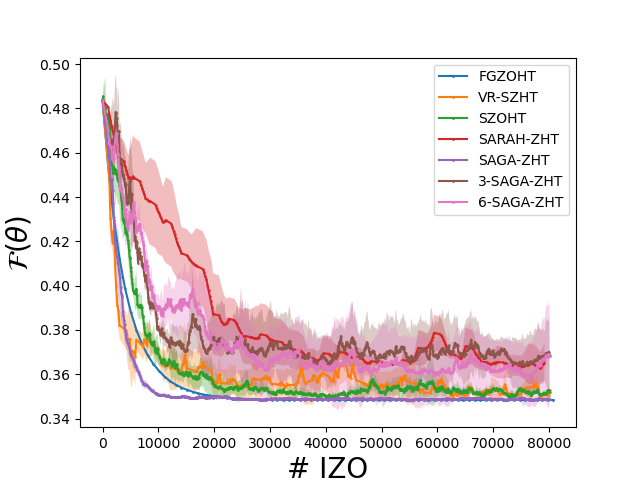}
    \caption{q=100}
    \label{fig:q100_k3_izo}
  \end{subfigure}
  \begin{subfigure}{0.3\textwidth}
    \centering
    \includegraphics[width=\linewidth]{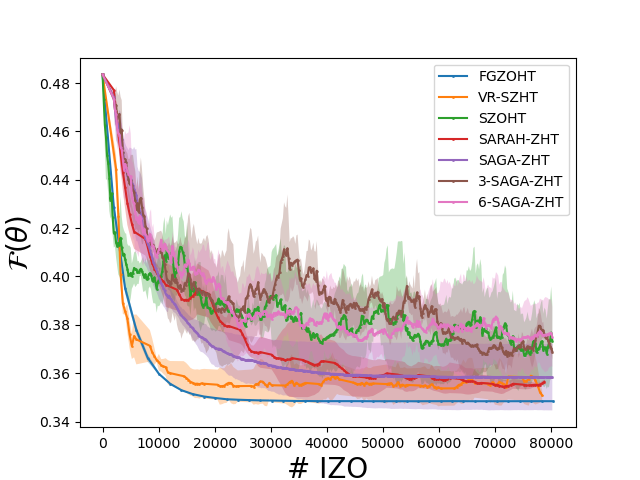}
    \caption{q=200}
    \label{fig:q200_k3_izo}
  \end{subfigure}
  
  \medskip
  
  \begin{subfigure}{0.3\textwidth}
    \centering
    \includegraphics[width=\linewidth]{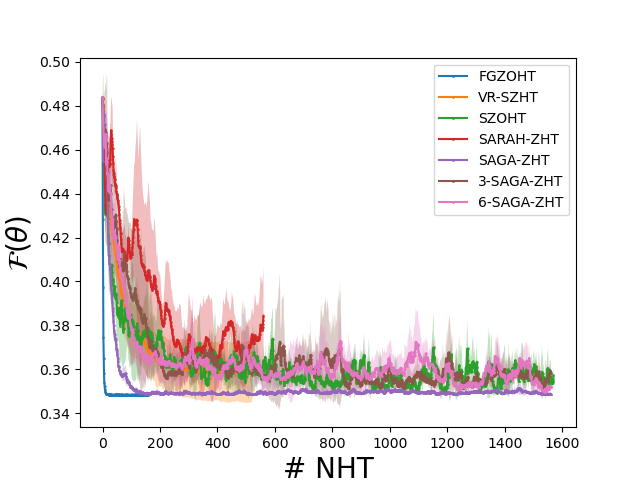}
    \caption{q=50}
    \label{fig:q50_k3_nht}
  \end{subfigure}
  \begin{subfigure}{0.3\textwidth}
    \centering
    \includegraphics[width=\linewidth]{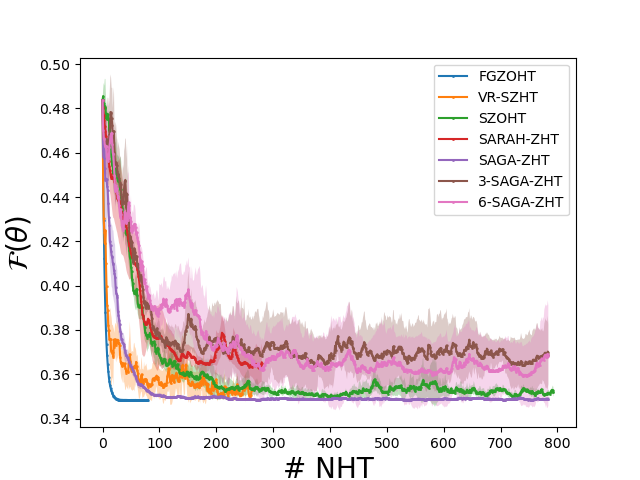}
    \caption{q=100}
    \label{fig:q100_k3_nht}
  \end{subfigure}
  \begin{subfigure}{0.3\textwidth}
    \centering
    \includegraphics[width=\linewidth]{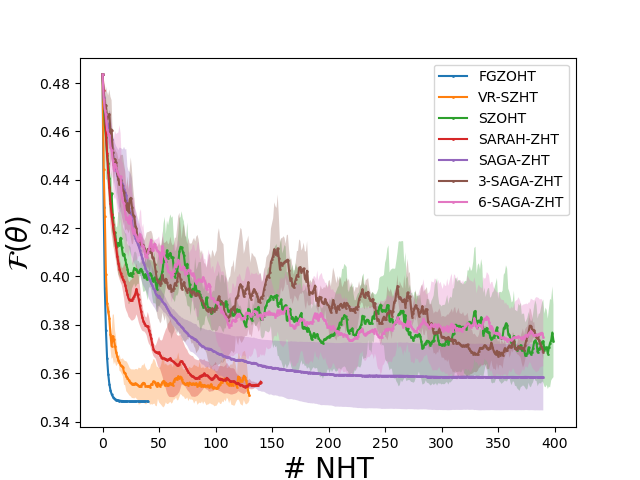}
    \caption{q=200}
    \label{fig:q200_k3_nht}
  \end{subfigure}

  \caption{\#IZO (up) and \#NHT (down) on the ridge regression task, synthetic example (k=3).}
  \label{fig:k3_curves}
\end{figure}
\begin{figure}[htbp]
  \centering
  \begin{subfigure}{0.3\textwidth}
    \centering
    \includegraphics[width=\linewidth]{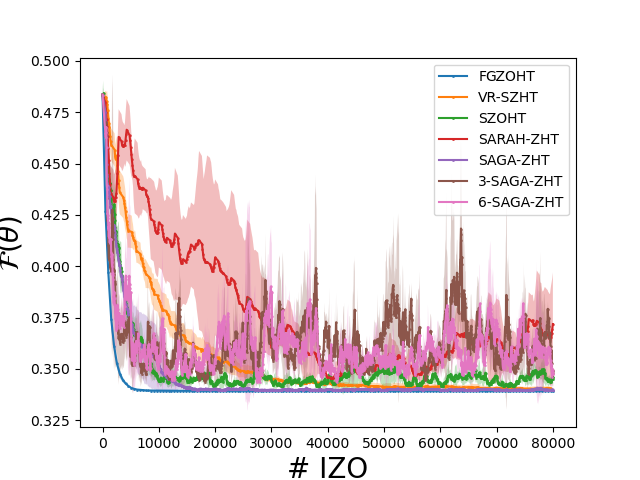}
    \caption{q=50}
    \label{fig:q50_k4_izo}
  \end{subfigure}
  \begin{subfigure}{0.3\textwidth}
    \centering
    \includegraphics[width=\linewidth]{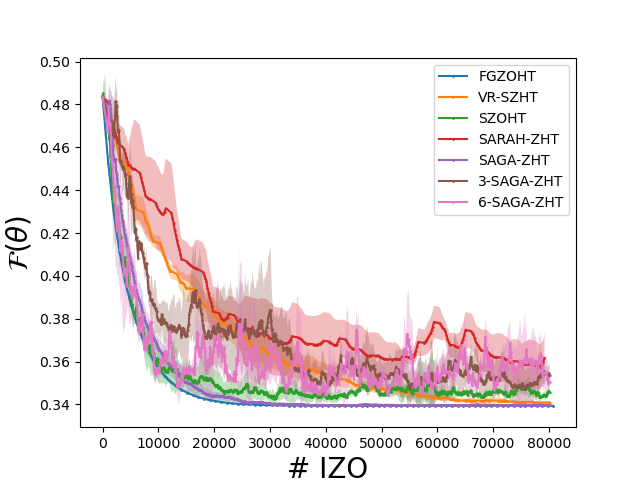}
    \caption{q=100}
    \label{fig:q100_k4_izo}
  \end{subfigure}
  \begin{subfigure}{0.3\textwidth}
    \centering
    \includegraphics[width=\linewidth]{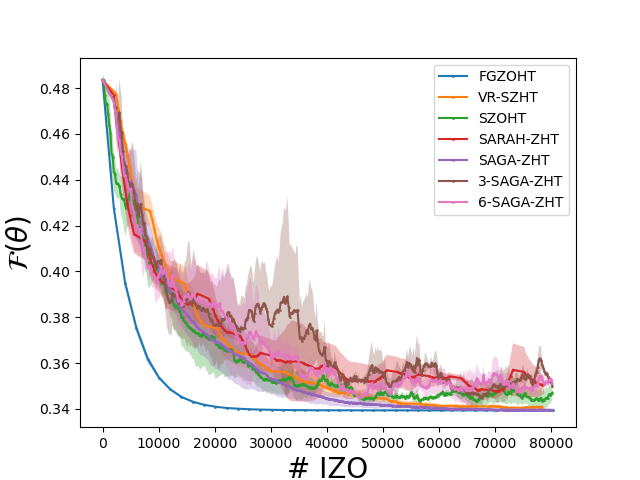}
    \caption{q=200}
    \label{fig:q200_k4_izo}
  \end{subfigure}
  
  \medskip
  
  \begin{subfigure}{0.3\textwidth}
    \centering
    \includegraphics[width=\linewidth]{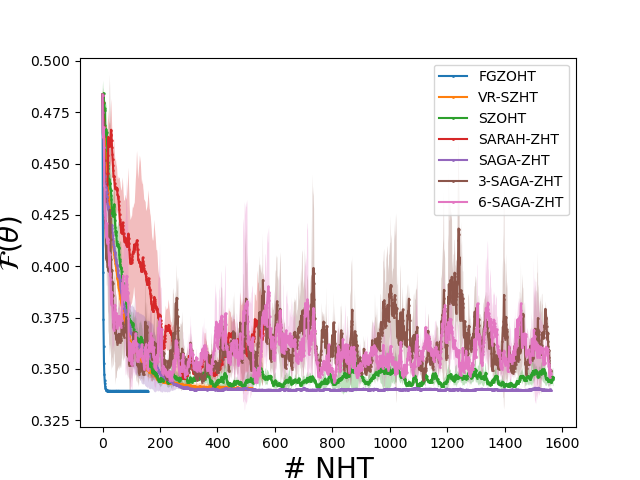}
    \caption{q=50}
    \label{fig:q50_k3_nht}
  \end{subfigure}
  \begin{subfigure}{0.3\textwidth}
    \centering
    \includegraphics[width=\linewidth]{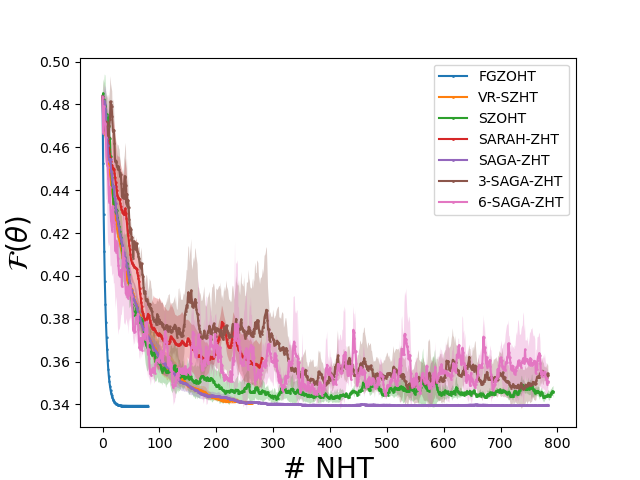}
    \caption{q=100}
    \label{fig:q100_k3_nht}
  \end{subfigure}
  \begin{subfigure}{0.3\textwidth}
    \centering
    \includegraphics[width=\linewidth]{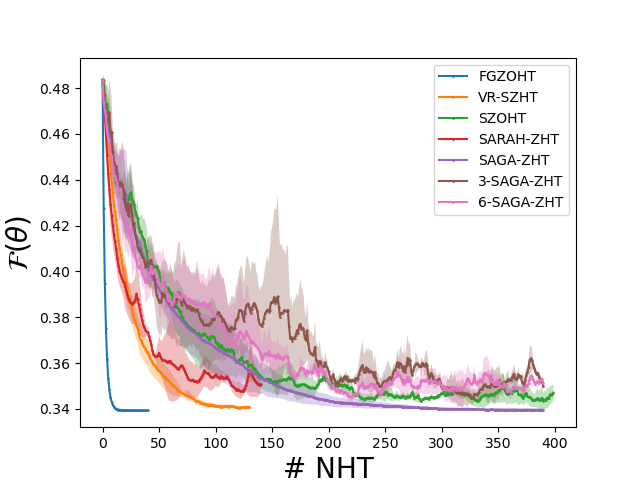}
    \caption{q=200}
    \label{fig:q200_k4_nht}
  \end{subfigure}

  \caption{\#IZO (up) and \#NHT (down) on the ridge regression task, synthetic example (k=4).}
  \label{fig:k4_curves}
\end{figure}


\paragraph{Results and Discussion}
We can observe the several phenomena on the Figures \ref{fig:k2_curves},  \ref{fig:k3_curves}, and  \ref{fig:k4_curves}. First, we can observe that for larger $k$, the algorithms converge to lower function values (which is natural because optimization is then over a larger set), but also, the algorithms are more stable (for example, SARAH-SZHT converges more easily with $k=4$ than with $k=2$), which is due to the hard-thresholding operator being more non-expansive. Then, although a larger number of random directions $q$ may slow down the query complexity (IZO), we observe that it can also stabilize some algorithms that would otherwise be less unstable, such as SARAH-SZHT (which converges better for $q=200$ than for $q=50$).

\subsection{Real-life Datasets}

\paragraph{Experimental Setting} Second, we now compare the above algorithms on the following open source real-life datasets (obtained from OpenML \cite{vanschoren2014openml}), of which a summary is presented in Table \ref{tab:ds}. We take $\lambda=0.5$.  Similarly as above, before training, we pre-process each column by subtracting its mean and dividing it by its empirical standard deviation. We run each algorithm with $\mu=10^{-4}, s_2=d$ (where $d$ depends on the dataset), and for the variance reduced algorithms with an inner and outer loop (VR-SZHT and SARAH-SZHT), we choose $m=\floor \frac{n}{2} \rfloor$. For all algorithms, the learning rate $\eta$ is found through grid-search in $\{10^{-i}, ~i\in \{1, ..., 7\}\}$, and we choose the one giving the lowest function value (averaged over 5 runs) at the end of training. We stop each algorithm once its number of IZO reaches 100,000. We plot the optimization curves (averaged over the 5 runs) for several values of $q$ and $k$, to study their impact on the convergence. 


\begin{table}[!bht]
  \caption{Datasets used in the  comparison. \textit{Reference:} \cite{Dua2019}, \textit{Source:}\cite{vanschoren2014openml}, downloaded with \texttt{scikit-learn} \cite{pedregosa2011scikit}.}
\label{tab:ds}
\vskip 0.15in
\begin{center}
\begin{small}
\begin{sc}
  \begin{tabular}{lccc}
  \toprule
 Dataset & $d$ & $n$\\
  \midrule
  \textbf{bodyfat}$^{(2)}$ & 14 & 252\\
  \textbf{auto-price}$^{(3)}$ & 15 & 159\\
\bottomrule
\end{tabular}
\end{sc}
\end{small}
\end{center}
\vskip -0.1in
\end{table}

\paragraph{Results and Discussion}


We present our results in Figures \ref{fig:bodyfat_curves} and \ref{fig:autoprice_curves}. Those results are consistent with preliminary results on the synthetic dataset from Section \ref{sec:ridgetoy}, namely, that overall, although taking a larger $q$ may worsen the IZO complexity, it can help some algorithms to converge more smoothly, by reducing the error of the zeroth-order estimator. Additionally, we can observe that taking a larger $k$ often helps to achieve smoother convergence. Finally, consistently across experiments, we observe that SARAH-SZHT has difficulties converging: this seems to indicate that SARAH-SZHT may be highly impacted by the errors introduced by the zeroth-order estimator. SARAH-SZHT could potentially be improved by a more careful choice of the number of inner iterations, and/or by running SARAH+, which is an adaptive version of SARAH \cite{nguyen2017linear}, which we leave for future work.

\begin{figure}[htbp]
  \centering
      \begin{subfigure}{0.245\textwidth}
    \centering
    \includegraphics[width=\linewidth]{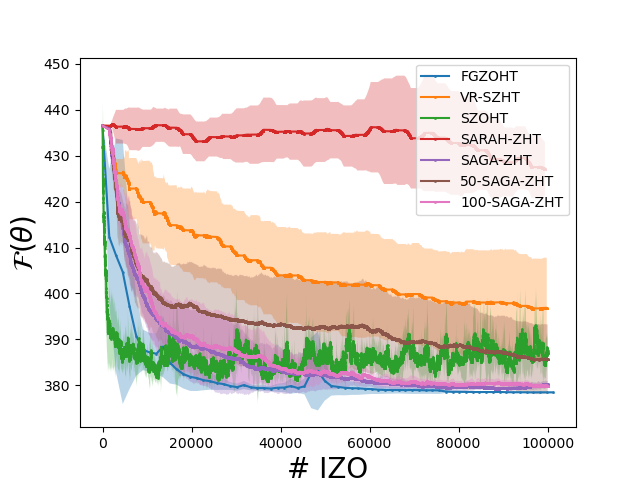}
    \caption{q=5, k=5}
    \label{fig:bodyfat_q5_k5_izo}
  \end{subfigure}
    \begin{subfigure}{0.245\textwidth}
    \centering
    \includegraphics[width=\linewidth]{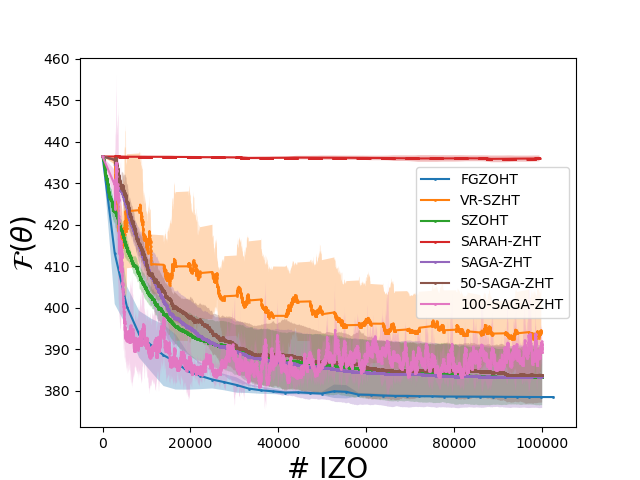}
    \caption{q=10, k=5}
    \label{fig:bodyfat_q10_k5_izo}
  \end{subfigure}
    \begin{subfigure}{0.245\textwidth}
    \centering
    \includegraphics[width=\linewidth]{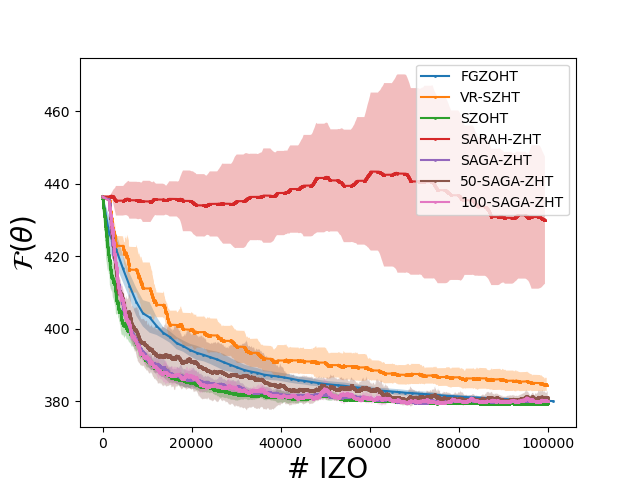}
    \caption{q=5, k=10}
    \label{fig:bodyfat_q5_k10_izo}
  \end{subfigure}
  \begin{subfigure}{0.245\textwidth}
    \centering
    \includegraphics[width=\linewidth]{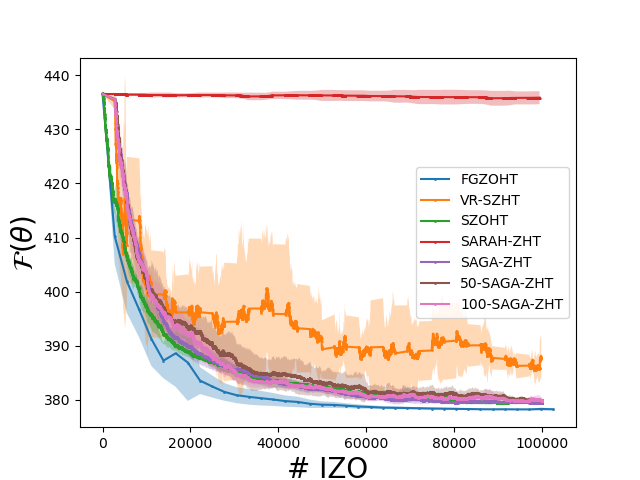}
    \caption{q=10, k=10}
    \label{fig:bodyfat_q10_k10_izo}
  \end{subfigure}
  
  \medskip
  
    \begin{subfigure}{0.245\textwidth}
    \centering
    \includegraphics[width=\linewidth]{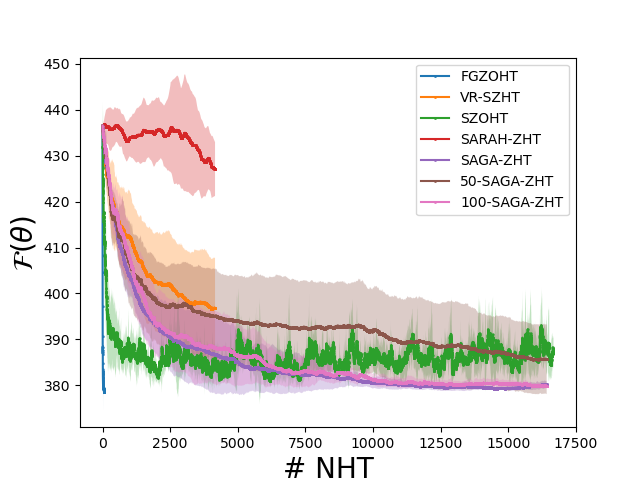}
    \caption{q=5, k=5}
    \label{fig:bodyfat_q5_k5_nht}
  \end{subfigure}
    \begin{subfigure}{0.245\textwidth}
    \centering
    \includegraphics[width=\linewidth]{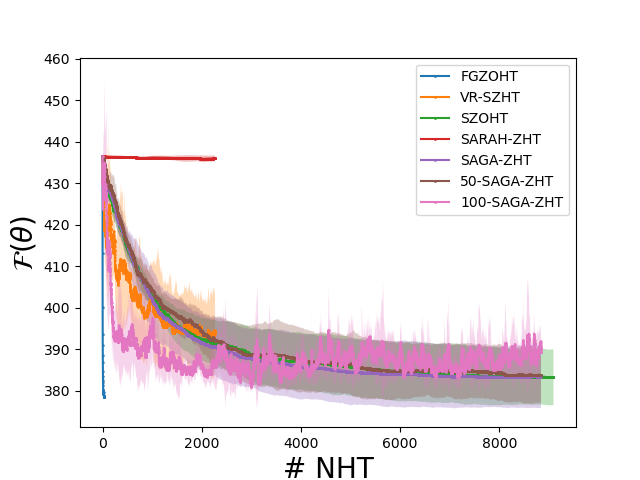}
    \caption{q=10, k=5}
    \label{fig:bodyfat_q10_k5_nht}
  \end{subfigure}
  \begin{subfigure}{0.245\textwidth}
    \centering
    \includegraphics[width=\linewidth]{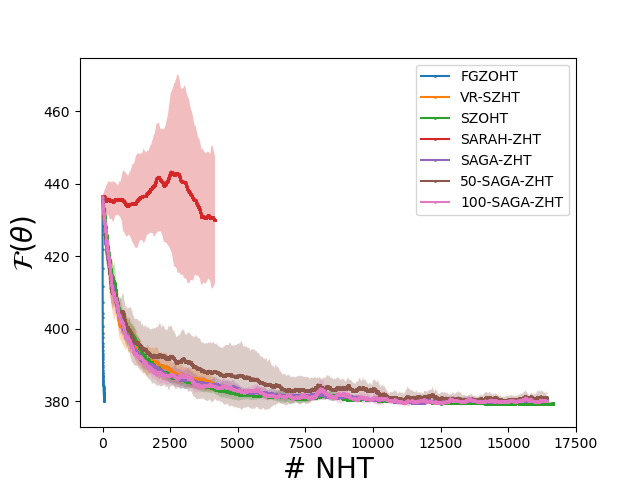}
    \caption{q=5, k=10}
    \label{fig:bodyfat_q5_k10_nht}
  \end{subfigure}
   \begin{subfigure}{0.245\textwidth}
    \centering
    \includegraphics[width=\linewidth]{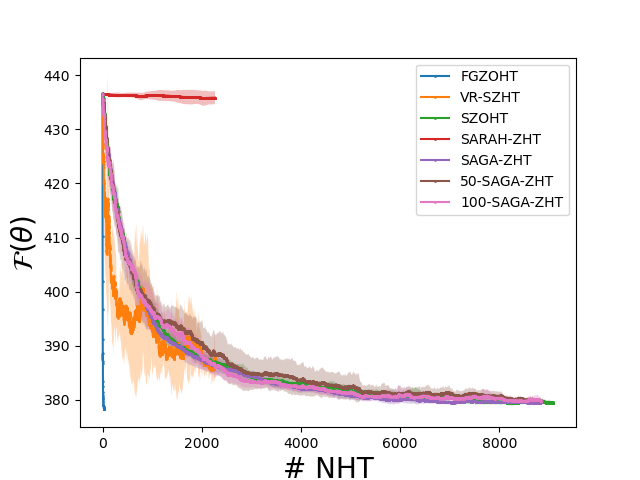}
    \caption{q=10, k=10}
    \label{fig:bodyfat_q10_k10_nht}
  \end{subfigure}

  \caption{\#IZO (up) and \#NHT (down) on the ridge regression task, \texttt{bodyfat} dataset.}
  \label{fig:bodyfat_curves}

\end{figure}

\begin{figure}[htbp]
  \centering
      \begin{subfigure}{0.245\textwidth}
    \centering
    \includegraphics[width=\linewidth]{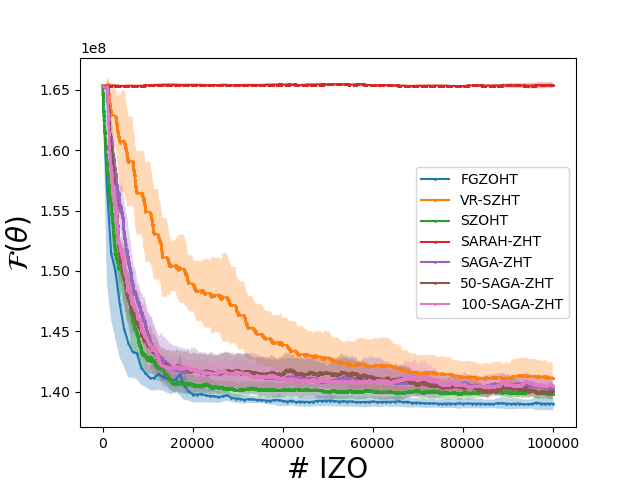}
    \caption{q=5, k=5}
    \label{fig:autoprice_q5_k5_izo}
  \end{subfigure}
    \begin{subfigure}{0.245\textwidth}
    \centering
    \includegraphics[width=\linewidth]{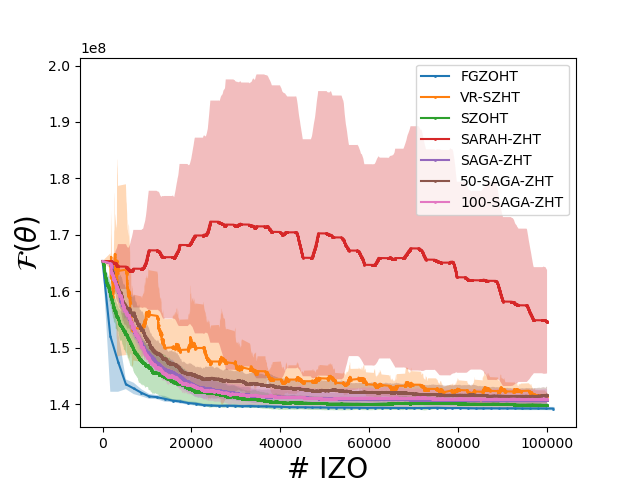}
    \caption{q=10, k=5}
    \label{fig:autoprice_q10_k5_izo}
  \end{subfigure}
    \begin{subfigure}{0.245\textwidth}
    \centering
    \includegraphics[width=\linewidth]{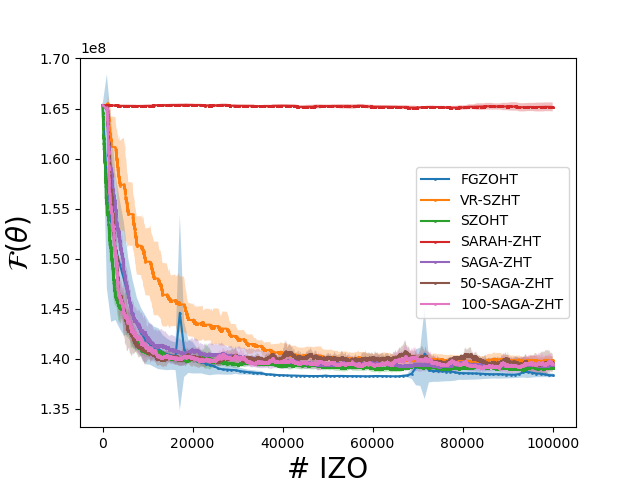}
    \caption{q=5, k=10}
    \label{fig:autoprice_q5_k10_izo}
  \end{subfigure}
  \begin{subfigure}{0.245\textwidth}
    \centering
    \includegraphics[width=\linewidth]{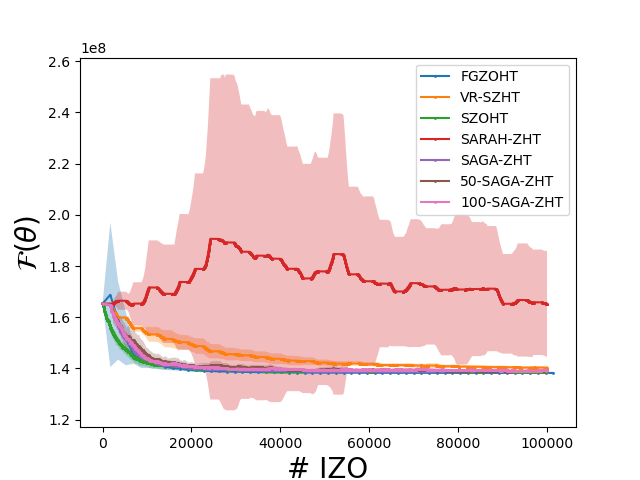}
    \caption{q=10, k=10}
    \label{fig:autoprice_q10_k10_izo}
  \end{subfigure}
  
  \medskip
  
    \begin{subfigure}{0.245\textwidth}
    \centering
    \includegraphics[width=\linewidth]{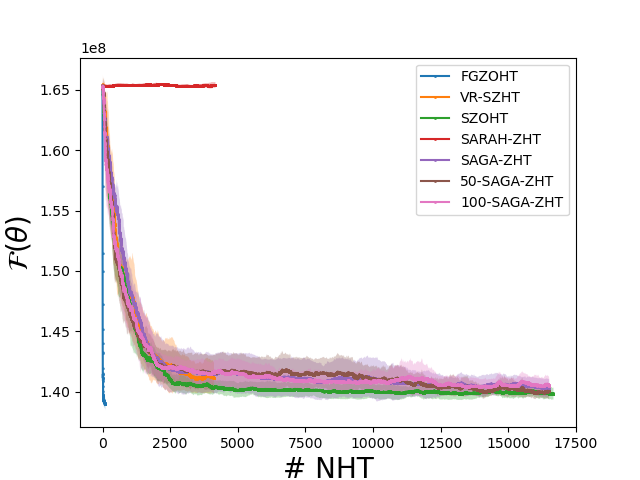}
    \caption{q=5, k=5}
    \label{fig:autoprice_q5_k5_nht}
  \end{subfigure}
    \begin{subfigure}{0.245\textwidth}
    \centering
    \includegraphics[width=\linewidth]{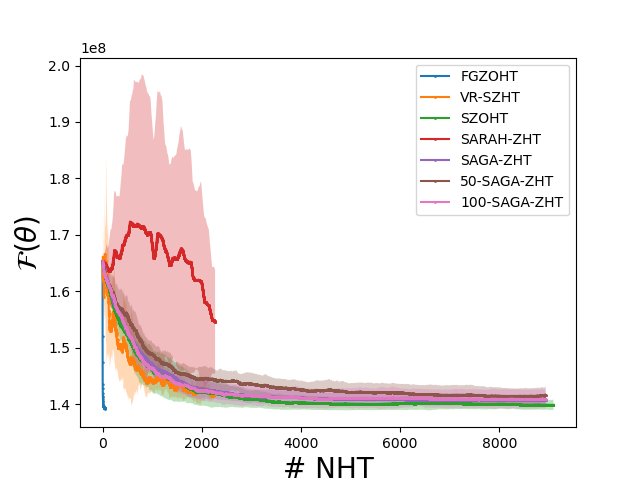}
    \caption{q=10, k=5}
    \label{fig:autoprice_q10_k5_nht}
  \end{subfigure}
  \begin{subfigure}{0.245\textwidth}
    \centering
    \includegraphics[width=\linewidth]{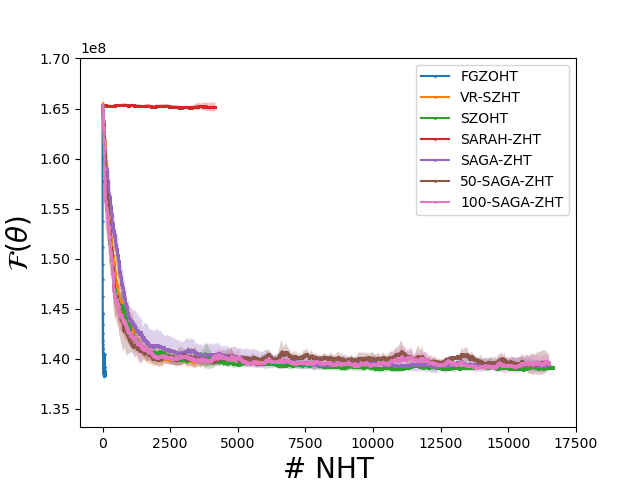}
    \caption{q=5, k=10}
    \label{fig:autoprice_q5_k10_nht}
  \end{subfigure}
   \begin{subfigure}{0.245\textwidth}
    \centering
    \includegraphics[width=\linewidth]{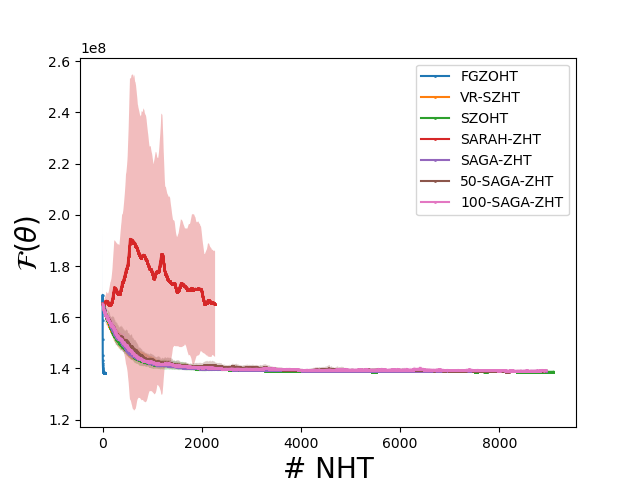}
    \caption{q=10, k=10}
    \label{fig:autoprice_q10_k10_nht}
  \end{subfigure}

  \caption{\#IZO (up) and \#NHT (down) on the ridge regression task, \texttt{autoprice} dataset.}
  \label{fig:autoprice_curves}

\end{figure}

\subsection{{Additional results for universal adversarial attacks}}\label{sec:advatextra}
{
In this section , we provide additional results for the universal adversarial attacks setting from our Experiments Section, for the 3 additional classes: 'ship', 'bird', and 'dog'. As we can observe in Figures \ref{fig:advat_cifar_bird} , \ref{fig:advat_cifar_ship}, and \ref{fig:advat_cifar_dog}  below respectively, in most of such cases, there is a variance-reduced algorithm which can achieve better performance than the vanilla zeroth-order hard-thresholding algorithms, (for instance, SARAH-ZHT in Figure \ref{fig:advat_cifar_ship},  and SAGA-ZHT in Figure \ref{fig:advat_cifar_bird}) which demonstrates the applicability of such algorithms. Correspondingly, this can also be verified by observing the misclassification success in Table \ref{table:CIFAR_bird}, \ref{table:CIFAR_ship} and   \ref{table:CIFAR_dog}: even if a smaller value for the cost does not necessarily imply a strictly higher attack success rate, still, overall, more successful universal attacks also have a higher success rate of attack.
}


\begin{figure}[H]
  \centering
  \includegraphics[scale=0.45]{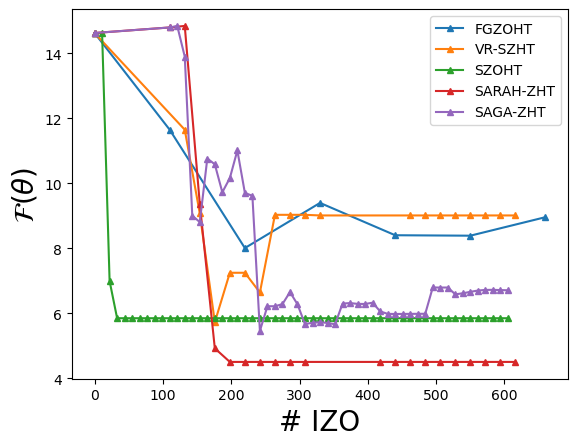}
   \includegraphics[scale=0.45]{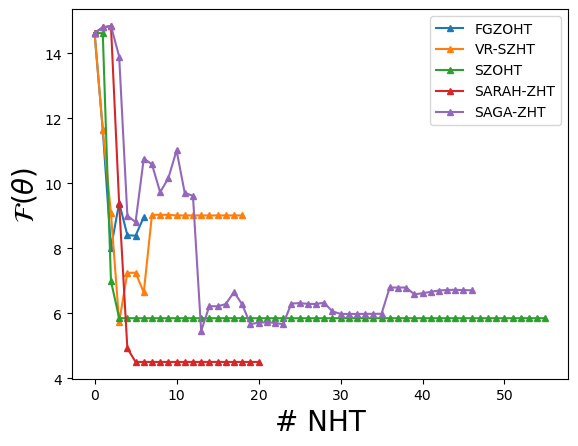}
     \caption{{\#IZO and \#NHT on the few pixels adversarial attacks task (CIFAR-10), for the original class 'ship'.}}\label{fig:advat_cifar_ship}
     \end{figure}

\begin{table}[H]
\caption{{Comparison of universal adversarial attacks on $n=10$ images from the CIFAR-10 test-set, from the 'ship' class. For each algorithm, the leftmost image is the sparse adversarial perturbation applied to each image in the row. ('auto' stands for 'automobile', and 'plane' for 'airplane') \vspace{0.3cm}}} \label{table:CIFAR_ship}
  \centering
  \small
  \begin{tabular}
      {ccccccccccc}
      \hline
      	Image ID & 1 & 15 & 18 & 51 & 54 & 55 & 72 & 73 & 79 & 80  \\
      \hline &&&&&&&&&& \vspace{-0.3cm} \\
      	Original &
\parbox[c]{1.5em}{\includegraphics[width=0.30in]{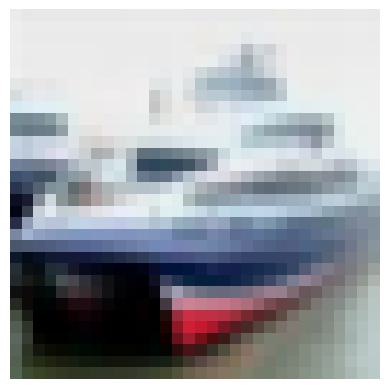}} &
\parbox[c]{1.5em}{\includegraphics[width=0.30in]{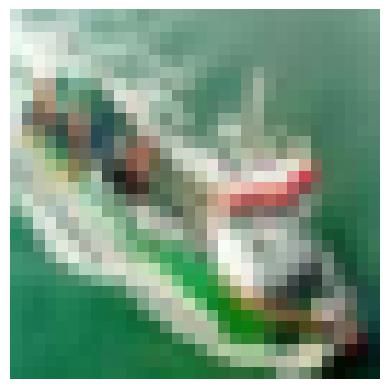}} &
\parbox[c]{1.5em}{\includegraphics[width=0.30in]{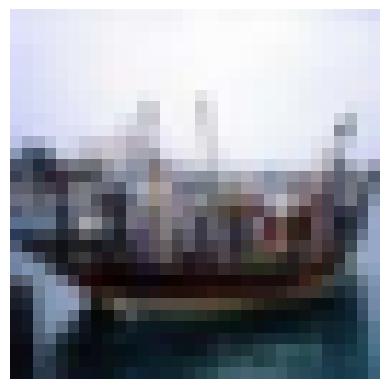}} &
\parbox[c]{1.5em}{\includegraphics[width=0.30in]{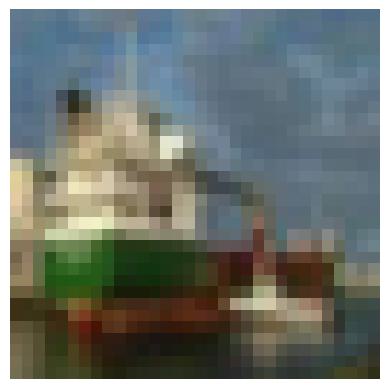}} &
\parbox[c]{1.5em}{\includegraphics[width=0.30in]{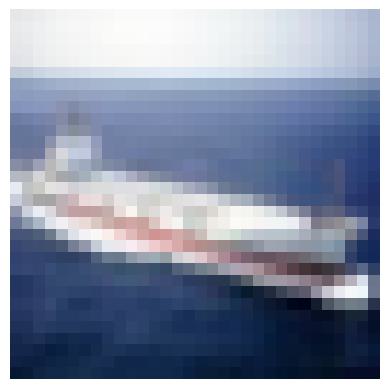}} &
\parbox[c]{1.5em}{\includegraphics[width=0.30in]{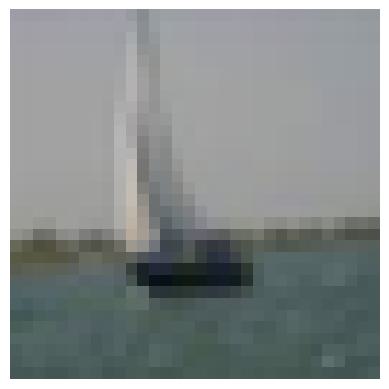}} &
\parbox[c]{1.5em}{\includegraphics[width=0.30in]{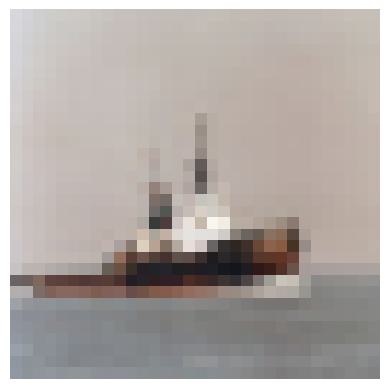}} &
\parbox[c]{1.5em}{\includegraphics[width=0.30in]{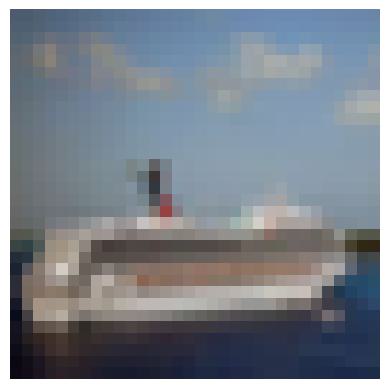}} &
\parbox[c]{1.5em}{\includegraphics[width=0.30in]{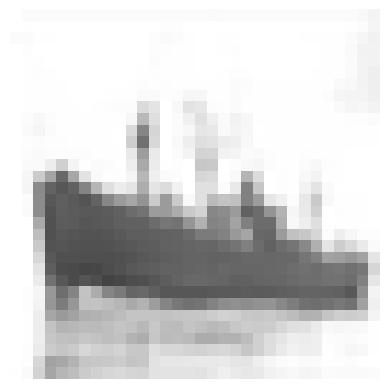}} &
\parbox[c]{1.5em}{\includegraphics[width=0.30in]{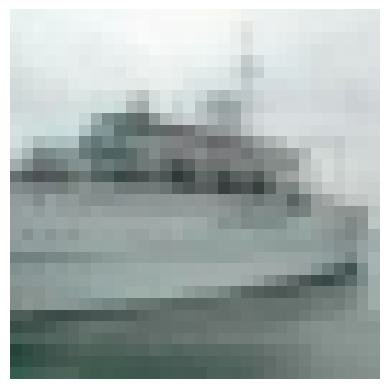}} \\
      \vspace{-0.2cm} \\
    \hline &&&&&&&&&& \vspace{-0.2cm} \\
      	FGZOHT &&&&&&&&&& \vspace{-0.2cm} \\
      &&&&&&&&&&\vspace{-0.2cm} \\
\parbox[c]{1.5em}{\includegraphics[width=0.30in]{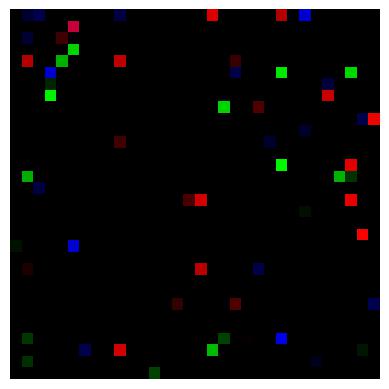}}  &
\parbox[c]{1.5em}{\includegraphics[width=0.30in]{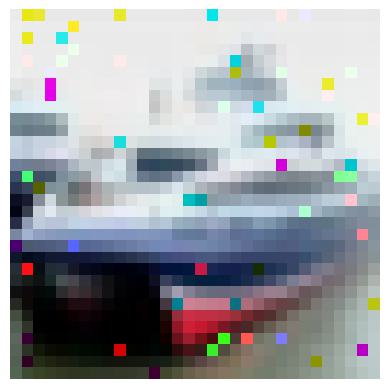}} &
\parbox[c]{1.5em}{\includegraphics[width=0.30in]{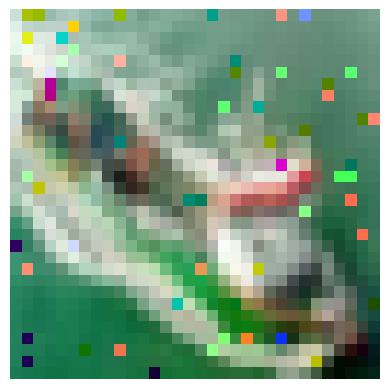}} &
\parbox[c]{1.5em}{\includegraphics[width=0.30in]{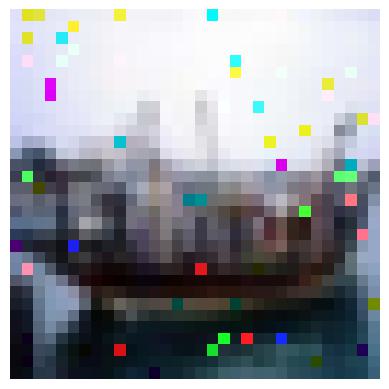}} &
\parbox[c]{1.5em}{\includegraphics[width=0.30in]{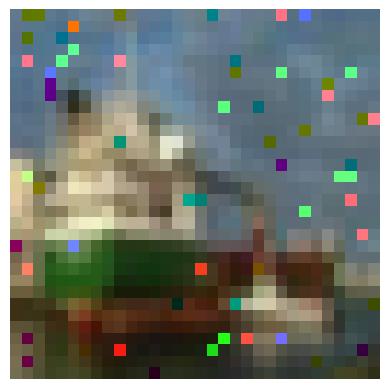}} &
\parbox[c]{1.5em}{\includegraphics[width=0.30in]{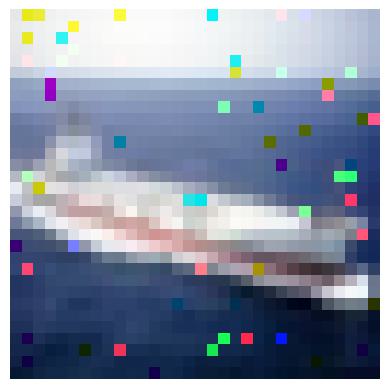}} &
\parbox[c]{1.5em}{\includegraphics[width=0.30in]{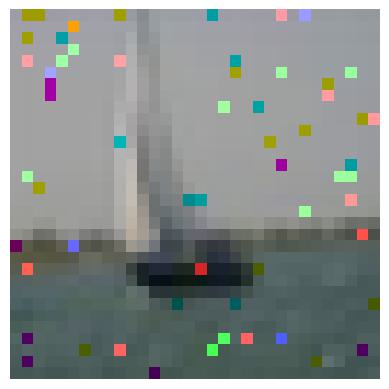}} &
\parbox[c]{1.5em}{\includegraphics[width=0.30in]{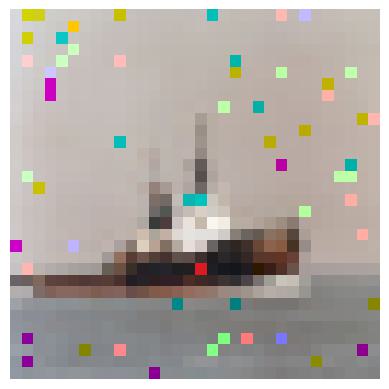}} &
\parbox[c]{1.5em}{\includegraphics[width=0.30in]{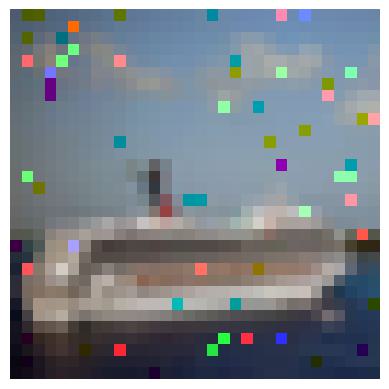}} &
\parbox[c]{1.5em}{\includegraphics[width=0.30in]{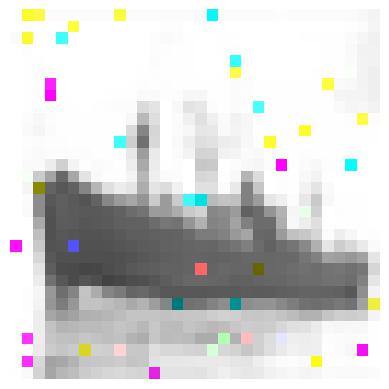}} &
\parbox[c]{1.5em}{\includegraphics[width=0.30in]{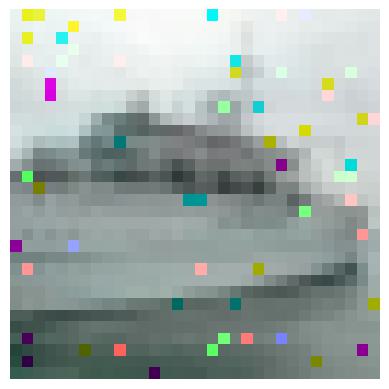}} \\
& ship & \textbf{frog} & ship & ship & ship & ship & ship & ship & ship & ship \\
      \hline
      
    \hline &&&&&&&&&& \vspace{-0.2cm} \\
      	SZOHT &&&&&&&&&& \vspace{-0.2cm} \\
      &&&&&&&&&&\vspace{-0.2cm} \\
\parbox[c]{1.5em}{\includegraphics[width=0.30in]{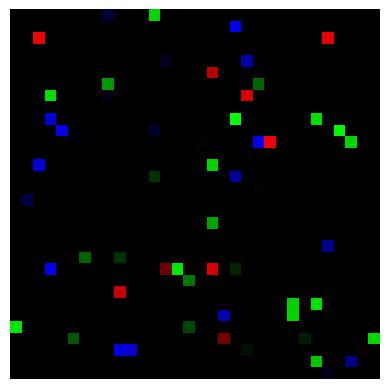}}  &
\parbox[c]{1.5em}{\includegraphics[width=0.30in]{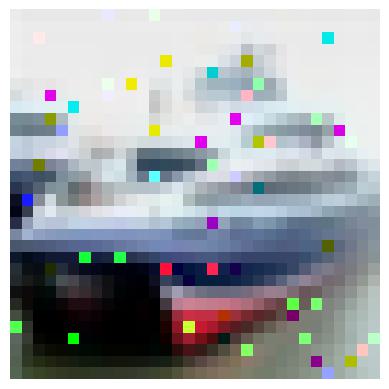}} &
\parbox[c]{1.5em}{\includegraphics[width=0.30in]{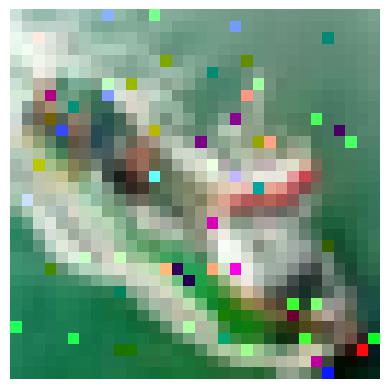}} &
\parbox[c]{1.5em}{\includegraphics[width=0.30in]{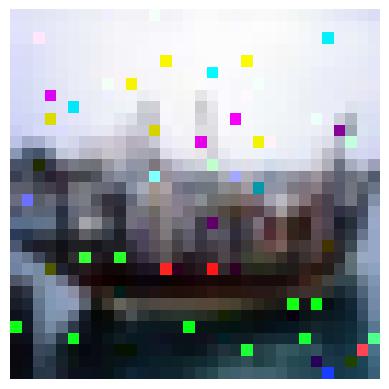}} &
\parbox[c]{1.5em}{\includegraphics[width=0.30in]{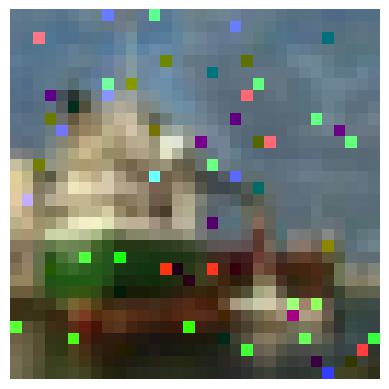}} &
\parbox[c]{1.5em}{\includegraphics[width=0.30in]{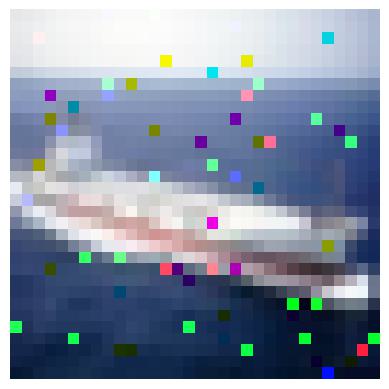}} &
\parbox[c]{1.5em}{\includegraphics[width=0.30in]{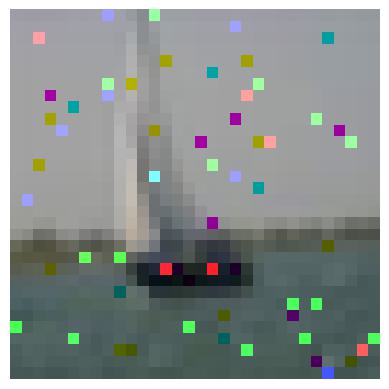}} &
\parbox[c]{1.5em}{\includegraphics[width=0.30in]{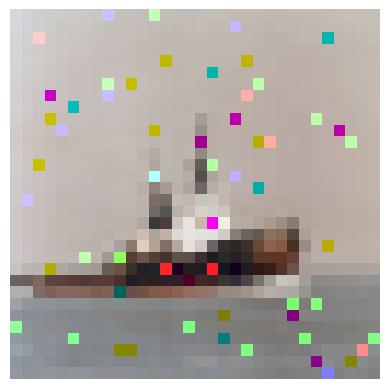}} &
\parbox[c]{1.5em}{\includegraphics[width=0.30in]{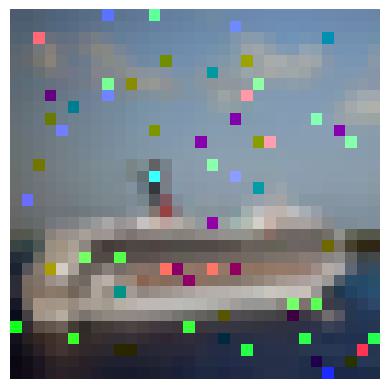}} &
\parbox[c]{1.5em}{\includegraphics[width=0.30in]{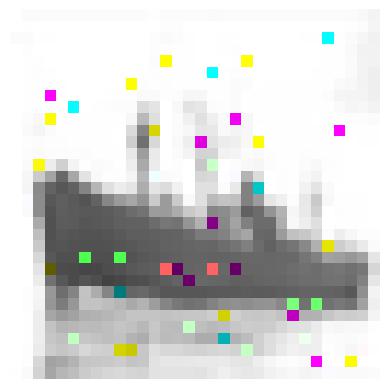}} &
\parbox[c]{1.5em}{\includegraphics[width=0.30in]{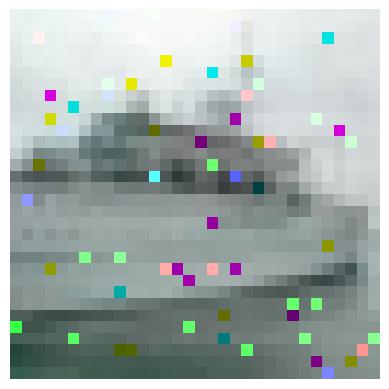}} \\
& ship & \textbf{frog} & ship & \textbf{deer} & ship & \textbf{deer} & ship & ship & ship & \textbf{plane} \\
         \hline

    \hline &&&&&&&&&& \vspace{-0.2cm} \\
      	VR-SZHT &&&&&&&&&& \vspace{-0.2cm} \\
      &&&&&&&&&&\vspace{-0.2cm} \\
\parbox[c]{1.5em}{\includegraphics[width=0.30in]{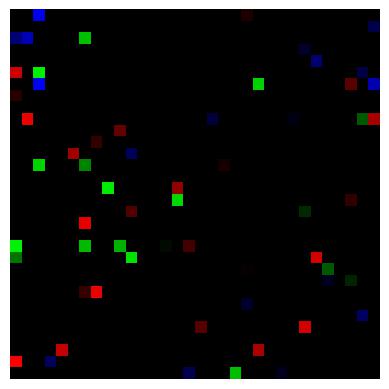}}  &
\parbox[c]{1.5em}{\includegraphics[width=0.30in]{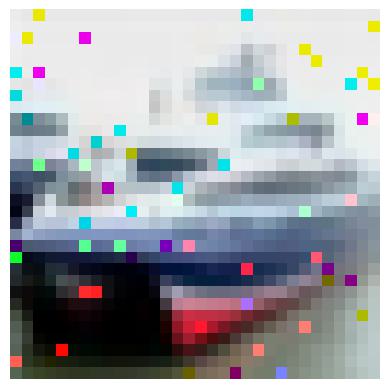}} &
\parbox[c]{1.5em}{\includegraphics[width=0.30in]{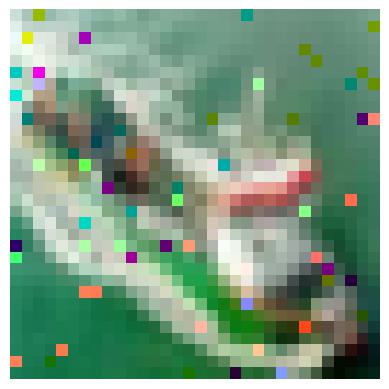}} &
\parbox[c]{1.5em}{\includegraphics[width=0.30in]{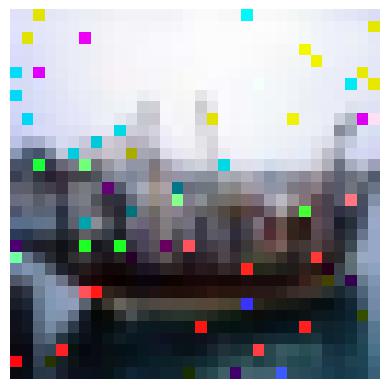}} &
\parbox[c]{1.5em}{\includegraphics[width=0.30in]{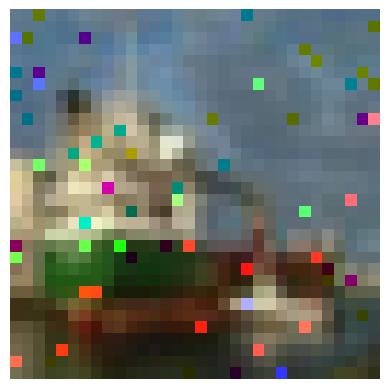}} &
\parbox[c]{1.5em}{\includegraphics[width=0.30in]{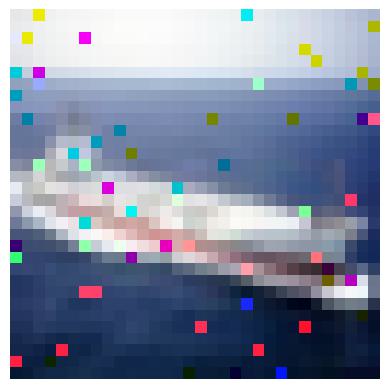}} &
\parbox[c]{1.5em}{\includegraphics[width=0.30in]{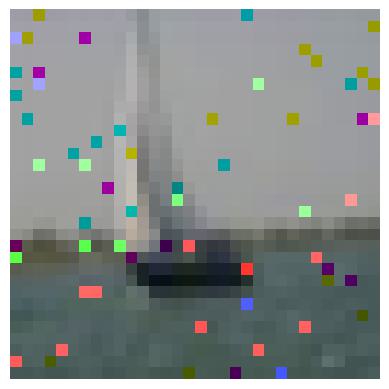}} &
\parbox[c]{1.5em}{\includegraphics[width=0.30in]{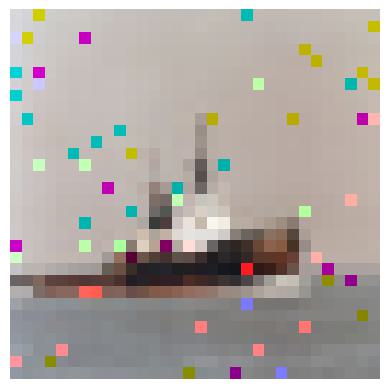}} &
\parbox[c]{1.5em}{\includegraphics[width=0.30in]{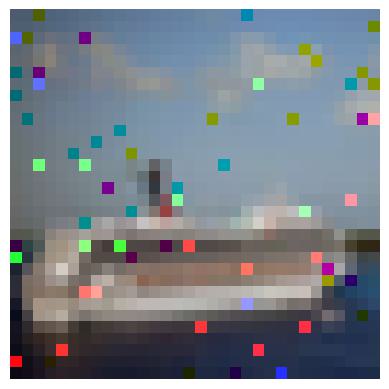}} &
\parbox[c]{1.5em}{\includegraphics[width=0.30in]{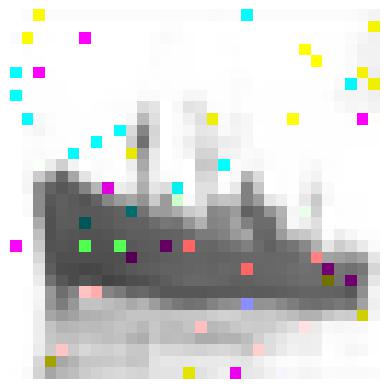}} &
\parbox[c]{1.5em}{\includegraphics[width=0.30in]{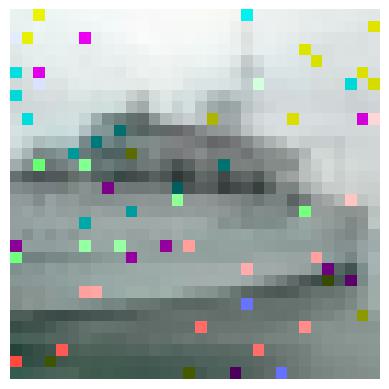}} \\
& ship & \textbf{frog} & \textbf{truck} & \textbf{plane} & ship & ship & ship & ship & ship & \textbf{auto} \\
         \hline

             \hline &&&&&&&&&& \vspace{-0.2cm} \\
      	SAGA-SZHT &&&&&&&&&& \vspace{-0.2cm} \\
      &&&&&&&&&&\vspace{-0.2cm} \\
\parbox[c]{1.5em}{\includegraphics[width=0.30in]{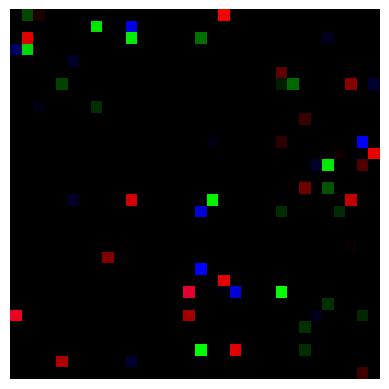}}  &
\parbox[c]{1.5em}{\includegraphics[width=0.30in]{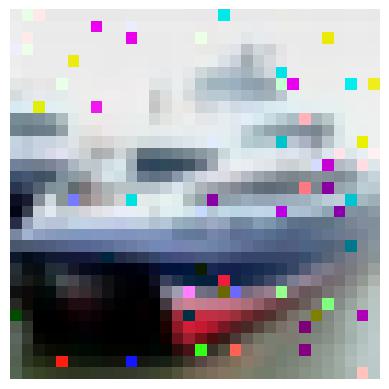}} &
\parbox[c]{1.5em}{\includegraphics[width=0.30in]{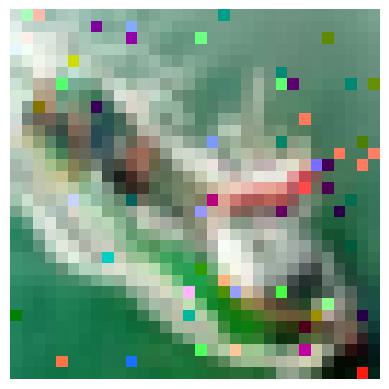}} &
\parbox[c]{1.5em}{\includegraphics[width=0.30in]{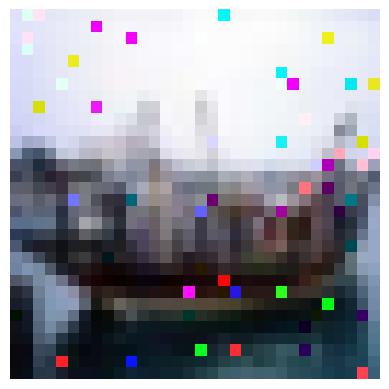}} &
\parbox[c]{1.5em}{\includegraphics[width=0.30in]{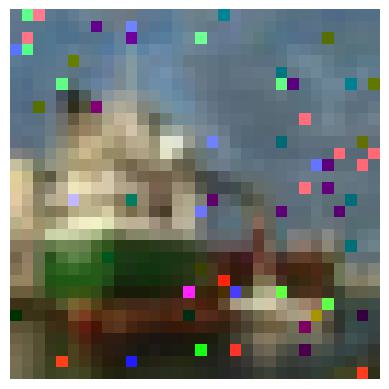}} &
\parbox[c]{1.5em}{\includegraphics[width=0.30in]{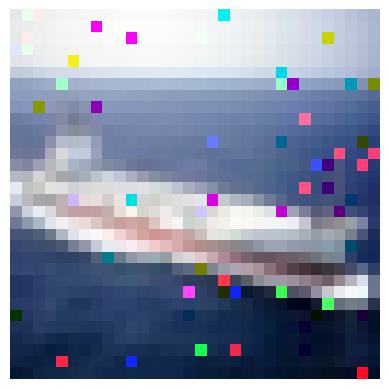}} &
\parbox[c]{1.5em}{\includegraphics[width=0.30in]{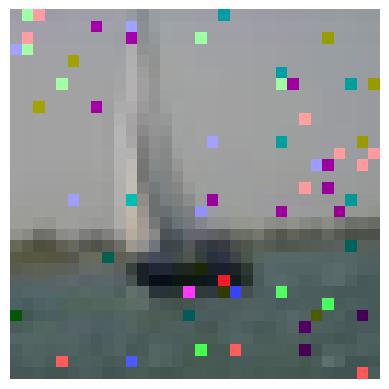}} &
\parbox[c]{1.5em}{\includegraphics[width=0.30in]{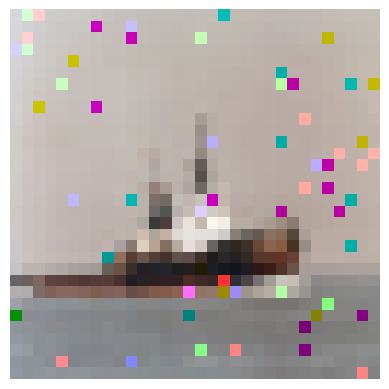}} &
\parbox[c]{1.5em}{\includegraphics[width=0.30in]{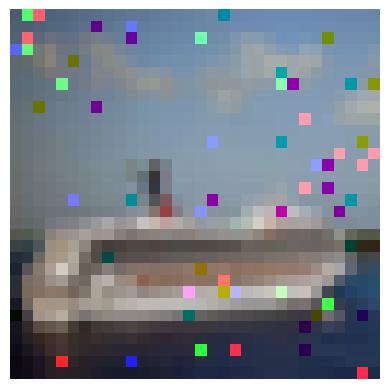}} &
\parbox[c]{1.5em}{\includegraphics[width=0.30in]{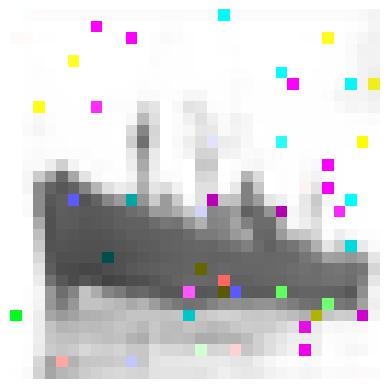}} &
\parbox[c]{1.5em}{\includegraphics[width=0.30in]{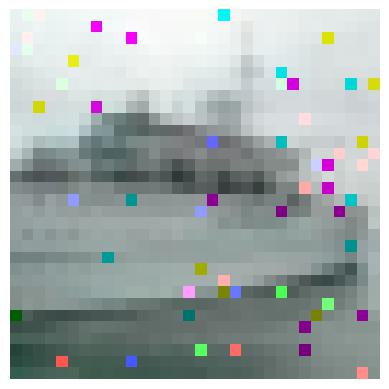}} \\
& ship & \textbf{frog} & \textbf{truck} & ship & \textbf{plane} & \textbf{deer} & ship & ship & ship & ship \\
         \hline

             \hline &&&&&&&&&& \vspace{-0.2cm} \\
      	SARAH-SZHT &&&&&&&&&& \vspace{-0.2cm} \\
      &&&&&&&&&&\vspace{-0.2cm} \\
\parbox[c]{1.5em}{\includegraphics[width=0.30in]{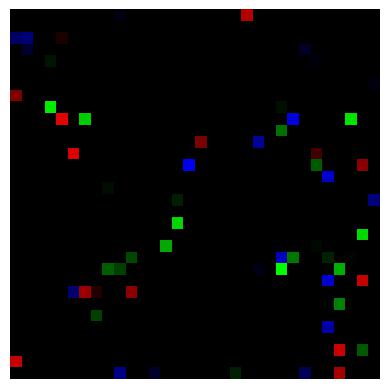}}  &
\parbox[c]{1.5em}{\includegraphics[width=0.30in]{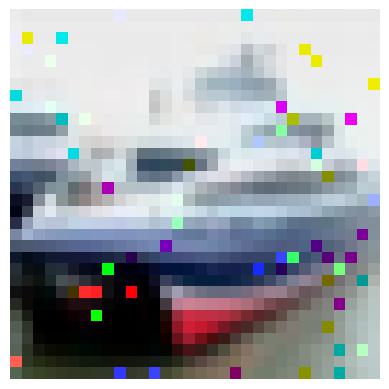}} &
\parbox[c]{1.5em}{\includegraphics[width=0.30in]{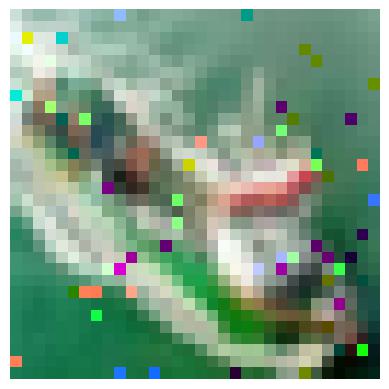}} &
\parbox[c]{1.5em}{\includegraphics[width=0.30in]{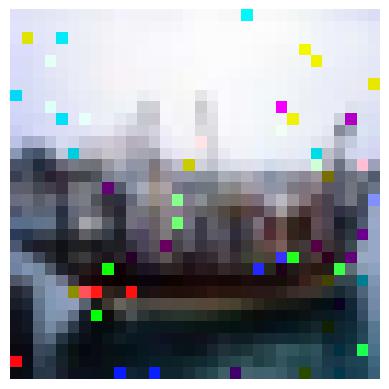}} &
\parbox[c]{1.5em}{\includegraphics[width=0.30in]{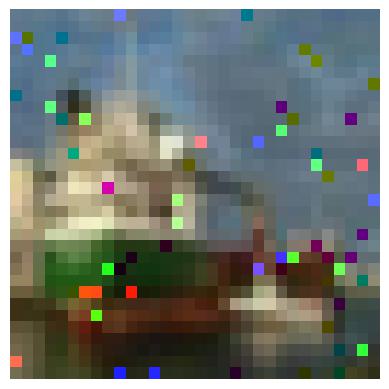}} &
\parbox[c]{1.5em}{\includegraphics[width=0.30in]{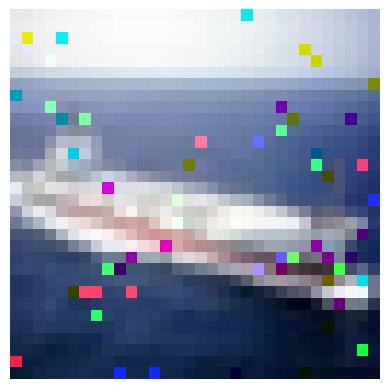}} &
\parbox[c]{1.5em}{\includegraphics[width=0.30in]{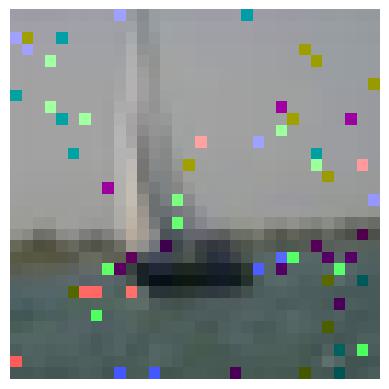}} &
\parbox[c]{1.5em}{\includegraphics[width=0.30in]{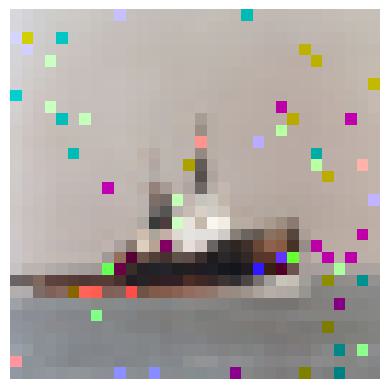}} &
\parbox[c]{1.5em}{\includegraphics[width=0.30in]{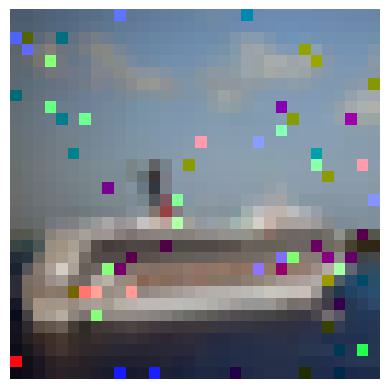}} &
\parbox[c]{1.5em}{\includegraphics[width=0.30in]{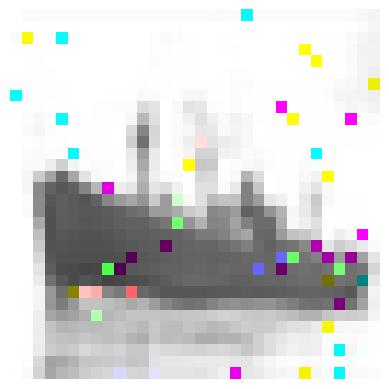}} &
\parbox[c]{1.5em}{\includegraphics[width=0.30in]{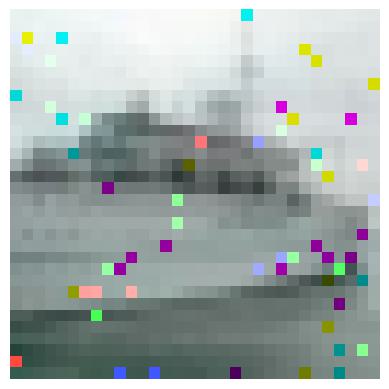}} \\
& ship & \textbf{frog} & \textbf{auto} & \textbf{auto} & ship & ship & ship & ship & ship & ship \\
         \hline
      
  \end{tabular}
\end{table}

\begin{figure}[H]
  \centering
  \includegraphics[scale=0.45]{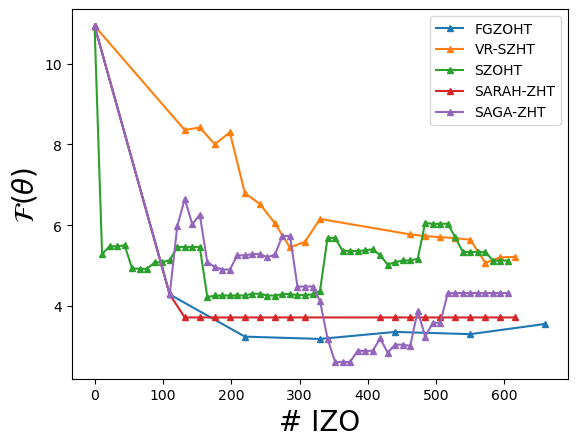}
   \includegraphics[scale=0.45]{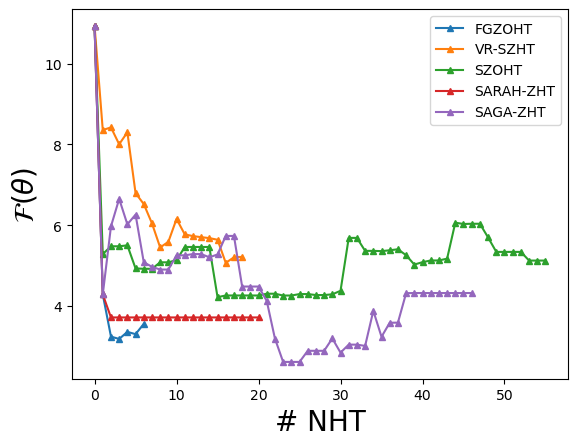}
     \caption{{\#IZO and \#NHT on the few pixels adversarial attacks task (CIFAR-10), for the original class 'bird'.}}\label{fig:advat_cifar_bird}
     \end{figure}

\begin{table}[H]
\caption{{Comparison of universal adversarial attacks on $n=10$ images from the CIFAR-10 test-set, from the 'bird' class. For each algorithm, the leftmost image is the sparse adversarial perturbation applied to each image in the row.  \vspace{0.3cm}}} \label{table:CIFAR_bird}
  \centering
  \begin{tabular}
      {ccccccccccc}
      \hline
      	Image ID & 65 & 67 & 70 & 75 & 86 & 113 & 123 & 129 & 138 & 149  \\
      \hline &&&&&&&&&& \vspace{-0.3cm} \\
      	Original &
\parbox[c]{1.5em}{\includegraphics[width=0.30in]{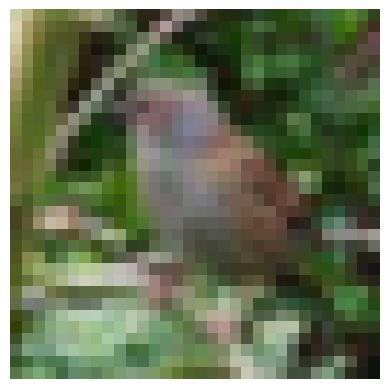}} &
\parbox[c]{1.5em}{\includegraphics[width=0.30in]{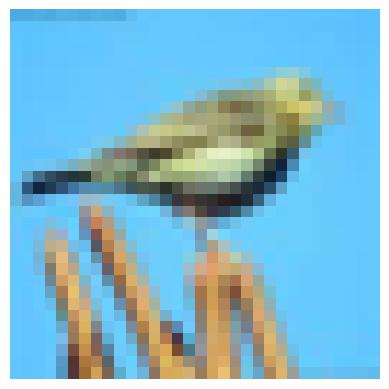}} &
\parbox[c]{1.5em}{\includegraphics[width=0.30in]{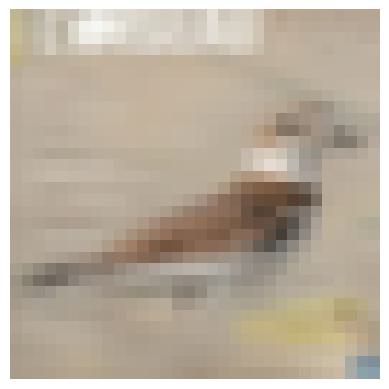}} &
\parbox[c]{1.5em}{\includegraphics[width=0.30in]{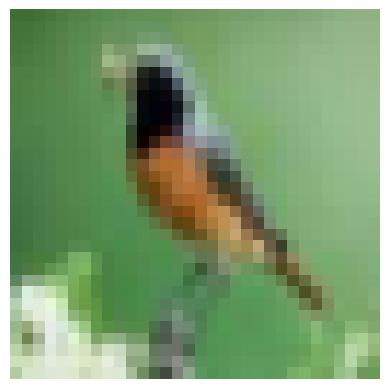}} &
\parbox[c]{1.5em}{\includegraphics[width=0.30in]{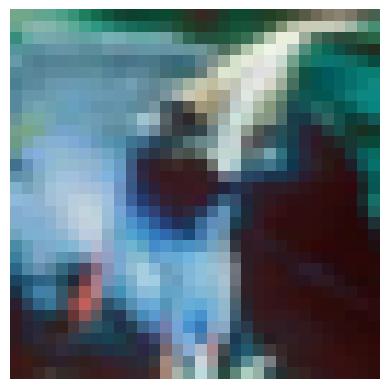}} &
\parbox[c]{1.5em}{\includegraphics[width=0.30in]{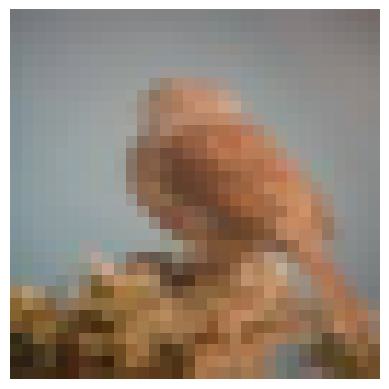}} &
\parbox[c]{1.5em}{\includegraphics[width=0.30in]{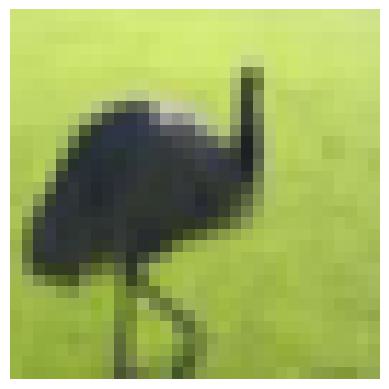}} &
\parbox[c]{1.5em}{\includegraphics[width=0.30in]{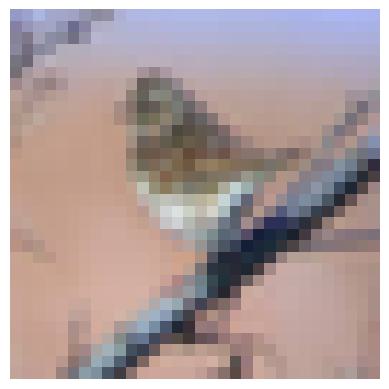}} &
\parbox[c]{1.5em}{\includegraphics[width=0.30in]{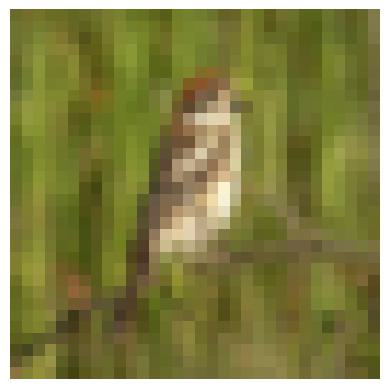}} &
\parbox[c]{1.5em}{\includegraphics[width=0.30in]{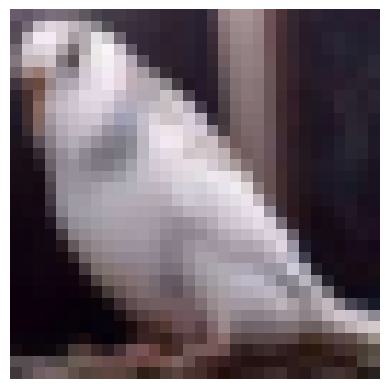}} \\
      \vspace{-0.2cm} \\
    \hline &&&&&&&&&& \vspace{-0.2cm} \\
      	FGZOHT &&&&&&&&&& \vspace{-0.2cm} \\
      &&&&&&&&&&\vspace{-0.2cm} \\
\parbox[c]{1.5em}{\includegraphics[width=0.30in]{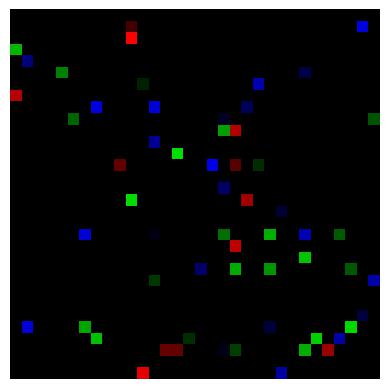}}  &
\parbox[c]{1.5em}{\includegraphics[width=0.30in]{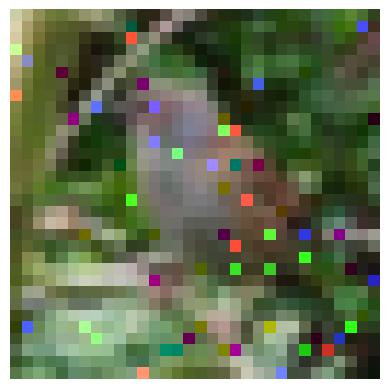}} &
\parbox[c]{1.5em}{\includegraphics[width=0.30in]{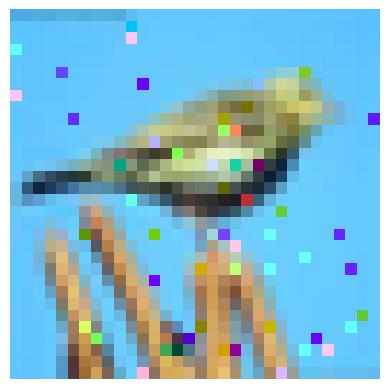}} &
\parbox[c]{1.5em}{\includegraphics[width=0.30in]{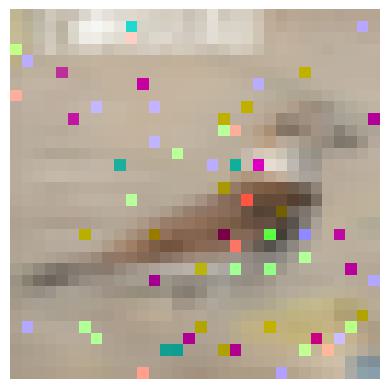}} &
\parbox[c]{1.5em}{\includegraphics[width=0.30in]{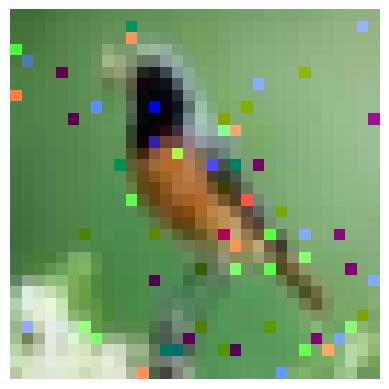}} &
\parbox[c]{1.5em}{\includegraphics[width=0.30in]{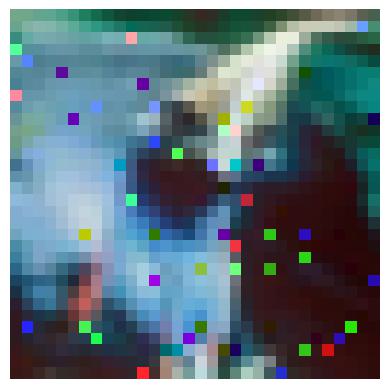}} &
\parbox[c]{1.5em}{\includegraphics[width=0.30in]{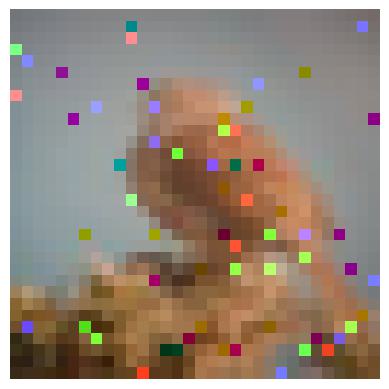}} &
\parbox[c]{1.5em}{\includegraphics[width=0.30in]{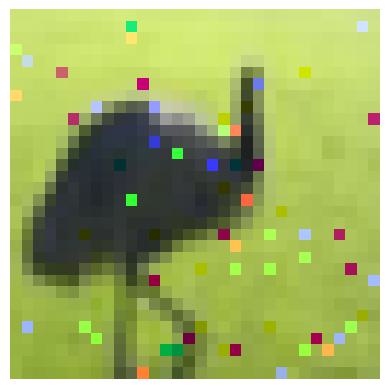}} &
\parbox[c]{1.5em}{\includegraphics[width=0.30in]{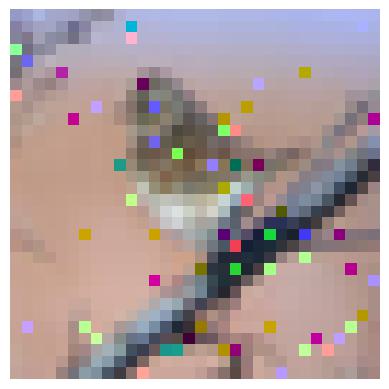}} &
\parbox[c]{1.5em}{\includegraphics[width=0.30in]{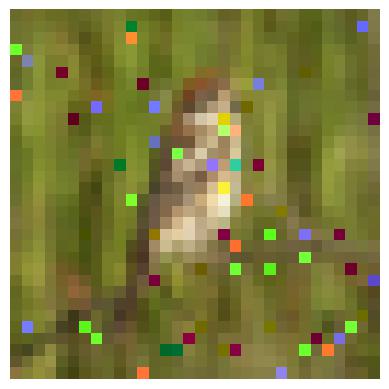}} &
\parbox[c]{1.5em}{\includegraphics[width=0.30in]{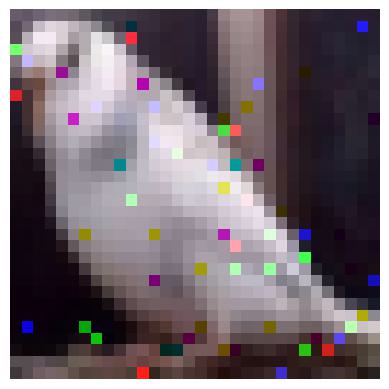}} \\
& \textbf{frog} & \textbf{bird} & \textbf{deer} & \textbf{dog} & \textbf{deer} & bird & bird & \textbf{deer} & \textbf{frog} & bird \\
      \hline
      
    \hline &&&&&&&&&& \vspace{-0.2cm} \\
      	SZOHT &&&&&&&&&& \vspace{-0.2cm} \\
      &&&&&&&&&&\vspace{-0.2cm} \\
\parbox[c]{1.5em}{\includegraphics[width=0.30in]{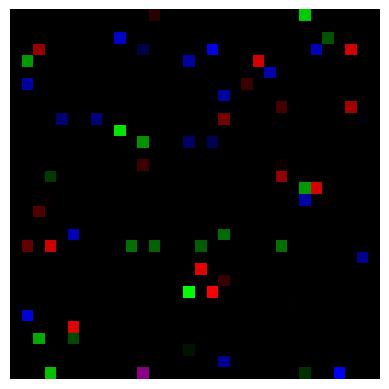}}  &
\parbox[c]{1.5em}{\includegraphics[width=0.30in]{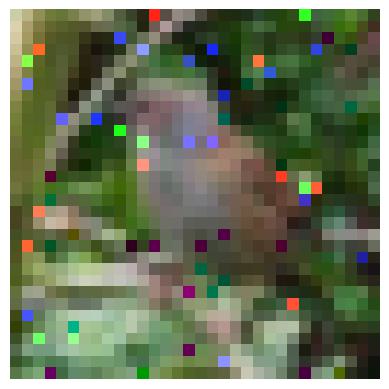}} &
\parbox[c]{1.5em}{\includegraphics[width=0.30in]{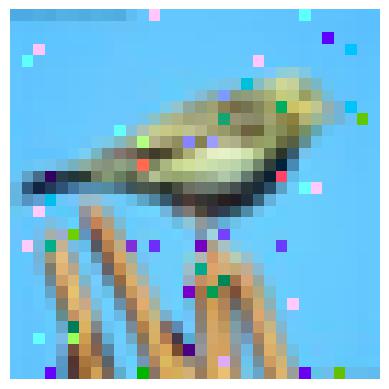}} &
\parbox[c]{1.5em}{\includegraphics[width=0.30in]{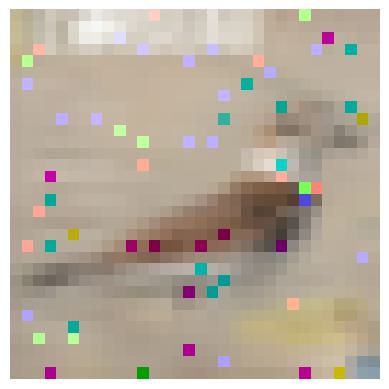}} &
\parbox[c]{1.5em}{\includegraphics[width=0.30in]{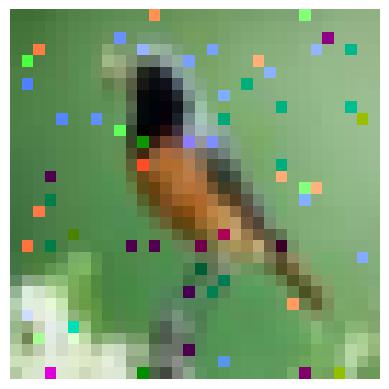}} &
\parbox[c]{1.5em}{\includegraphics[width=0.30in]{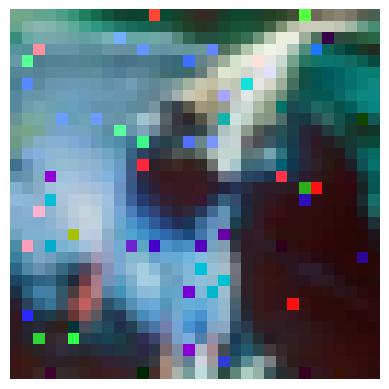}} &
\parbox[c]{1.5em}{\includegraphics[width=0.30in]{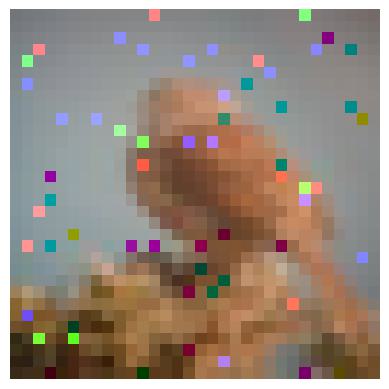}} &
\parbox[c]{1.5em}{\includegraphics[width=0.30in]{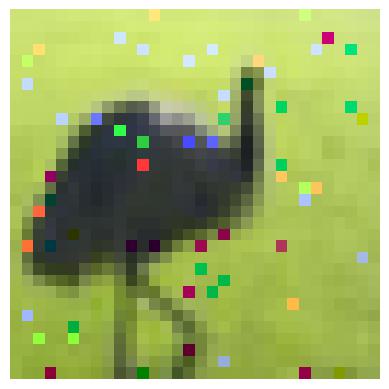}} &
\parbox[c]{1.5em}{\includegraphics[width=0.30in]{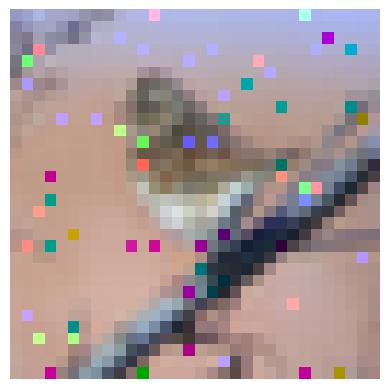}} &
\parbox[c]{1.5em}{\includegraphics[width=0.30in]{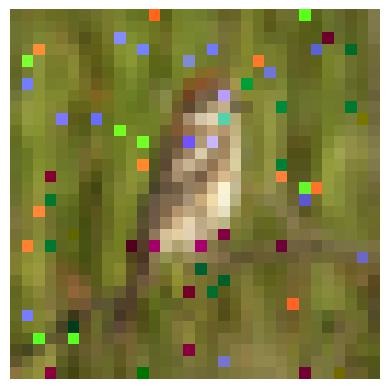}} &
\parbox[c]{1.5em}{\includegraphics[width=0.30in]{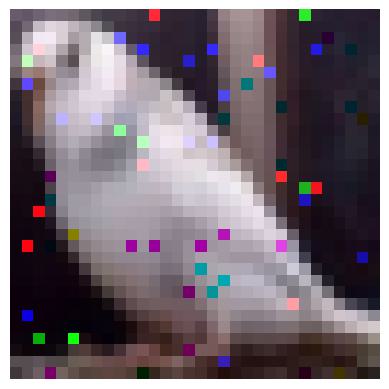}} \\
& \textbf{frog} & bird & bird & bird & bird & bird & bird & \textbf{ship} & bird & bird \\
         \hline

    \hline &&&&&&&&&& \vspace{-0.2cm} \\
      	VR-SZHT &&&&&&&&&& \vspace{-0.2cm} \\
      &&&&&&&&&&\vspace{-0.2cm} \\
\parbox[c]{1.5em}{\includegraphics[width=0.30in]{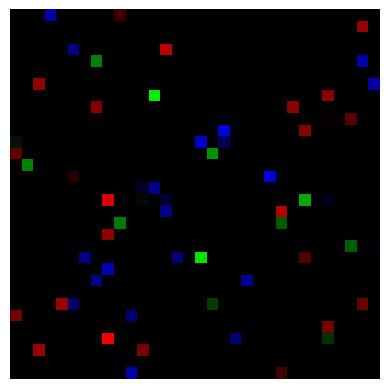}}  &
\parbox[c]{1.5em}{\includegraphics[width=0.30in]{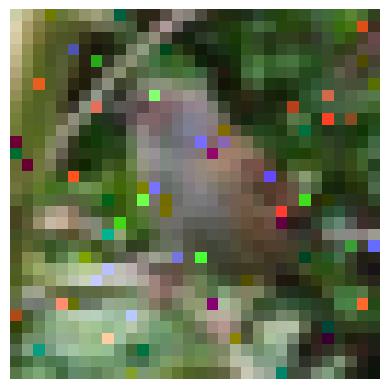}} &
\parbox[c]{1.5em}{\includegraphics[width=0.30in]{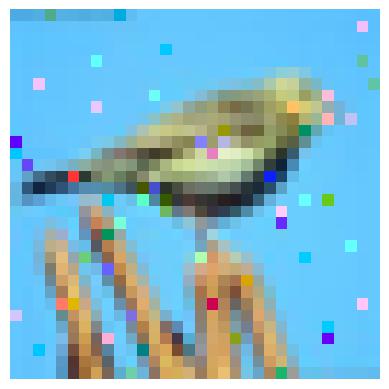}} &
\parbox[c]{1.5em}{\includegraphics[width=0.30in]{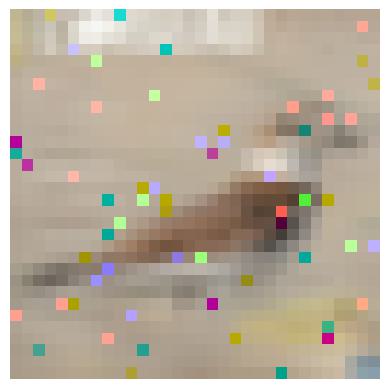}} &
\parbox[c]{1.5em}{\includegraphics[width=0.30in]{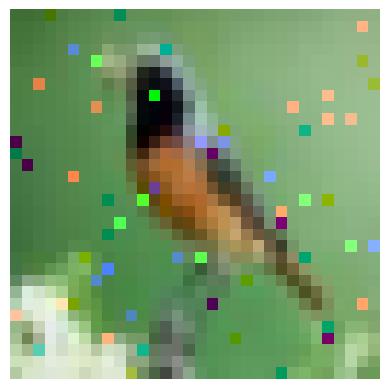}} &
\parbox[c]{1.5em}{\includegraphics[width=0.30in]{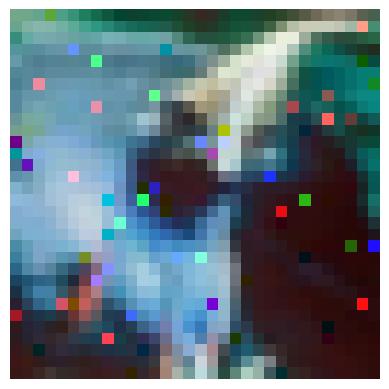}} &
\parbox[c]{1.5em}{\includegraphics[width=0.30in]{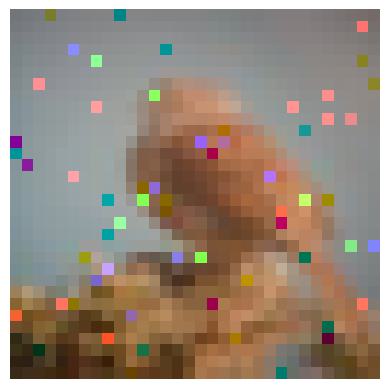}} &
\parbox[c]{1.5em}{\includegraphics[width=0.30in]{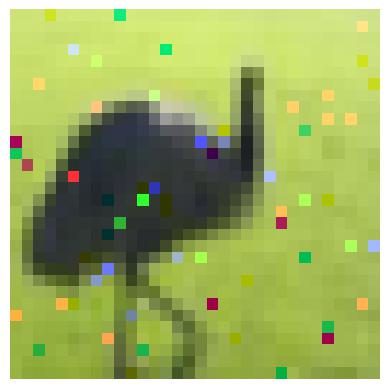}} &
\parbox[c]{1.5em}{\includegraphics[width=0.30in]{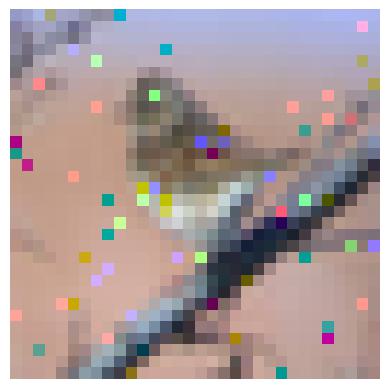}} &
\parbox[c]{1.5em}{\includegraphics[width=0.30in]{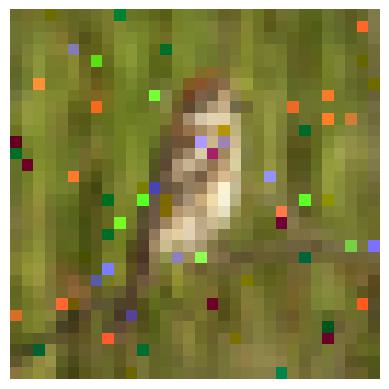}} &
\parbox[c]{1.5em}{\includegraphics[width=0.30in]{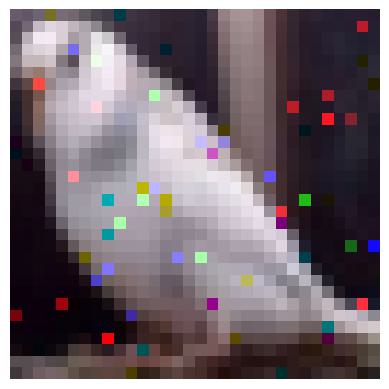}} \\
& \textbf{frog} & bird & bird & bird & bird & \textbf{frog} & bird & \textbf{ship} & bird & bird \\
         \hline

             \hline &&&&&&&&&& \vspace{-0.2cm} \\
      	SAGA-SZHT &&&&&&&&&& \vspace{-0.2cm} \\
      &&&&&&&&&&\vspace{-0.2cm} \\
\parbox[c]{1.5em}{\includegraphics[width=0.30in]{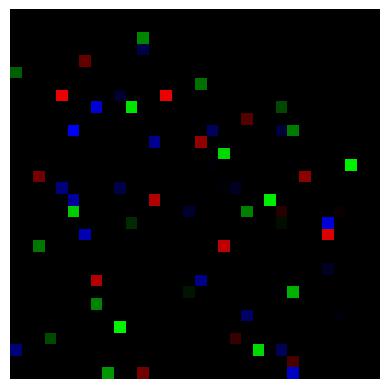}}  &
\parbox[c]{1.5em}{\includegraphics[width=0.30in]{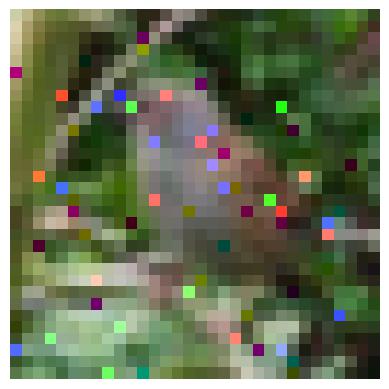}} &
\parbox[c]{1.5em}{\includegraphics[width=0.30in]{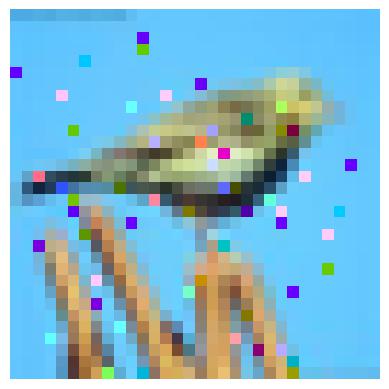}} &
\parbox[c]{1.5em}{\includegraphics[width=0.30in]{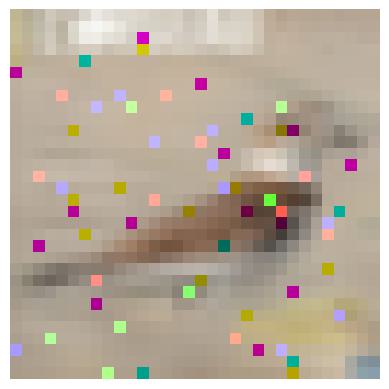}} &
\parbox[c]{1.5em}{\includegraphics[width=0.30in]{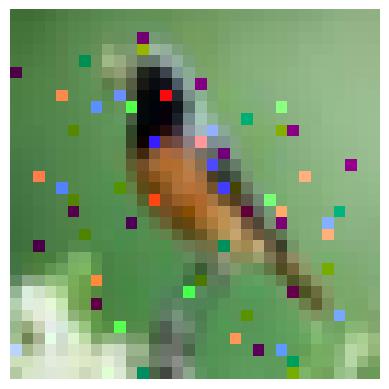}} &
\parbox[c]{1.5em}{\includegraphics[width=0.30in]{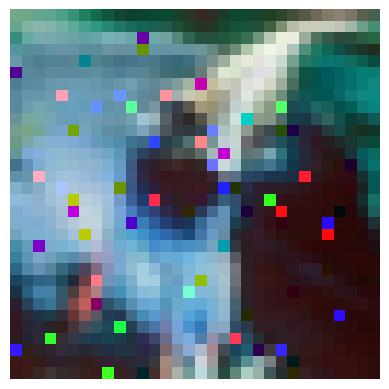}} &
\parbox[c]{1.5em}{\includegraphics[width=0.30in]{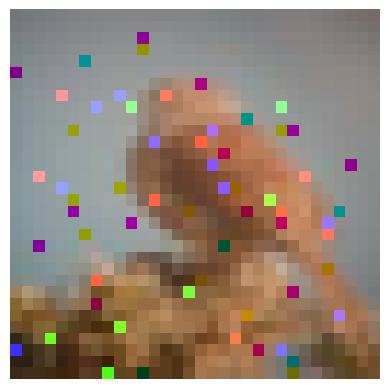}} &
\parbox[c]{1.5em}{\includegraphics[width=0.30in]{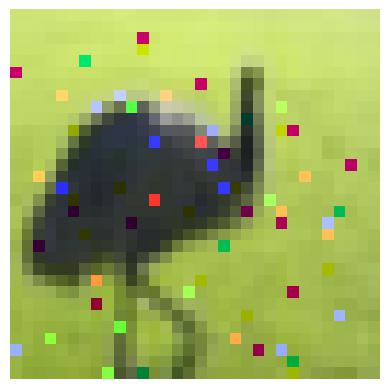}} &
\parbox[c]{1.5em}{\includegraphics[width=0.30in]{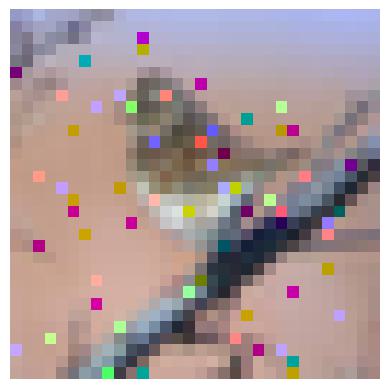}} &
\parbox[c]{1.5em}{\includegraphics[width=0.30in]{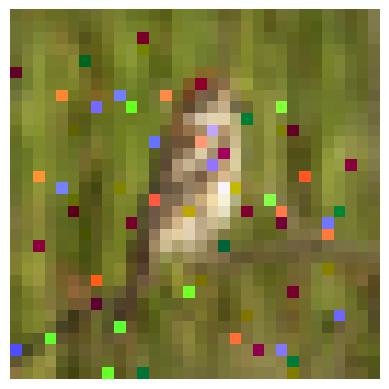}} &
\parbox[c]{1.5em}{\includegraphics[width=0.30in]{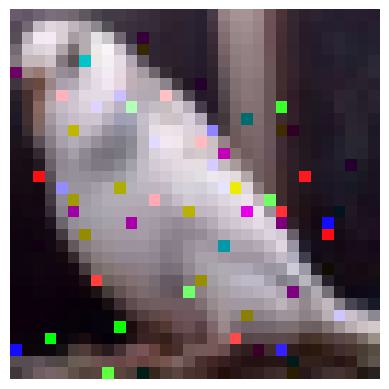}} \\
& \textbf{frog} & bird & \textbf{deer} & \textbf{dog} & bird & bird & \textbf{cat} & \textbf{deer} & \textbf{frog} & bird \\
         \hline

             \hline &&&&&&&&&& \vspace{-0.2cm} \\
      	SARAH-SZHT &&&&&&&&&& \vspace{-0.2cm} \\
      &&&&&&&&&&\vspace{-0.2cm} \\
\parbox[c]{1.5em}{\includegraphics[width=0.30in]{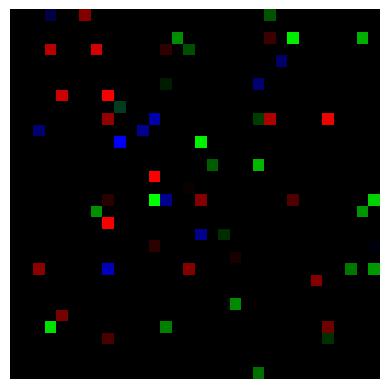}}  &
\parbox[c]{1.5em}{\includegraphics[width=0.30in]{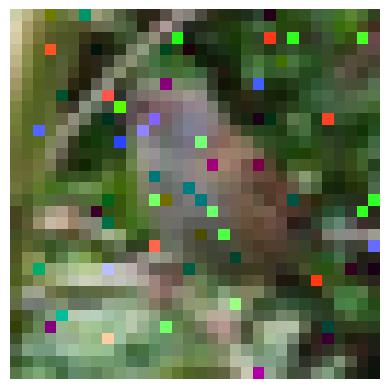}} &
\parbox[c]{1.5em}{\includegraphics[width=0.30in]{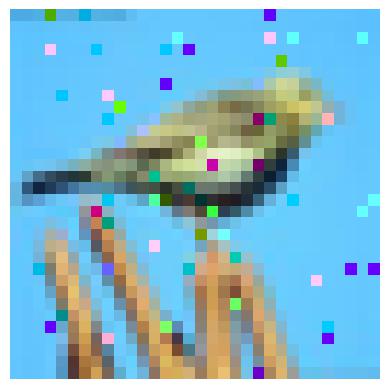}} &
\parbox[c]{1.5em}{\includegraphics[width=0.30in]{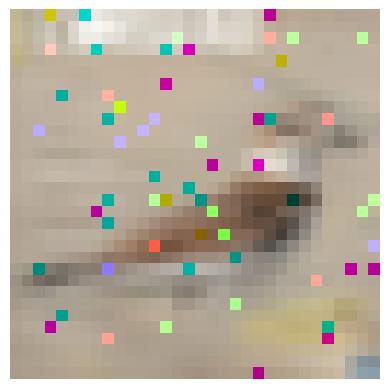}} &
\parbox[c]{1.5em}{\includegraphics[width=0.30in]{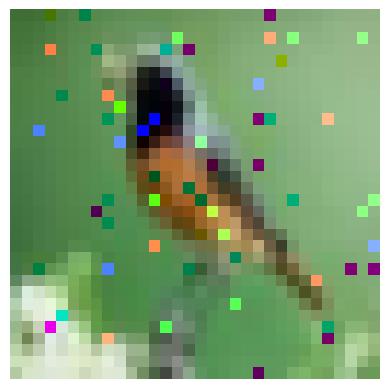}} &
\parbox[c]{1.5em}{\includegraphics[width=0.30in]{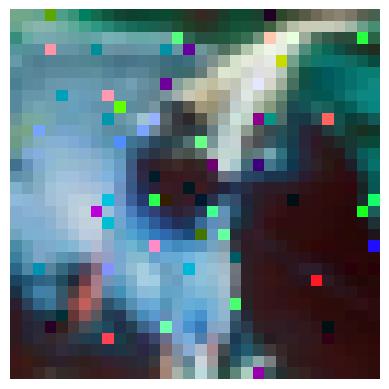}} &
\parbox[c]{1.5em}{\includegraphics[width=0.30in]{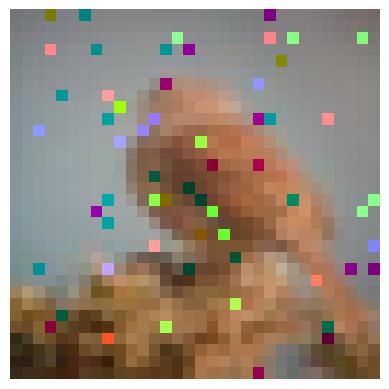}} &
\parbox[c]{1.5em}{\includegraphics[width=0.30in]{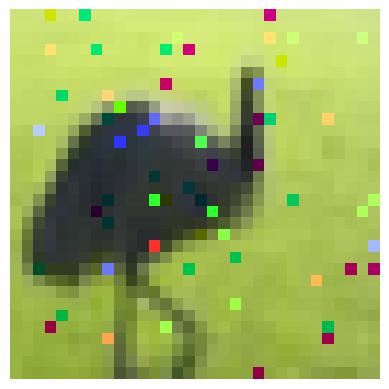}} &
\parbox[c]{1.5em}{\includegraphics[width=0.30in]{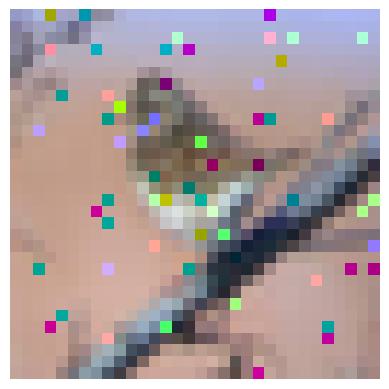}} &
\parbox[c]{1.5em}{\includegraphics[width=0.30in]{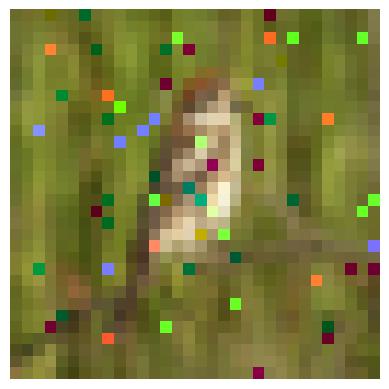}} &
\parbox[c]{1.5em}{\includegraphics[width=0.30in]{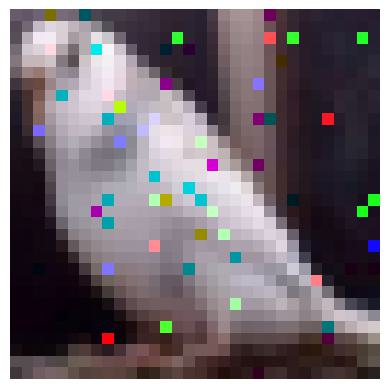}} \\
& \textbf{frog} & bird & \textbf{frog} & bird & bird & \textbf{frog} & bird & bird & \textbf{frog} & bird \\
         \hline
      
  \end{tabular}
\end{table}

\begin{figure}[H]
  \centering
  \includegraphics[scale=0.45]{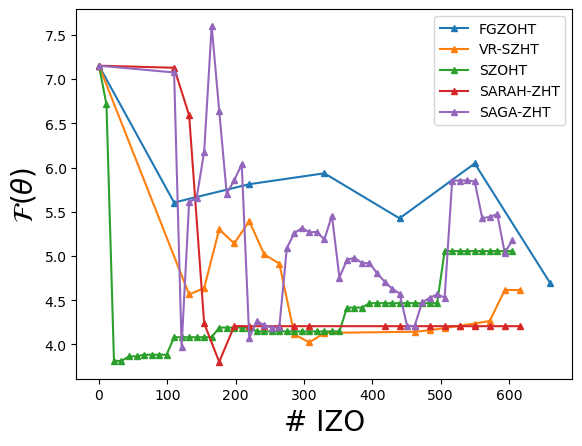}
   \includegraphics[scale=0.45]{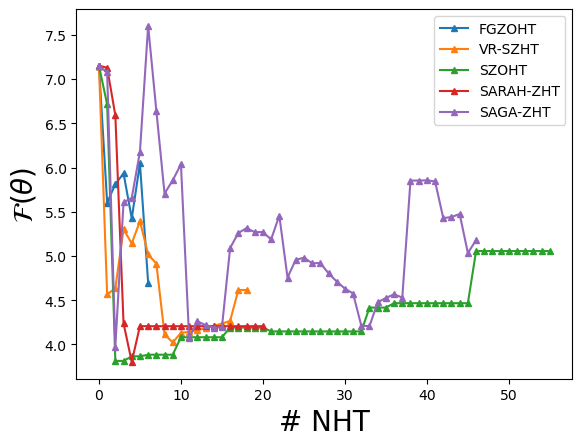}
     \caption{{\#IZO and \#NHT on the few pixels adversarial attacks task (CIFAR-10), for the original class 'dog'.}}\label{fig:advat_cifar_dog}
     \end{figure}

\begin{table}[H]
\caption{{Comparison of universal adversarial attacks on $n=10$ images from the CIFAR-10 test-set, from the 'dog' class. For each algorithm, the leftmost image is the sparse adversarial perturbation applied to each image in the row.  \vspace{0.3cm}}} \label{table:CIFAR_dog}
  \centering
  \begin{tabular}
      {ccccccccccc}
      \hline
      	Image ID & 12 & 16 & 31 & 33 & 39 & 42 & 101 & 128 & 141 & 148 \\
      \hline &&&&&&&&&& \vspace{-0.3cm} \\
      	Original &
        \parbox[c]{1.5em}{\includegraphics[width=0.30in]{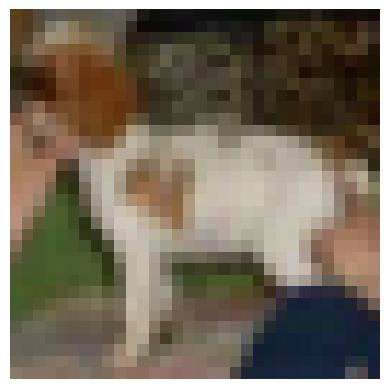}} &
        \parbox[c]{1.5em}{\includegraphics[width=0.30in]{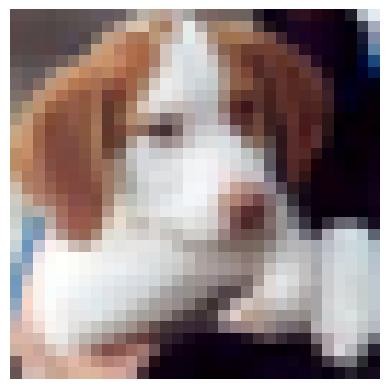}} &
        \parbox[c]{1.5em}{\includegraphics[width=0.30in]{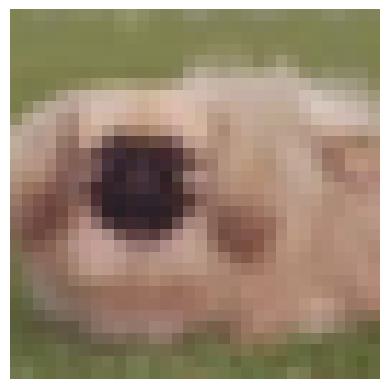}} &
        \parbox[c]{1.5em}{\includegraphics[width=0.30in]{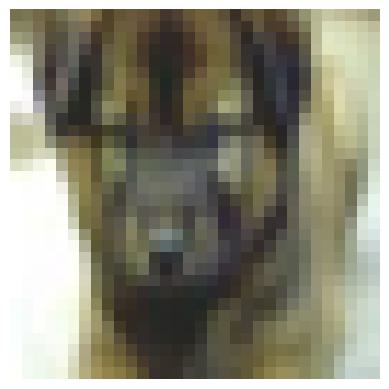}} &
        \parbox[c]{1.5em}{\includegraphics[width=0.30in]{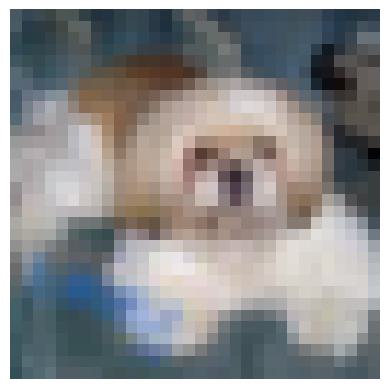}} &
        \parbox[c]{1.5em}{\includegraphics[width=0.30in]{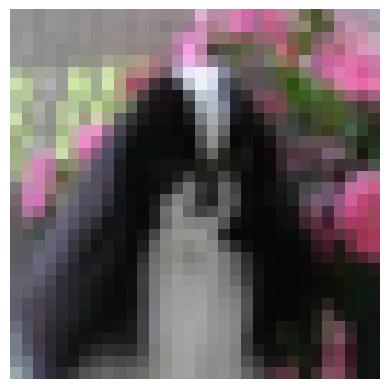}} &
        \parbox[c]{1.5em}{\includegraphics[width=0.30in]{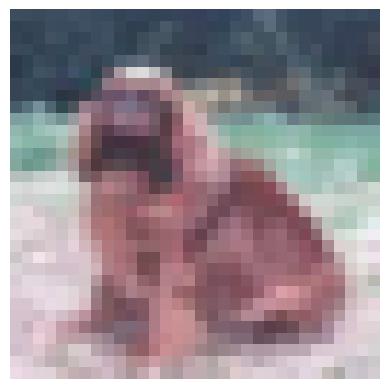}} &
        \parbox[c]{1.5em}{\includegraphics[width=0.30in]{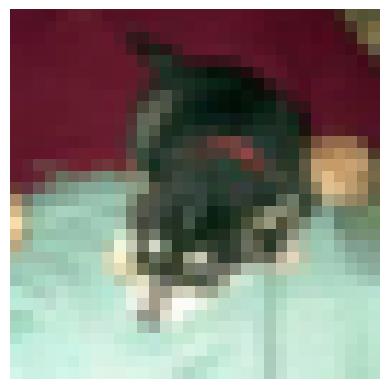}} &
        \parbox[c]{1.5em}{\includegraphics[width=0.30in]{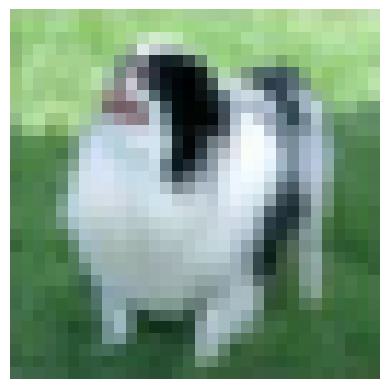}} &
        \parbox[c]{1.5em}{\includegraphics[width=0.30in]{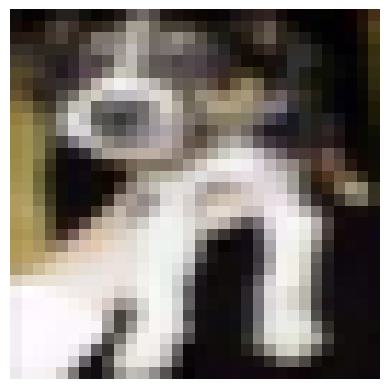}} \\
      \vspace{-0.2cm} \\
    \hline &&&&&&&&&& \vspace{-0.2cm} \\
      	FGZOHT &&&&&&&&&& \vspace{-0.2cm} \\
      &&&&&&&&&&\vspace{-0.2cm} \\
      	\parbox[c]{1.5em}{\includegraphics[width=0.30in]{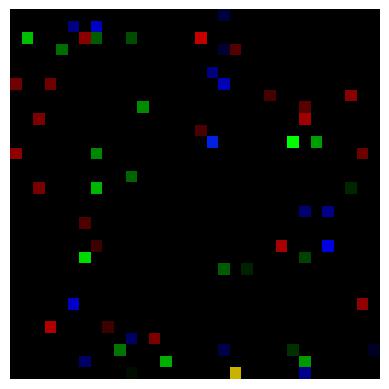}}  &
        \parbox[c]{1.5em}{\includegraphics[width=0.30in]{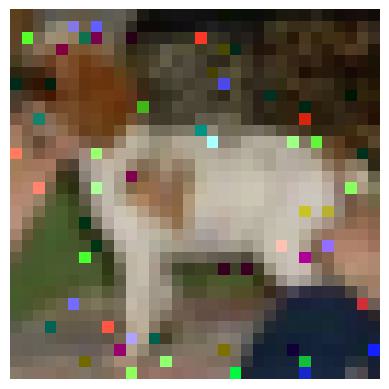}} &
        \parbox[c]{1.5em}{\includegraphics[width=0.30in]{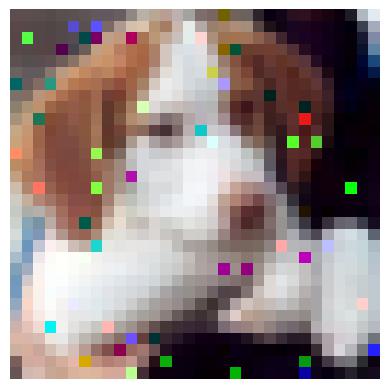}} &
        \parbox[c]{1.5em}{\includegraphics[width=0.30in]{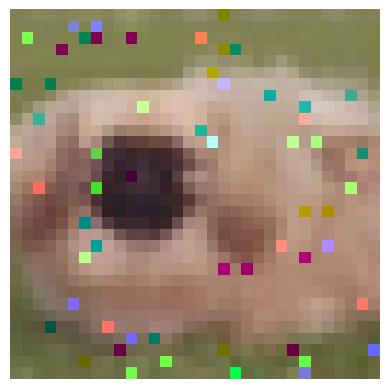}} &
        \parbox[c]{1.5em}{\includegraphics[width=0.30in]{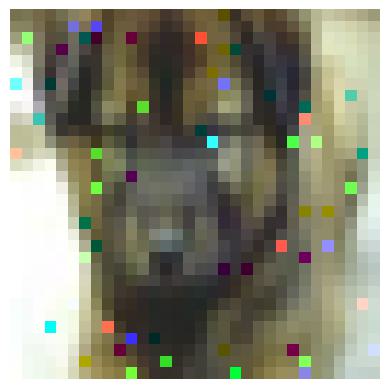}} &
        \parbox[c]{1.5em}{\includegraphics[width=0.30in]{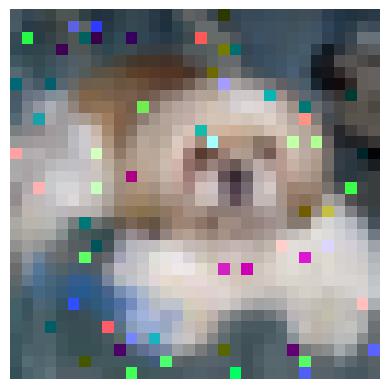}} &
        \parbox[c]{1.5em}{\includegraphics[width=0.30in]{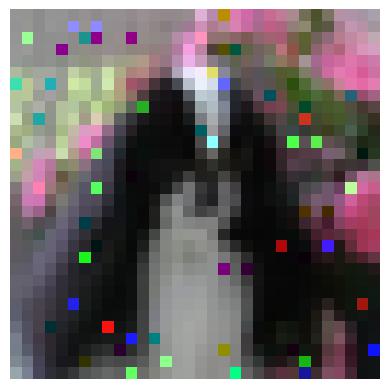}} &
        \parbox[c]{1.5em}{\includegraphics[width=0.30in]{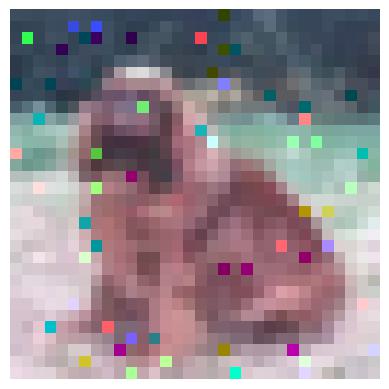}} &
        \parbox[c]{1.5em}{\includegraphics[width=0.30in]{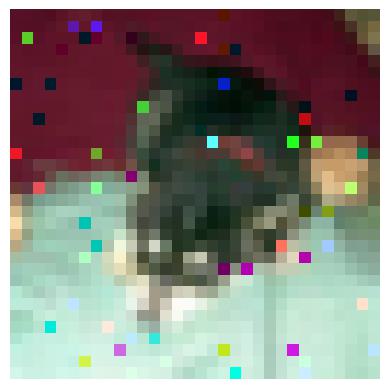}} &
        \parbox[c]{1.5em}{\includegraphics[width=0.30in]{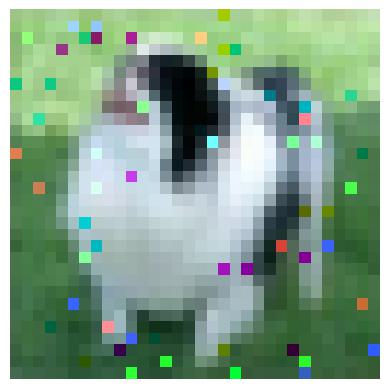}} &
        \parbox[c]{1.5em}{\includegraphics[width=0.30in]{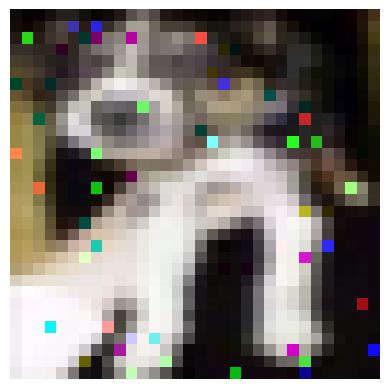}} \\
         & \textbf{bird} & dog & \textbf{deer} & dog & dog & \textbf{cat} & \textbf{cat} & \textbf{cat} & dog & dog \\
      \hline
      
    \hline &&&&&&&&&& \vspace{-0.2cm} \\
      	SZOHT &&&&&&&&&& \vspace{-0.2cm} \\
      &&&&&&&&&&\vspace{-0.2cm} \\
        \parbox[c]{1.5em}{\includegraphics[width=0.30in]{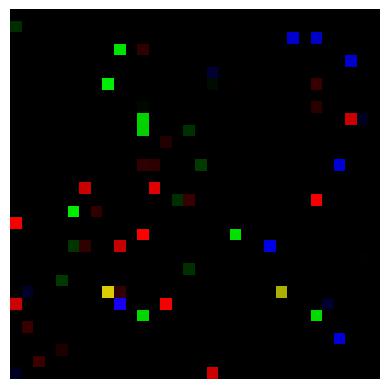}}  &
        \parbox[c]{1.5em}{\includegraphics[width=0.30in]{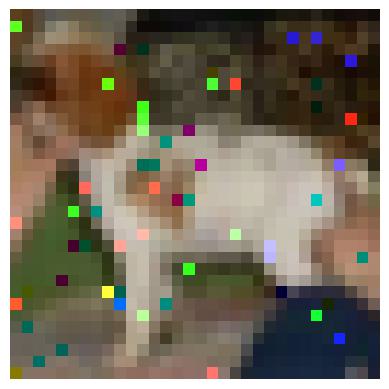}} &
        \parbox[c]{1.5em}{\includegraphics[width=0.30in]{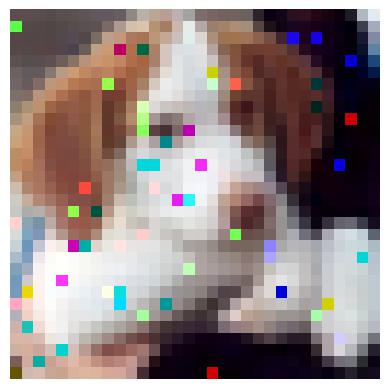}} &
        \parbox[c]{1.5em}{\includegraphics[width=0.30in]{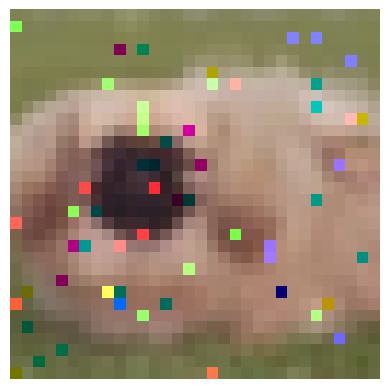}} &
        \parbox[c]{1.5em}{\includegraphics[width=0.30in]{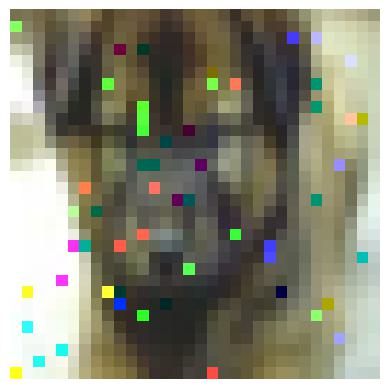}} &
        \parbox[c]{1.5em}{\includegraphics[width=0.30in]{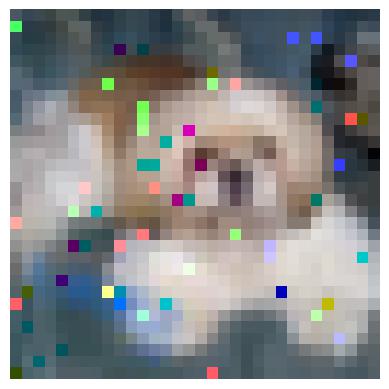}} &
        \parbox[c]{1.5em}{\includegraphics[width=0.30in]{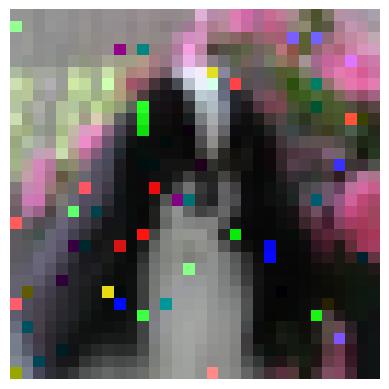}} &
        \parbox[c]{1.5em}{\includegraphics[width=0.30in]{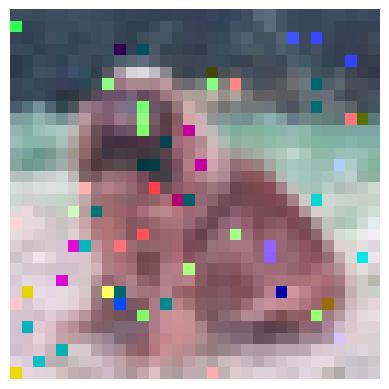}} &
        \parbox[c]{1.5em}{\includegraphics[width=0.30in]{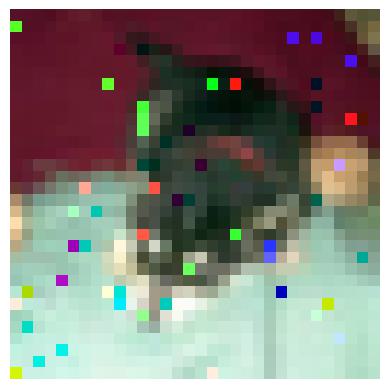}} &
        \parbox[c]{1.5em}{\includegraphics[width=0.30in]{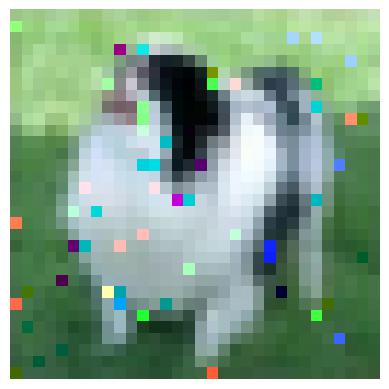}} &
        \parbox[c]{1.5em}{\includegraphics[width=0.30in]{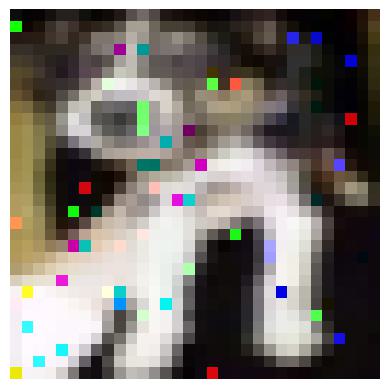}} \\
         & \textbf{frog} & dog & \textbf{bird} & dog & dog & \textbf{cat} & dog & \textbf{cat} & dog & dog \\
         \hline

    \hline &&&&&&&&&& \vspace{-0.2cm} \\
      	VR-SZHT &&&&&&&&&& \vspace{-0.2cm} \\
      &&&&&&&&&&\vspace{-0.2cm} \\
        \parbox[c]{1.5em}{\includegraphics[width=0.30in]{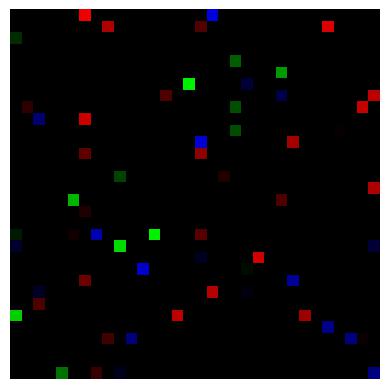}}  &
        \parbox[c]{1.5em}{\includegraphics[width=0.30in]{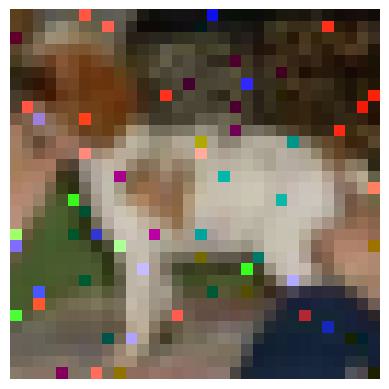}} &
        \parbox[c]{1.5em}{\includegraphics[width=0.30in]{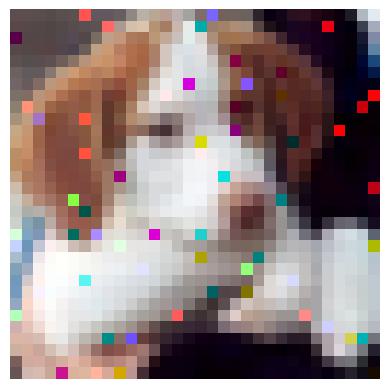}} &
        \parbox[c]{1.5em}{\includegraphics[width=0.30in]{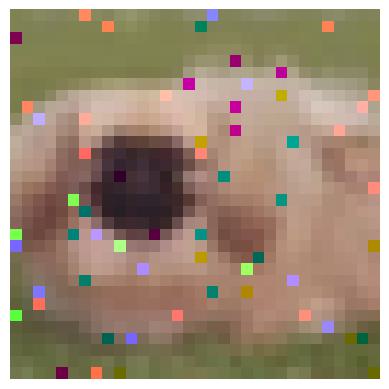}} &
        \parbox[c]{1.5em}{\includegraphics[width=0.30in]{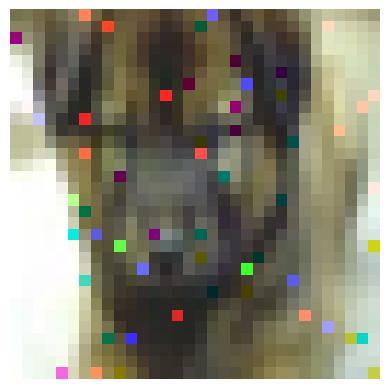}} &
        \parbox[c]{1.5em}{\includegraphics[width=0.30in]{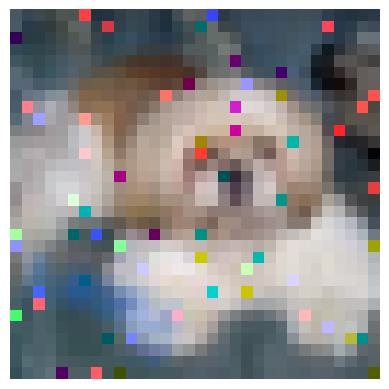}} &
        \parbox[c]{1.5em}{\includegraphics[width=0.30in]{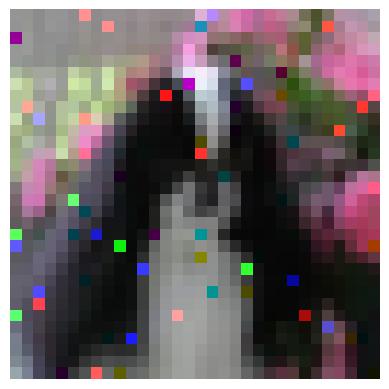}} &
        \parbox[c]{1.5em}{\includegraphics[width=0.30in]{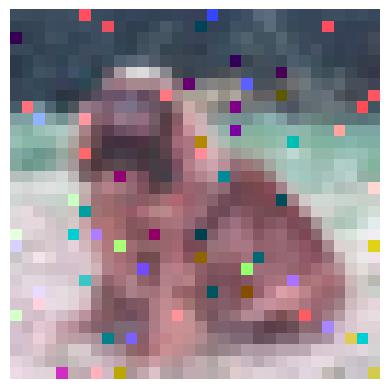}} &
        \parbox[c]{1.5em}{\includegraphics[width=0.30in]{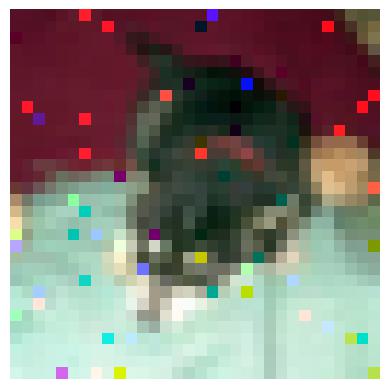}} &
        \parbox[c]{1.5em}{\includegraphics[width=0.30in]{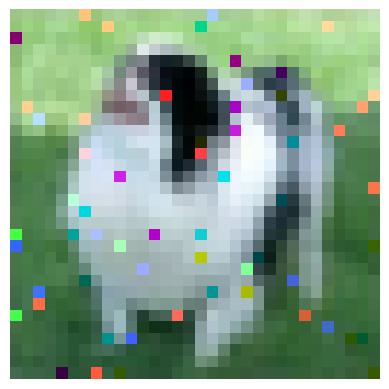}} &
        \parbox[c]{1.5em}{\includegraphics[width=0.30in]{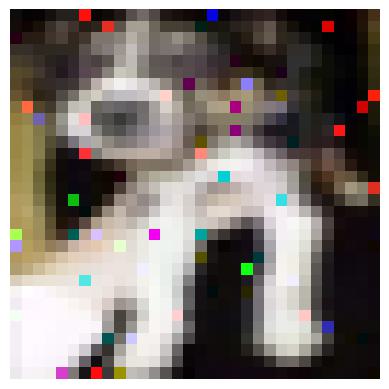}} \\
         & dog & dog & \textbf{bird} & dog & dog & \textbf{horse} & dog & dog & dog & dog \\
         \hline

             \hline &&&&&&&&&& \vspace{-0.2cm} \\
      	SAGA-SZHT &&&&&&&&&& \vspace{-0.2cm} \\
      &&&&&&&&&&\vspace{-0.2cm} \\
        \parbox[c]{1.5em}{\includegraphics[width=0.30in]{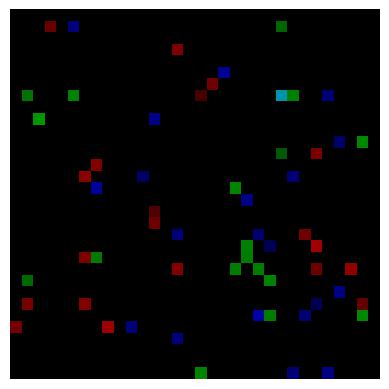}}  &
        \parbox[c]{1.5em}{\includegraphics[width=0.30in]{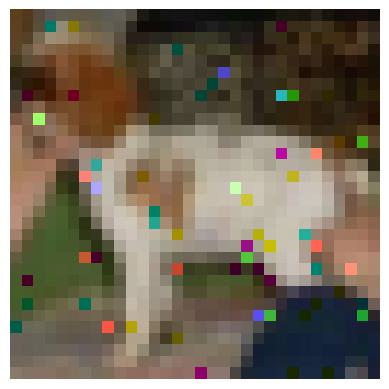}} &
        \parbox[c]{1.5em}{\includegraphics[width=0.30in]{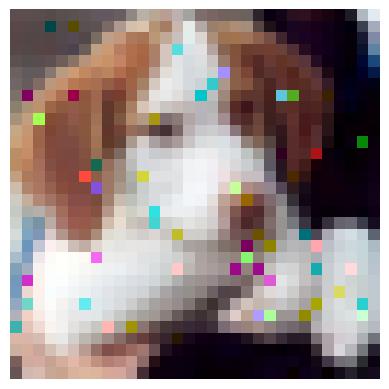}} &
        \parbox[c]{1.5em}{\includegraphics[width=0.30in]{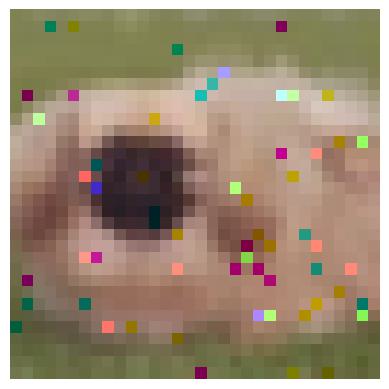}} &
        \parbox[c]{1.5em}{\includegraphics[width=0.30in]{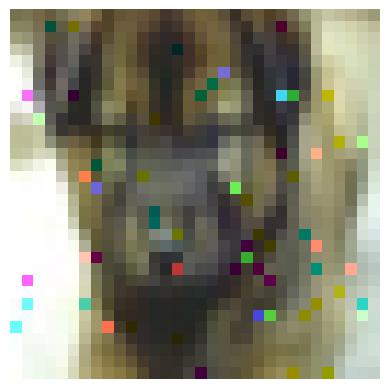}} &
        \parbox[c]{1.5em}{\includegraphics[width=0.30in]{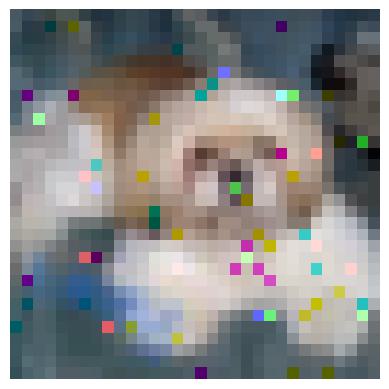}} &
        \parbox[c]{1.5em}{\includegraphics[width=0.30in]{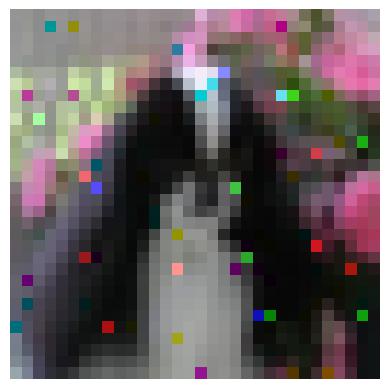}} &
        \parbox[c]{1.5em}{\includegraphics[width=0.30in]{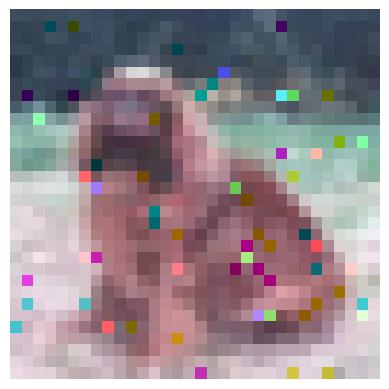}} &
        \parbox[c]{1.5em}{\includegraphics[width=0.30in]{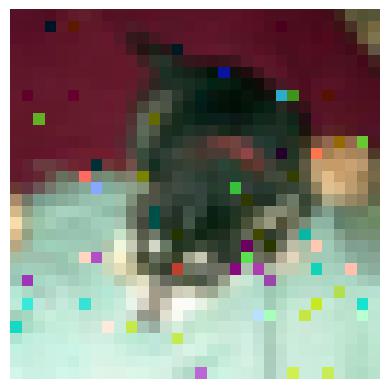}} &
        \parbox[c]{1.5em}{\includegraphics[width=0.30in]{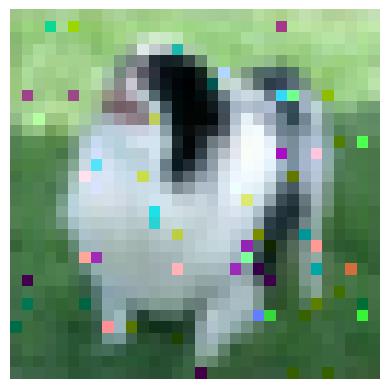}} &
        \parbox[c]{1.5em}{\includegraphics[width=0.30in]{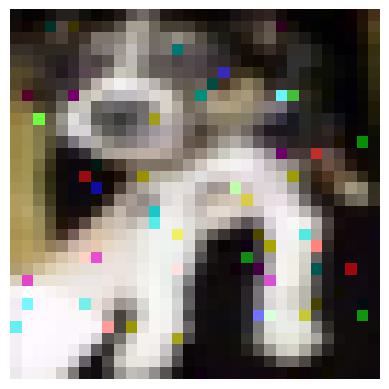}} \\
         & dog & dog & dog & dog & dog & \textbf{horse} & dog & \textbf{cat} & dog & dog \\
         \hline

             \hline &&&&&&&&&& \vspace{-0.2cm} \\
      	SARAH-SZHT &&&&&&&&&& \vspace{-0.2cm} \\
      &&&&&&&&&&\vspace{-0.2cm} \\
        \parbox[c]{1.5em}{\includegraphics[width=0.30in]{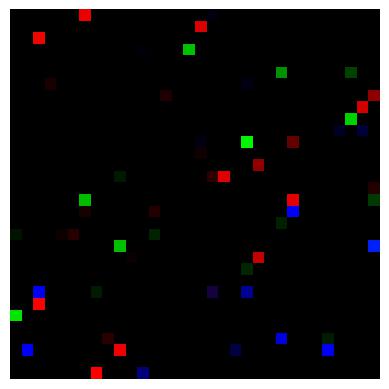}}  &
        \parbox[c]{1.5em}{\includegraphics[width=0.30in]{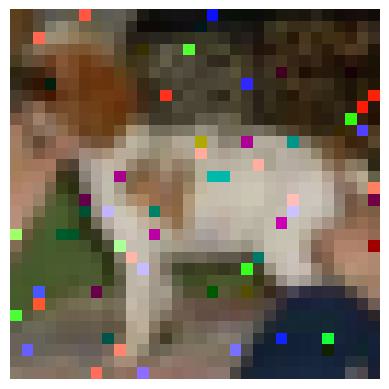}} &
        \parbox[c]{1.5em}{\includegraphics[width=0.30in]{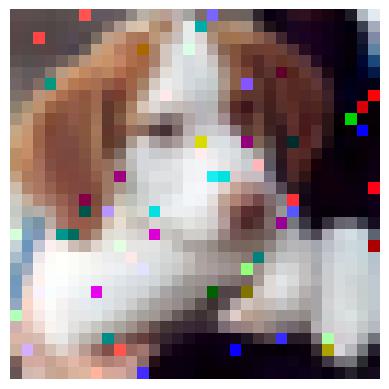}} &
        \parbox[c]{1.5em}{\includegraphics[width=0.30in]{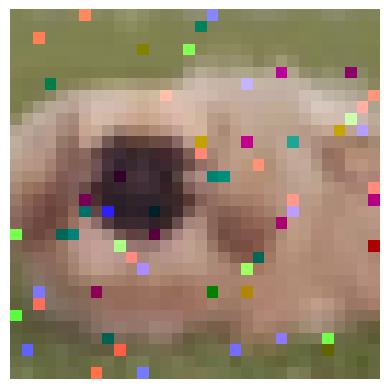}} &
        \parbox[c]{1.5em}{\includegraphics[width=0.30in]{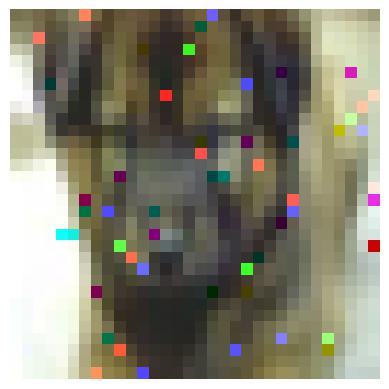}} &
        \parbox[c]{1.5em}{\includegraphics[width=0.30in]{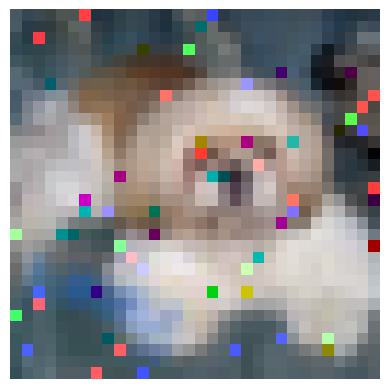}} &
        \parbox[c]{1.5em}{\includegraphics[width=0.30in]{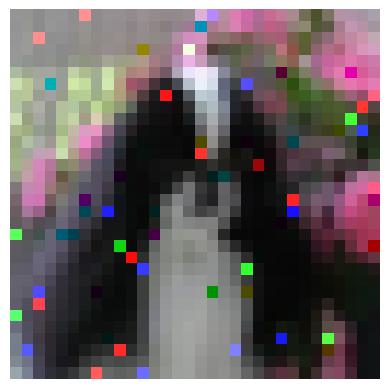}} &
        \parbox[c]{1.5em}{\includegraphics[width=0.30in]{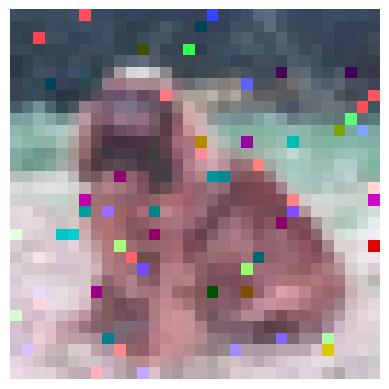}} &
        \parbox[c]{1.5em}{\includegraphics[width=0.30in]{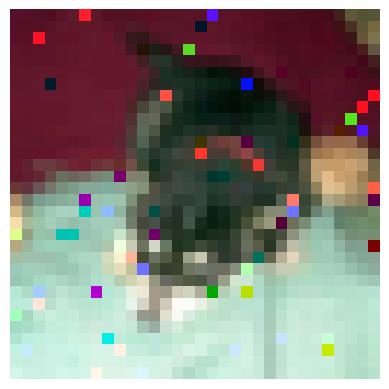}} &
        \parbox[c]{1.5em}{\includegraphics[width=0.30in]{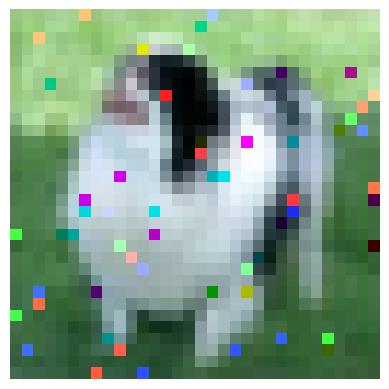}} &
        \parbox[c]{1.5em}{\includegraphics[width=0.30in]{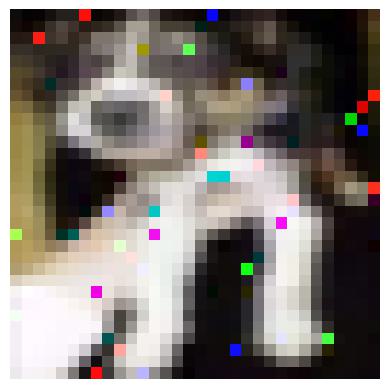}} \\
         & \textbf{deer} & dog & dog & \textbf{frog} & dog & \textbf{cat} & \textbf{frog} & dog & dog & dog \\
         \hline
      
  \end{tabular}
\end{table}

\end{document}